\documentclass[twoside,11pt]{article}

\usepackage{blindtext}

\usepackage[abbrvbib, preprint]{jmlr2e}
\usepackage{amsmath}
\usepackage{mathtools}
\usepackage{amsbsy}
\usepackage{amssymb}
\usepackage[dvipsnames, table]{xcolor}
\usepackage{graphicx}
\usepackage{natbib}
\usepackage{hyperref}
\usepackage{enumerate}
\usepackage{bm}
\usepackage{algorithm}
\usepackage{algpseudocode}
\usepackage{enumitem}
\usepackage{booktabs}
\usepackage{subcaption}
\usepackage{tikz}
\usetikzlibrary{positioning,calc,fit,backgrounds}
\usepackage{chemformula}
\usepackage{tablefootnote}
\usepackage{adjustbox}
\usepackage{tocloft}

\newcommand{\appendixtocname}{Appendix contents}
\newlistof{appendices}{app}{\appendixtocname}

\newcommand{\appsection}[1]{%
  \section{#1}%
  \addcontentsline{app}{section}{\protect\numberline{\thesection}#1}%
}

\newcommand{\Econv}{\mathcal{E}_{\mathrm{conv}\{f^{e_1},\,\dots,\,f^{e_K}\}}}

\newtheorem{setting}[theorem]{Setting}
\newtheorem{assumption}[theorem]{Assumption}

\definecolor{mygray}{gray}{0.5}
\definecolor{myblue}{RGB}{87,144,252}
\definecolor{myorange}{RGB}{248,156,32}
\definecolor{mygreen}{rgb}{0.133, 0.545, 0.133}
\definecolor{mypurple}{rgb}{0.58, 0.34, 0.92}

\tikzset{
  MSEconvexhull/.pic={
    \coordinate (Center) at (2,2);
    \def\radius{1.5}
    \coordinate (V1) at ($(Center)+(\radius,0)$);
    \coordinate (V2) at ($(Center)+({\radius*cos(140)},{\radius*sin(140)})$);
    \coordinate (V3) at ($(Center)+({\radius*cos(280)},{\radius*sin(280)})$);

    \draw[->] (-0.2,0) -- (4,0) node[right] {$\beta_1$};
    \draw[->] (0,-0.2) -- (0,4) node[above] {$\beta_2$};

    \fill[myblue, opacity=0.2] (V1) -- (V2) -- (V3) -- cycle;
    \draw[line width=1pt, myblue] (V1) -- (V2) -- (V3) -- cycle;
    \filldraw[myblue] (V1) circle (2pt) node[right] {$\bm{\beta}^{(1)}$};
    \filldraw[myblue] (V2) circle (2pt) node[above left, xshift=2pt] {$\bm{\beta}^{(2)}$};
    \filldraw[myblue] (V3) circle (2pt) node[below left, yshift=2pt] {$\bm{\beta}^{(3)}$};
    
    \coordinate (Est) at (V2);
    \filldraw[myorange] (Est) circle (3pt)
        node[above right, xshift=2pt] {$\bm{\beta}^{\text{MSE}}$};
    \node[above=4pt] at ($(0,4.25)!0.5!(4.25,4.25)$) {\textbf{MSE}};
  }
}

\tikzset{
  NRWconvexhull/.pic={
    \coordinate (Center) at (2,2);
    \def\radius{1.5}
    \coordinate (V1) at ($(Center)+(\radius,0)$);
    \coordinate (V2) at ($(Center)+({\radius*cos(140)},{\radius*sin(150)})$);
    \coordinate (V3) at ($(Center)+({\radius*cos(280)},{\radius*sin(280)})$);

    \draw[->] (-0.2,0) -- (4,0) node[right] {$\beta_1$};
    \draw[->] (0,-0.2) -- (0,4) node[above] {$\beta_2$};

    \fill[myblue, opacity=0.2] (V1) -- (V2) -- (V3) -- cycle;
    \draw[line width=1pt, myblue] (V1) -- (V2) -- (V3) -- cycle;
    \filldraw[myblue] (V1) circle (2pt) node[right] {$\bm{\beta}^{(1)}$};
    \filldraw[myblue] (V2) circle (2pt) node[above left, xshift=2pt] {$\bm{\beta}^{(2)}$};
    \filldraw[myblue] (V3) circle (2pt) node[below left, yshift=2pt] {$\bm{\beta}^{(3)}$};

    \path let
      \p1 = (V2),
      \p2 = (V3),
      \p0 = (0,0)
    in
      \pgfextra{
        \pgfmathsetmacro{\dx}{\x2-\x1}
        \pgfmathsetmacro{\dy}{\y2-\y1}
        \pgfmathsetmacro{\px}{0pt-\x1}
        \pgfmathsetmacro{\py}{0pt-\y1}
        \pgfmathsetmacro{\dot}{\px*\dx + \py*\dy}
        \pgfmathsetmacro{\lensq}{\dx*\dx + \dy*\dy}
        \pgfmathsetmacro{\t}{\dot/\lensq}
      }
      coordinate (EstAlt) at (\x1 + \t*\dx, \y1 + \t*\dy);

    \coordinate (Est) at (EstAlt);
    \filldraw[myorange] (Est) circle (3pt)
        node[left=1pt,xshift=-2.5pt] {$\bm{\beta}^{\text{NRW}}$};
    \draw[dashed, myorange, line width=1.5pt] (0,0) -- (Est);
    \node[above=4pt] at ($(0,4.25)!0.5!(4.25,4.25)$) {\textbf{Negative Reward}};
  }
}

\tikzset{
  Regconvexhull/.pic={
    \coordinate (Center) at (2,2);
    \draw[dashed, myorange, line width=1.5pt] (Center) circle (1.5);
    \def\radius{1.5}
    \coordinate (V1) at ($(Center)+(\radius,0)$);
    \coordinate (V2) at ($(Center)+({\radius*cos(140)},{\radius*sin(140)})$);
    \coordinate (V3) at ($(Center)+({\radius*cos(280)},{\radius*sin(280)})$);

    \draw[->] (-0.2,0) -- (4,0) node[right] {$\beta_1$};
    \draw[->] (0,-0.2) -- (0,4) node[above] {$\beta_2$};

    \fill[myblue, opacity=0.2] (V1) -- (V2) -- (V3) -- cycle;
    \coordinate (Est) at (Center);
    \filldraw[myorange] (Est) circle (3pt)
        node[below right, xshift=2pt] {$\bm{\beta}^{\text{Reg}}$};
    \draw[dashed, myorange, line width=1.5pt] (V1) -- (Est);
    \draw[dashed, myorange, line width=1.5pt] (V2) -- (Est);
    \draw[dashed, myorange, line width=1.5pt] (V3) -- (Est);

    \draw[line width=1pt, myblue] (V1) -- (V2) -- (V3) -- cycle;
    \filldraw[myblue] (V1) circle (2pt) node[right] {$\bm{\beta}^{(1)}$};
    \filldraw[myblue] (V2) circle (2pt) node[above left, xshift=2pt] {$\bm{\beta}^{(2)}$};
    \filldraw[myblue] (V3) circle (2pt) node[below left, yshift=2pt] {$\bm{\beta}^{(3)}$};
        
    \node[above=4pt] at ($(0,4.25)!0.5!(4.25,4.25)$) {\textbf{Regret}};
  }
}

\usepackage{lastpage}
\jmlrheading{??}{2026}{1-\pageref{LastPage}}{??/??; Revised ??/??}{?/??}{??-??}{Francesco Freni, Anya Fries, Linus K\"uhne, Markus Reichstein, Jonas Peters}

\ShortHeadings{Maximum Risk Minimization with Random Forests}{Freni, Fries, K\"uhne, Reichstein, and Peters}
\firstpageno{1}

\begin{document}

\title{Maximum Risk Minimization with Random Forests}

\author{\name Francesco Freni\thanks{Most of this work was done while FF was at ETH Z\"urich.} \email freni.francesco@math.ku.dk \\
       \addr Department of Mathematical Sciences\\
       University of Copenhagen\\
       Copenhagen, Denmark
       \AND
       \name Anya Fries \email anya.fries@stat.math.ethz.ch \\
       \addr Seminar for Statistics\\
       ETH Z\"urich\\
       Z\"urich, Switzerland
       \AND
       \name Linus Kühne \email linus.kuehne@ai.ethz.ch \\
       \addr Seminar for Statistics and ETH AI Center\\
       ETH Z\"urich\\
       Z\"urich, Switzerland
       \AND
       \name Markus Reichstein \email mreichstein@bgc-jena.mpg.de \\
       \addr Biogeochemical Integration Department \\
       Max-Planck-Institute for Biogeochemistry\\
       Jena, Germany
       \AND
       \name Jonas Peters \email jonas.peters@stat.math.ethz.ch \\
       \addr Seminar for Statistics\\
       ETH Z\"urich\\
       Z\"urich, Switzerland
}
\editor{My editor}

\maketitle

\begin{abstract}%
    We consider a regression setting where observations are collected 
    in different environments modeled by  different data distributions.
    The field of out-of-distribution (OOD) generalization aims to design methods that 
    generalize better to test environments whose distributions differ from those observed during training.
  One line of such works has
  proposed to 
  minimize the maximum risk across environments, %
    a principle that we refer to as 
    MaxRM (Maximum Risk Minimization).
    In this work, we introduce variants of random forests %
    based on the principle of 
    MaxRM.
    We 
    provide computationally efficient algorithms
    and prove statistical consistency for our primary method.
    Our proposed method can be used with each of the following three risks: the mean squared error, the negative reward,
    and the regret (which quantifies the excess risk relative to the best predictor). 
    For MaxRM with regret as the risk, we prove a novel out-of-sample guarantee over unseen test distributions.
    Finally, we evaluate the proposed methods on both simulated and real-world data. %
\end{abstract}


\section{Introduction}
We consider a regression setting, where the goal is to predict a response $Y\in\mathbb{R}$ from a set of covariates $X\in\mathbb{R}^p$. When applying
traditional machine learning methods, one often assumes
that %
training and test distribution %
are identical and that
training data are sampled i.i.d.\ from the joint distribution of $(X,\,Y)$. 
However, in many real-world problems,
observations are collected from multiple environments\footnote{
We use the term \emph{environments} to refer to conditions that induce different distributions over $(X,Y)$, such as 
different subpopulations, experimental conditions or time periods.
}
and the test distribution may differ from all of the training distributions---a setup that is often referred to as out-of-distribution (OOD) generalization or said to be subject to distribution shift (across environments).
In this work, 
we assume that the training data is partitioned into $K$ environments, denoted by $\mathcal{E}_\text{tr}\coloneqq\{e_1,\dots,e_K\}$
and 
that the test data come from 
environments $e~\in~\mathcal{E}_\text{te}$, which may differ from the training environments.
For all $e~\in~\mathcal{E}_\text{tr} \cup \mathcal{E}_\text{te}$, environment $e$ has data drawn i.i.d.\ from a distribution $P_e$ over $\mathbb{R}^p\times\mathbb{R}$.

Several strands of literature 
have emerged that aim to tackle OOD generalization.
\citet{Liu2023} categorize existing methods and provide a list of datasets that are commonly used to evaluate models under distribution shifts.

One research direction relies on unsupervised learning, such as unsupervised domain adaptation \citep{Zhang2015, Zhao2019}, unsupervised domain generalization \citep{Zhang2022} and disentangled representation learning \citep{Bengio2013, Locatello2019}.

Another line of work models distribution changes 
 via
interventions and builds on the link between invariance and causality \citep{Haavelmo1943, Pearl2009}: 
the conditional distribution of $Y$ given the causal predictors of $Y$ remains %
unchanged under interventions on
covariates other than $Y$. Such invariance principles have been used in causal discovery \citep{Peters2016, HeinzeDeml2017, Mogensen2022icml, Christiansen2018} and modern distribution generalization approaches \citep{Scholkopf2012, Rojas-Carulla2018, Meinshausen2018, Arjovsky2020, Rothenhausler2021, Christiansen2022, Lu2022}. 

Because this
causality-oriented perspective 
provides
worst-case optimality guarantees over a class of interventions \citep[e.g.,][]{Rojas-Carulla2018,Rothenhausler2021, Christiansen2022}, it 
can be viewed as an instance of the broader distributionally robust optimization (DRO) framework
\citep[e.g.,][]{Bental1998, Bendavid2006}. DRO formulates the learning problem as a minimax optimization. Given covariates $X$, a response $Y$ and 
a loss function $\ell$, 
one considers
\begin{equation} \label{eq:dro}
    \min_\theta\,\sup_{P\in\mathcal{P}}\,\mathbb{E}_P[\ell(X,Y;\theta)],
\end{equation}
where
$\theta$ is the model parameter vector,
$\mathbb{E}_P[\cdot]$ is the expectation under $P$, and $\mathcal{P}\coloneqq\{P:D(P,P_0)\le\rho\}$ is the neighborhood of radius $\rho$ around the training distribution $P_0$, under some divergence $D$, such as $f$-divergence or Wasserstein distance \citep{Namkoong2016, Sinha2020, Duchi2021}.

In linear settings, the maximin regression framework \citep{Meinshausen2015} offers a different approach to achieve robustness against environment-based shifts, defining an estimator that maximizes the minimum explained variance across the observed training environments. This estimator 
equals the coefficient vector
closest to the origin within the convex hull formed by all possible environment-specific coefficient vectors. 
\citet{Buhlmann2016} extend this idea to nonlinear models, introducing the magging (maximin aggregating) estimator, which is a convex combination of environment-specific predictors and corresponds to the maximin estimator in the linear case. \citet{Guo2024} and \citet{Wang2025} apply the principle of maximizing the minimum explained variance across environments to multi-source unsupervised domain adaptation (MSDA) for linear and nonlinear models, respectively. In contrast to the maximin and magging estimators, which rely only on labeled data from multiple environments, this setting additionally makes use of unlabeled data from the test environment and accounts for potential shifts in the distribution of the covariates. \citet{Guo2025} further study MSDA for classification tasks under the cross-entropy loss.

In this work, 
we allow the different training and test distributions to stem from differences in conditional or marginal distributions and assume neither an underlying causal structure, nor an invariant set of covariates,
nor the availability of additional test-environment observations. 
Under this setting, we focus on methods that aim to solve the optimization problem 
\begin{equation}\label{eq:prob_gdro}
    \min_{\theta}\,\max_{e\in\mathcal{E}_\text{tr}}\,\mathbb{E}_{P_e}[\ell(X^e,Y^e;\theta)] \quad \text{or} \quad 
    \min_{f \in \mathcal{F}}\,\max_{e\in\mathcal{E}_\text{tr}}\,\mathbb{E}_{P_e}[\ell(X^e,Y^e;f)],
\end{equation}
which we refer to as maximum risk minimization (MaxRM). This optimization problem
has a close connection to DRO. Specifically, 
group DRO \citep{Sagawa2020, Oren2019, Hu2018}
studies DRO in the 
setting with $K$ training environments described above.
If $\mathcal{P}$ 
equals the set of all convex combinations
of environment-specific 
training distributions $P_e$,
$e \in \mathcal{E}_{\text{tr}}$,
the set of solutions to~\eqref{eq:dro} %
can also be written as the solutions of the optimization problem~\eqref{eq:prob_gdro}.
If the loss is convex in the model parameters,~\eqref{eq:prob_gdro} 
can be solved using convex optimization methods \citep{Shapiro2002}. When the loss is non-convex, existing approaches rely on stochastic approximation methods \citep{Sagawa2020, Zhang2024SA}, which provide global convergence guarantees in the convex case. Moreover, most applications of group DRO focus on classification rather than regression,
and existing implementations rely on neural networks, 
making the methods
sensitive to model architecture and hyperparameter selection 
\citep[][Section~5.4]{Liu2021}.

While existing work on MaxRM, as described, 
focuses on model classes such as linear regression or neural networks, 
we 
propose novel methods that
apply MaxRM to random forests 
\citep{Breiman2001}---a model class that is known to work 
well in a wide range of applications, in particular those with large noise. %
The risk can be specified as either the mean squared error (MSE), the negative reward 
as in \citet{Wang2025},
or the regret\footnote{The regret is introduced by \citet{Agarwal22} within the DRO framework and further studied by \citet{Mo2024};
it is
also considered in causal inference \citep[e.g.,][]{Zhang2024}.}.

Current  results about 
theoretical properties of these estimators 
do not include the regret.
We thus extend existing proofs 
and show that, 
for all of the above choices, we obtain worst-case optimality guarantees if the test distribution comes from the convex hull of the training distributions (Theorem~\ref{thm:gdro_equivalence}).

In terms of algorithms, we 
propose computationally efficient methods that 
adapt the construction of regression trees by either adjusting 
only the 
assignment of values 
to the leaf regions according to the specified risk
(MaxRM-RF-posthoc) or 
additionally adjusting the partitioning of the input space 
(MaxRM-RF-local and MaxRM-RF-global).
We then show how to adapt the 
tree weights used to combine the trees into a single predictor according to the MaxRM objective.
For the post-hoc adjustment strategy, we 
develop two additional algorithms %
based on either the extragradient method or block-coordinate descent;
we also
prove a consistency result stating
that the leaf values obtained from the empirical optimization problem converge to the population minimizers
(Theorem~\ref{thm:consistency}).

In our simulation experiments, we find that 
our method outperforms
an existing group DRO implementation relying on neural networks.
Furthermore, in contrast to magging \citep{Buhlmann2016},
MaxRM random forests allow the covariate distribution to vary across environments (in such settings, 
magging is 
not expected to perform well because
the predictor which minimizes
the maximum risk can, in general, not be written as a convex combination of environment-specific predictors; see Appendix~\ref{app:comparison_magging} for a proof). 
Section~\ref{sec:ca_housing} shows that 
our methodology may lead to improved worst-case prediction 
when compared to existing baselines 
on real data.

To illustrate, Figure~\ref{fig:intro} shows a 
simple simulation example. 
MaxRM random forest,
configured to minimize the maximum MSE across environments using the post-hoc adjustment strategy (MaxRM-RF(mse) in the figure) is compared with
the magging estimator based on random forests (Magging-RF(mse) in the figure), and group DRO implemented via neural networks using \citet[][Algorithm~$1$]{Sagawa2020} (GroupDRO-NN(mse)).
Both are configured 
to minimize
the maximum MSE across environments.
Here, the proposed MaxRM random forest (orange) 
most closely approximates
the oracle function (dashed red), which minimizes the population maximum MSE across environments. 
\begin{figure}[t]
  \centering
  \begin{minipage}[c]{0.60\linewidth}
    \vspace{0pt}
    \centering
    \includegraphics[width=\linewidth]{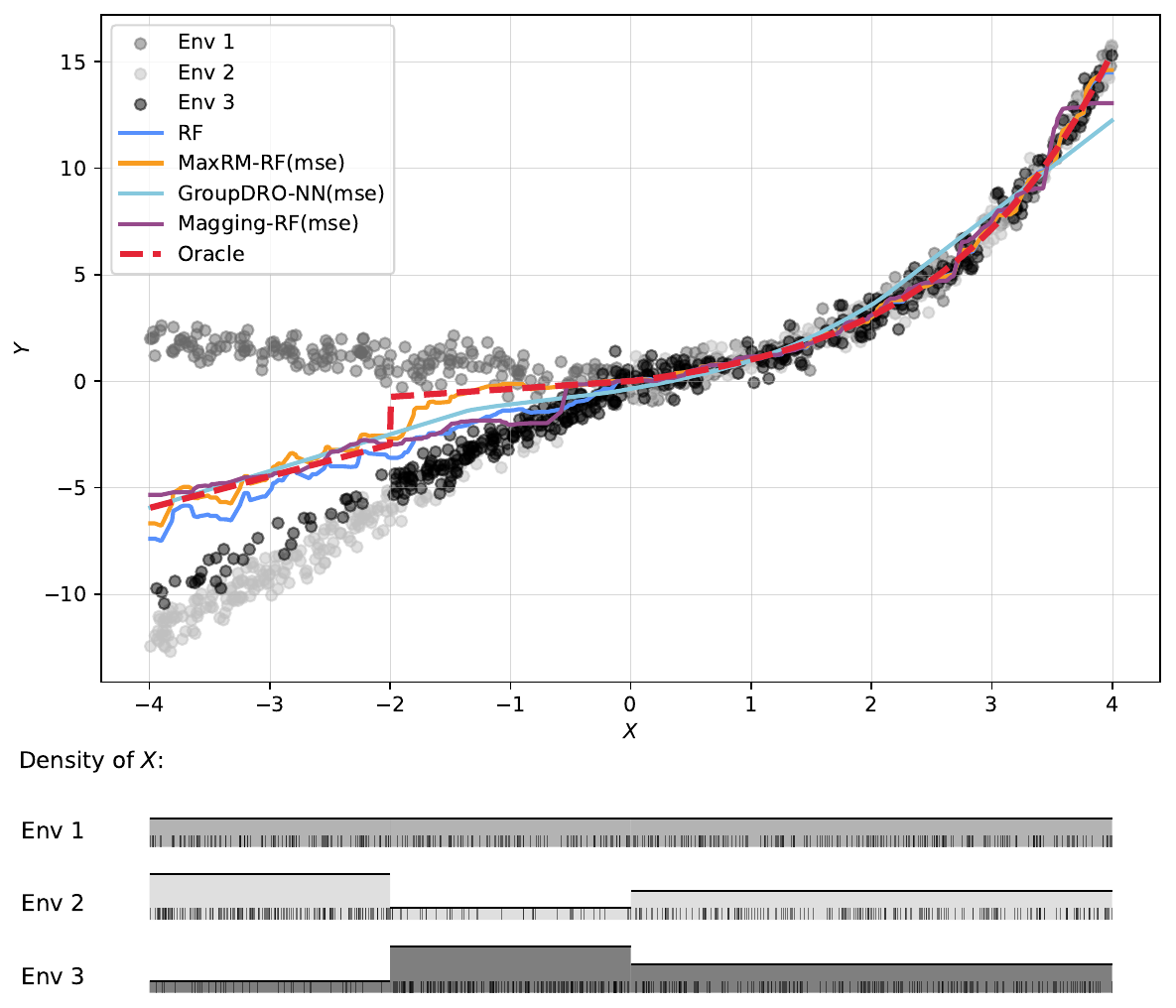}
  \end{minipage}\hfill
  \begin{minipage}[c]{0.4\linewidth}
    \vspace{0pt}
    \centering
    \small
    \begin{tabular}{@{}lc@{}}
      \toprule
      Method & MISE \\
      \midrule
      RF                 & $6.97\pm0.25$ \\
      MaxRM-RF(mse)     & $2.25\pm0.19$ \\
      GroupDRO-NN(mse)  & $11.88\pm1.07$\\
      Magging-RF(mse)       & $4.48\pm0.10$ \\
      \bottomrule
    \end{tabular}
  \end{minipage}
  \caption{
  {\it Left:} 
  The observed data (dots) are drawn from three distinct training environments, exhibiting both shifts in the conditionals $P^{Y|X}_e$ and marginals $P^X_e$ (bottom left subplot) across environments. MaxRM random forest (solid orange) more closely approximates the analytically computable oracle solution (minimizing the population maximum MSE across environments; dashed red)
  than
  magging; this 
  is expected in that the latter is only guaranteed to minimize the maximum risk across training environments if there are no shifts in $P^X_e$ (see Appendix~\ref{app:comparison_magging}).
  Similarly to MaxRM-RF,
  group DRO also
  aims to minimize the maximum (empirical) MSE 
  across training environments
  but its implementation 
  performs worse than MaxRM-RF. %
  {\it Right:} mean integrated squared error (MISE) with $95\%$ confidence intervals over $100$ simulation runs. MaxRM random forest achieves the lowest MISE.
  }
  \label{fig:intro}
\end{figure}

\subsection{Notation}
We now introduce some notation used throughout the paper. We use $(X,\,Y)$ for a random vector with joint distribution $P$ over $\mathbb{R}^p\times\mathbb{R}$, and denote by $P^X$ and $P^{Y|X}$ the marginal distribution of $X$ and the conditional distribution of $Y|X$, respectively. 
We denote the set of observable environments by $\mathcal{E}$ and denote by 
$\mathcal{E}_\text{tr}, \, \mathcal{E}_\text{te} \subseteq \mathcal{E}$
the sets of training and test environments, respectively. 
For all $e \in \mathcal{E}$, let $(X^e,\, Y^e)$ be a random vector with joint distribution $P_e$, marginal $P_e^X$ and conditional $P_e^{Y|X}$. Let $\mathcal{P}$ denote a set of probability measures over $\mathbb{R}^p\times\mathbb{R}$ containing the joint distributions that can be encountered at test time ($\mathcal{P}$ may depend on the context and the considered method). Expectations under $P_e$, $P_e^X$ and $P\in\mathcal{P}$ are denoted by $\mathbb{E}_e[\cdot]$, $\mathbb{E}_{P_e^X}[\cdot]$ and $\mathbb{E}_P[\cdot]$, respectively.

\section{Maximum risk minimization}\label{sec:maxrm}
Empirical risk minimization \citep[ERM;][]{Vapnik1998} minimizes the average loss over all training observations. In contrast, 
we 
focus on minimizing the maximum risk across training environments, which we refer to as maximum risk minimization (MaxRM); we will see in Theorem~\ref{thm:gdro_equivalence} that this principle comes with optimality guarantees on the test set. Throughout the paper, we assume the following setting.
\begin{setting}\label{setting:setting_maxrm}
For all $e\in\mathcal{E}$, the response satisfies $$Y^e=f^e(X^e)+\epsilon^e,$$ 
where the noise $\epsilon^e$ is such that $\mathbb{E}_e[(\epsilon^e)^2\mid X^e]=:\sigma_e^2<\infty$
and $\mathbb{E}_e[\epsilon^e\mid X^e] = 0$. Thus, 
$f^e : \mathbb{R}^p \to \mathbb{R}$ satisfies, for all $\bm{x}\in\mathbb{R}^p$, $f^e(\bm{x})=\mathbb{E}_e[Y^e\mid X^e=\bm{x}]$. We only observe data from a subset of $K \in \mathbb{N}$ training environments $\mathcal{E}_\textnormal{tr} ~\coloneqq ~\{e_1,\dots,e_K\} \subseteq \mathcal{E}$. In particular, for all $e\in\mathcal{E}_\textnormal{tr}$, we have $n_e$ i.i.d.\ realizations $\{(\bm{x}_i^e,y_i^e)\}_{i=1}^{n_e}$, with known environment labels, and denote the sample size as $n \coloneqq \sum_{k=1}^K n_{e_k}$. 
We consider predictors $f: \mathbb{R}^p \to \mathbb{R}$ from a class $\mathcal{F}$ of measurable functions. Finally, we assume that the data-generating process, the function class, and the loss function are such that 
all expectations appearing in this paper
are finite.
\end{setting}
The risk of a function $f\in\mathcal{F}$ under probability distribution $P\in\mathcal{P}$ and loss function $\ell:\mathbb{R}^p\times\mathbb{R}\to\mathbb{R}^{+}$ is defined as $R_P(f)\coloneqq\mathbb{E}_P[\ell(X,\,Y;\,f)]$, which depends on the choice of the loss~$\ell$. For all $e\in\mathcal{E}$, we let $R_e(f)$ denote the risk of function $f$ under $P_e$. 
ERM
considers 
(the empirical version of)
$\min_{f \in \mathcal{F}} \sum_{e \in \mathcal{E}_{\text{tr}}} n_e R_{e}(f)$,
ignoring possible differences among these environments, as well as changes in distribution at test time. 
MaxRM, instead, focuses on the worst-case risk across all training environments,
$\min_{f \in \mathcal{F}}\,\max_{e\in\mathcal{E}_\text{tr}}\,R_{e}(f)$,
see~\eqref{eq:prob_gdro}. 
We will see in Theorem~\ref{thm:gdro_equivalence} that this comes with optimality guarantees over test distributions.
More concretely, the corresponding predictors can be shown to solve optimization problems of the %
form 
\begin{equation}\label{eq:main_goal}
    \min_{f\in\mathcal{F}}\,\max_{P\in\mathcal{P}}\,R_P(f),
\end{equation}
which involves the risk $R_P$ and an uncertainty set\footnote{The name `uncertainty set' stems from the fact that problems of the form~\eqref{eq:main_goal} have been studied as an object of interest in itself \citep[see, e.g.,][Section~$2$ for an overview]{Kuhn2025}.} $\mathcal{P}$.
Section~\ref{subsec:risk}
describes choices for risks and in
Section~\ref{subsec:uncertainty_set} 
we discuss 
different generalization guarantees.

\subsection{Choice of the risk}\label{subsec:risk}

A common choice of risk, both in the DRO framework for regression \citep{Duchi2021} and in standard regression settings, is the mean squared error (MSE), defined as the expected squared loss:
\begin{equation*}%
    R_P^\text{MSE}(f)\coloneqq\mathbb{E}_P[(Y-f(X))^2].
\end{equation*}
In the linear setting, \citet{Meinshausen2015} introduce the maximin estimator, which is obtained by minimizing the maximum \emph{negative reward} (NRW) across environments. The negative reward is 
expressed as the MSE of $f$ compared to the MSE of the null model:%
\begin{equation*}%
    R_P^\text{NRW}(f)\coloneqq\mathbb{E}_P[(Y-f(X))^2]-\mathbb{E}_P[Y^2].
\end{equation*}
Minimizing the maximum negative reward is also considered by  \citet{Buhlmann2016} and \citet{Wang2025}, who study the nonlinear setting, with the latter focusing on multi-source unsupervised domain adaptation (MSDA).
These two methods coincide when the covariate distribution is the same across environments.

In contrast, the \emph{regret} (Reg), introduced by \citet{Agarwal22} in the DRO framework, measures the MSE of $f$ against the optimal predictor within $\mathcal{F}$:
\begin{equation*}%
    R_P^\text{Reg}(f)\coloneqq\mathbb{E}_P[(Y-f(X))^2]-\inf_{g\in\mathcal{F}}\mathbb{E}_P[(Y-g(X))^2].
\end{equation*}
\citet{Mo2024} study the regret for linear regression, and \citet{Zhang2024} for heterogeneous treatment effect estimation in causal inference.

\begin{remark}[Heterogeneous noise levels]
    In Setting~\ref{setting:setting_maxrm},
    MaxRM when using MSE, see~\eqref{eq:prob_gdro}, can be dominated by the environment with the largest noise variance. Let us assume for all training environments $e \in \mathcal{E}_\mathrm{tr}$
    that $f^e \in \mathcal{F}$. 
    Then, for all $e \in \mathcal{E}_\mathrm{tr}$,
    $\inf_{g\in\mathcal{F}} \mathbb{E}_e[(Y^e-g(X^e))^2] = \sigma_e^2$. This implies for all $f \in \mathcal{F}$, $$\max_{e\in\mathcal{E}_\mathrm{tr}} R_e^{\mathrm{Reg}}(f) 
    = \max_{e\in\mathcal{E}_\mathrm{tr}} \bigl( R_e^{\mathrm{MSE}}(f) - \sigma_e^2 \bigr)= \max_{e\in\mathcal{E}_\mathrm{tr}} \mathbb{E}_e\left[(f^e(X) - f(X))^2\right].$$Thus, in a population setting, minimizing the maximum regret is equivalent to minimizing the maximum MSE corrected for environment-specific noise levels---the solution of MaxRM used with regret does not change when varying $\sigma_e$ and corresponds to the solution of MaxRM with MSE in a noiseless setting.
\end{remark}

\citet[Section~2.3]{Wang2025} provide a geometric interpretation for the above risks in the case of $K=3$ environments, where the uncertainty set 
$\mathcal{P}$ is the convex hull of the joint training distributions in Equation~\eqref{eq:gdro_uncertainty_set}.

\subsection{
Generalization guarantee
}\label{subsec:uncertainty_set}

For the cases of MSE and negative reward, the literature on group DRO \citep{Sagawa2020,Oren2019,Hu2018} has established that 
MaxRM solves 
\eqref{eq:main_goal}
for $\mathcal{P}$ being the convex hull of the joint training distributions\footnote{
In MSDA \citep{Wang2025}, unlabeled covariates from the test environment are used to learn an improved model tailored to that environment.
Let $e^\prime\in\mathcal{E}\setminus\mathcal{E}_\text{tr}$ be the test environment.
In this setting, $P_{e^\prime}^X$ can be estimated from the data, while $P_{e^\prime}^{Y|X}$ is unknown. To define the uncertainty set $\mathcal{P}$, \citet{Wang2025} fix the marginal to $P_{e^\prime}^X$, while allowing the conditional to vary withing the convex hull of the training conditionals:
$\mathcal{P}_\text{MSDA}\!\left(P_{e^\prime}^{X}\right) \coloneqq
\{\, Q \,|\, 
\forall A \in \mathcal{B}(\mathbb{R}^p),\, B \in \mathcal{B}(\mathbb{R}) :\ 
 Q(X\in A,\,Y\in B) = \int_{A} Q^{Y\mid X=\bm{x}}(B)\, P_{e^\prime}^{X}(\mathrm{d}\bm{x}),
 Q^{Y\mid X} = \sum_{k=1}^{K} q_k\, P_{e_k}^{Y\mid X},\ \bm{q}\in \Delta_{K}
\}
$.
\citet{Guo2024} first introduce this uncertainty set in the linear setting, where, for all $e\in\mathcal{E}_\text{tr}$ and $\bm{x}\in\mathbb{R}^2$, $f^e(\bm{x})=\bm{x}^\top\bm{\beta}^{(e)}$ with $\bm{\beta}^{(e)}\in\mathbb{R}^2$, and \citet{Guo2025} use $\mathcal{P}_\text{MSDA}$ in classification tasks.
When the covariate distribution is the same across environments, the uncertainty sets $\mathcal{P}_\text{CVXH}$ and $\mathcal{P}_\text{MSDA}$
coincide
(and thus, by Theorem~\ref{thm:gdro_equivalence}, MSDA is equivalent to MaxRM).
}: 
\begin{equation}\label{eq:gdro_uncertainty_set}
    \mathcal{P}_\text{CVXH}\coloneqq\left\{P\;\middle|\;  P=\sum_{k=1}^K q_k P_{e_k},\, \bm{q}\in\Delta_K\right\},
\end{equation}
where $\Delta_{K}\!\coloneqq\!\bigl\{\bm q\in\mathbb R^{K}\mid\min_{k\in \{1,\,\dots,\,K\}}q_K\!\ge0,\,\sum_{k=1}^K q_k=1\bigr\}$ is the $(K-1)$-dimensional simplex.

We now prove that this result holds for the regret, too. As for the proof for the MSE and the NRW, 
its main argument is based on
Bauer's maximum principle
\citep{Bauer1958}.
For completeness, 
the following theorem 
also contains the known result for the MSE and the NRW.
\begin{theorem}[Equivalence between group DRO and MaxRM]\label{thm:gdro_equivalence}
    For all risk definitions $r\in\{\mathrm{MSE},\,\mathrm{NRW},\,\mathrm{Reg}\}$, minimizing the maximum risk over $\mathcal{P}_\mathrm{CVXH}$ is equivalent to solving MaxRM:
for all $f \in \mathcal{F}$, we have
    \begin{equation*}
\max_{P\in\mathcal{P}_\mathrm{CVXH}}\,R_P^r(f)=\max_{e\in\mathcal{E}_\textnormal{tr}}\,R_e^r(f),
    \end{equation*}
so also the minimizers of both sides over $\mathcal{F}$ are equal whenever they exist.
\end{theorem}
A proof can be found in Appendix~\ref{sec:proof_gdro_equivalence}.

Under the assumption that the covariate distributions do not change,
the guarantees of MaxRM are even stronger. \citet[][Section~2.4]{Wang2025}
show that
MaxRM also solves~\eqref{eq:main_goal} when 
the uncertainty
set~\eqref{eq:gdro_uncertainty_set} is enlarged to distributions whose conditional mean lies in the convex hull of the training conditional means. 
They formally state the result only for the negative reward, but it also holds for the regret. 
For the
mean squared error, however, the statement is false without additional assumptions. For completeness, we 
state the result for all three risks in
Proposition~\ref{prop:generalization_test_env}.

To do so, we 
use the following
assumptions and definitions. 
\begin{assumption}[No changes in covariate distribution]\label{ass:no_changeXdistr}
    The distribution of the covariates is the same across all 
    environments: there exists $P^X$ such that, for all %
    $e \in \mathcal{E}$, $P_e^X = P^X$. 
\end{assumption} 
We define $\Econv$ as the subset of
environments whose regression functions lie in the convex hull of the training regression functions, that is,
\begin{equation*}
    \Econv \coloneqq \Bigl\{e' \in \mathcal{E}\!:\, 
    \exists\bm{q}\in\Delta_K\ \textrm{such that }
    f^{e'}(\cdot)=\sum_{k=1}^K q_k f^{e_k}(\cdot)\Bigr\}.
\end{equation*}

\begin{assumption}[Constant noise variance]\label{ass:constant_noisevar}
    The noise variance is constant across environments, that is, for all $e\in\mathcal{E}_\textnormal{tr}\cup\Econv$, $\sigma_e^2=\sigma^2$.
\end{assumption}
The next result characterizes when the maximum risk across $\Econv$ (all %
environments whose regression function is a convex combination of the training regression functions) coincides with the maximum risk across training environments. 
\begin{proposition}[Generalization to the convex hull of training conditional means]\label{prop:generalization_test_env}
Let $r \in \{\mathrm{NRW},\, \mathrm{Reg},\, \mathrm{MSE}\}$. 
Suppose Assumption~\ref{ass:no_changeXdistr} (no changes in covariate distribution) holds. If $r = \mathrm{Reg}$, assume that the function class $\mathcal{F}$ contains the convex hull of $\{f^{e_1},\,\dots,\,f^{e_K}\}$.
\begin{enumerate}[label=(\roman*)]
    \item\label{itm:prop_equality} 
    If Assumption~\ref{ass:constant_noisevar} (constant noise variance) holds, then,
   for all $f\in\mathcal{F}$,
    \begin{equation} \label{eq:prop_5}
        \max_{e^\prime \in \Econv} R_{e^\prime}^r(f) = \max_{e\in\mathcal{E}_\mathrm{tr}} R_e^r(f).
    \end{equation}
    \item\label{itm:prop_mse_degeneration} 
    If Assumption~\ref{ass:constant_noisevar} (constant noise variance)  does not hold, then
    Equality~\eqref{eq:prop_5} holds if $r\in\{\mathrm{NRW},\, \mathrm{Reg}\}$, but may fail if $r = \mathrm{MSE}$.
\end{enumerate} 
\end{proposition}
A proof can be found in Appendix~\ref{sec:proof_test_env}.

In the setting 
    with 
    no changes in covariate distribution
    (see Proposition~\ref{prop:generalization_test_env} above),
\citet[][Theorem~2.1]{Wang2025} show 
that
the function $f^{*,\,r}$ 
that minimizes the left-hand side of~\eqref{eq:prop_5}
can be written as a convex combination of the regression functions in the training environments $\{f^{e_1}, \ldots, f^{e_K}\}$. 
Namely, there exists $\bm{q}^{*,\,r}\in\Delta_K$ such that, for all $\bm{x}\in\mathbb{R}^p$, $f^{*,\,r}(\bm{x})\coloneqq\sum_{k=1}^K q_k^{*,\,r} f^{e_k}(\bm{x})$. 
The optimal $\bm{q}^{*,\,r}$ is found by solving a convex optimization problem that depends on the chosen risk. 
This fact is exploited by magging \citep{Buhlmann2016}, for example.
Our setting allows for the covariate distribution to vary across environments and in this case, $f^{*,\,r}$ does not admit such a representation. 

\section{MaxRM random forest}\label{sec:maxrm_rf}
We now show how to modify random forests \citep[RFs;][]{Breiman2001} to minimize the maximum risk across the observed training environments, that is, to solve MaxRM.
Define the uncertainty set $\mathcal{P}$ of joint distributions that can be observed at test time as 
the convex hull of joint training distributions, as in \eqref{eq:gdro_uncertainty_set}.
We consider the three definitions of risk introduced in Section~\ref{subsec:risk}, that is, for all $r\in\{\text{MSE, NRW, Reg}\}$, we aim to find
\begin{equation} \label{eq:RFform}
f^{*,\,r}\in\arg\min_{f\in\mathcal{F}}\,\max_{e\in\mathcal{E}_\text{tr}}\,R_e^r(f),
\end{equation}
where now $\mathcal{F}$ denotes the function class of random forests. We tackle Problem~\eqref{eq:RFform} from different angles, both at the tree level and the forest level.

RFs combine multiple noisy regression trees into a single ensemble predictor. The key ingredients for fitting the regression trees  are (i) bagging \citep[bootstrap aggregating;][]{Breiman1996} (each tree is trained on an independent bootstrap sample of the original dataset); and (ii) random covariate selection (when growing a tree, at each step, the determination of the next split and the values of the resulting nodes are based on a random subset of $m_\text{try}\le p$ covariates).
After constructing a finite collection of $B$ trees $\{h_{\bm{\theta}^{(b)}}\}_{b=1}^B$, where, for all $b\in\{1,\,\dots,\,B\}$, tree $h_{\bm{\theta}^{(b)}}$ is parameterized by a vector of leaf values $\bm{\theta}^{(b)}$, the RF prediction equals the average of the tree predictions:
\begin{equation*}%
    f_{\bm{c}, \bm{w}}(\cdot) \coloneqq \sum_{b=1}^B w_b\, h_{\bm{\theta}^{(b)}}(\cdot),
\end{equation*}
where $\bm{c} \coloneqq (\bm{\theta}^{(1)},\, \ldots,\, \bm{\theta}^{(B)})$ and $\bm{w}\coloneqq(w_1,\,\dots,\,w_B)\in\Delta_B$ concatenate the parameters and weights of all trees, respectively. In standard RFs, for all $b \in \{1, \ldots, B\}$, $w_b := \frac{1}{B}$. 

In MaxRM-RFs, we instead modify the construction of individual trees, the tree weights, or both to optimize the MaxRM objective. Section~\ref{sec:leaves} describes how to adjust the assignment of leaf values to the regions of a fixed partition by solving a convex optimization problem. Section~\ref{sec:partition} then proposes ways to combine this modified leaf-value assignment with different strategies for partitioning the input space. Section~\ref{sec:weights} explains how to choose $\bm{w}$ according to the MaxRM objective. Finally, Section~\ref{sec:algorithms_posthoc} presents alternative algorithms for the optimization problems introduced in Section~\ref{sec:leaves}. Code for the  methodology we propose is available at \url{https://github.com/francescofreni/adaXT/}.

\subsection{Assigning leaf values} \label{sec:leaves}
Both MaxRM and standard regression trees model the regression function as a piecewise constant function over a partition $\mathcal{L} \coloneqq \{\mathcal{X}_1, \ldots, \mathcal{X}_T\}$ of the input space $\mathbb{R}^p$ into $T$ disjoint leaf regions:
$\bigcup_{t=1}^T \mathcal{X}_t = \mathbb{R}^p$, $\mathcal{X}_t \cap \mathcal{X}_{t^\prime} = \emptyset$ for all $t \neq t^\prime,\; t,t^\prime \in \{1,\,\dots,\,T\}$.
Let $\bm{\theta}\coloneqq(\theta_1,\,\dots,\,\theta_T)^\top\in\mathbb{R}^T$ be the vector of leaf values. For all $\bm{x}\in\mathbb{R}^p$, the prediction of tree $h_{\bm{\theta}}$ at $\bm{x}$ is then given by $h_{\bm{\theta}}(\bm{x}) \coloneqq \sum_{t=1}^T \theta_t \mathbf{1}_{\{\bm{x} \in \mathcal{X}_t\}}$.

Let $\mathcal{L}$ be the partition of the input space found after constructing the regression tree as for standard RFs. For all $e\in\mathcal{E}_\text{tr}$, the MSE of tree $h_{\bm{\theta}}$ in environment $e$ is given by
$R_e^\text{MSE}(h_{\bm{\theta}})=\mathbb{E}_e\left[\sum_{t=1}^T(Y^e - \theta_t)^2\,\mathbf{1}_{\{X^e\in\mathcal{X}_t\}}\right]$.
We now determine the leaf values $\bm{\theta}$ which solve MaxRM, that is,
\begin{equation}\label{eq:obj_posthoc}
    \min_{\bm{\theta}\in\mathbb{R}^{T}}\,\max_{e\in\mathcal{E}_\text{tr}}\,R_e^\text{MSE}(h_{\bm{\theta}}).
\end{equation}
This problem can equivalently be written in epigraph form as
\begin{equation} \label{eq:obj_local_epigraphnew}
\begin{aligned}
\min_{\bm{\theta}\in\mathbb{R}^{T},\,z\in[0,\,\infty)}\quad &z,\\
\text{s.t.}\quad &R_e^\text{MSE}(h_{\bm{\theta}})\,\le\,z,\quad\forall\,e\in\mathcal{E}_\text{tr}.
\end{aligned}
\end{equation}
In practice, for all $e\in\mathcal{E}_\text{tr}$, we replace $R_e^\text{MSE}$ by its empirical counterpart
\begin{equation}\label{eq:empirical-MSE}
    \hat{R}_e^\text{MSE}(h_{\bm{\theta}})=\frac{1}{n_e}\sum_{i=1}^{n_e}\left[
    \sum_{t=1}^T(y_i^e - \theta_t)^2\,\mathbf{1}_{\{x_i^e\in\mathcal{X}_t\}}\right] =\frac{1}{n_e}\lVert A_e\bm{\theta}-\bm{y}_e\rVert_2^2,
\end{equation}
where $\bm{y}_e\coloneqq(y_1^e,\dots,y_{n_e}^e)^\top$ is the vector of observed responses in environment $e$, and $A_e~\in~\{0,1\}^{n_e\times(M+1)}$ is a binary matrix whose $(i,j)$-th entry is $1$ if observation $i$ falls into region $j$, and $0$ otherwise.
For all $e\in\mathcal{E}_\text{tr}$, the constraint $\hat{R}_e^\text{MSE}(h_{\bm{\theta}}) \leq z$ is thus equivalent to $$\lVert A_e\bm{\theta}-\bm{y}_e\rVert_2^2 \le n_ez,$$ which describes a rotated second-order cone.
Since a rotated second-order cone can be converted into a second-order cone \citep{Alizadeh2003}, 
\eqref{eq:obj_local_epigraphnew} can be represented as a second-order cone program (SOCP) and can be solved using interior‐point methods \citep[][Chapter~$11$]{Nesterov1994, Domahidi2013, Goulart2024, Boyd2004}.
(Appendix~\ref{sec:bcd} describes how to combine this with block-coordinate descent.)

The same line of arguments
extends analogously to the negative reward and the regret, as discussed in the following remark.
\begin{remark}\label{rmk:multiple_risks}
    The same procedure applies if, instead of the \textnormal{MSE}, we use the negative reward or the regret. In this case, in~\eqref{eq:obj_local_epigraphnew}, we replace, for all $e\in\mathcal{E}_\textnormal{tr}$, $R_e^\textnormal{MSE}$ with the corresponding risk. The resulting constraints differ from the \textnormal{MSE} case only by an additive constant $c_e$ independent of $\theta^L$ or $\theta^R$. Specifically, we impose an upper bound of $z+c_e\in[0,\,\infty)$ instead of $z\in[0,\,\infty)$, where $c_e=\mathbb{E}_e[(Y^e)^2]$ for the negative reward and $c_e=\inf_{g\in\mathcal{F}}\,\mathbb{E}_e[(Y^e - g(X^e))^2]$ for the regret. Therefore, if we replace, for all $e\in\mathcal{E}_\textnormal{tr}$, $R_e^\textnormal{NRW}$ (or $R_e^\textnormal{Reg}$) with its empirical counterpart $\hat{R}_e^\textnormal{NRW}$ (or $\hat{R}_e^\textnormal{Reg}$), the optimization problem can still be represented as a second-order cone program. 
\end{remark}

The optimization described  in~\eqref{eq:obj_posthoc}
is not strictly convex, as discussed in the following remark. 
\begin{remark} \label{rem:indeterm}
    After constructing the partition of the input space using standard regression trees, some leaf regions may contain no observations from certain environments. In particular, if a leaf region $\mathcal{X}_t$, $t\in\{1,\,\dots,\,T\}$, contains no observations from the worst-case environment, the MaxRM optimization problem imposes no constraints on the corresponding leaf value $\theta_t$, as long as the chosen value of $\theta_t$ does not increase the risk of another environment to the point that it becomes the new worst-case environment. Therefore, the MaxRM objective does not uniquely determine $\theta_t$ in such leaves, and any value that keeps the risks of all non-worst-case environments below the maximum risk is feasible. In practice,
    we simply use an implicit regularization towards the RF solution of that leaf induced by the choice of initial values (see also Appendix~\ref{app:indet}). 
\end{remark}

In summary, the first method we propose constructs, as RF, 
standard regression trees on bootstrap samples but then re-optimizes the leaf values (`post-hoc') according to the MaxRM objective. 
The trees are weighted equally to obtain the final predictor.
Given the convexity of the optimization procedure described above, this procedure is computationally efficient and scales to datasets with large numbers of variables and observations. We refer to this procedure as MaxRM-RF-posthoc
(we argue in Section~\ref{sec:exp_variants_comparison} that out of all proposed variants, MaxRM-RF-posthoc yields the best trade-off between performance and computational cost).

We now turn to the question whether we can improve the overall performance if we take the MaxRM objective into account when partitioning the input space.

\subsection{Partitioning the input space} \label{sec:partition}
In standard regression trees, introducing a new split does not affect the values of the already existing leaves: they are still optimal for the overall objective, irrespective of the new split and the values set in the resulting leaves.
This is no longer the case under the MaxRM objective. 
To address this, we propose two alternative strategies that determine both the partition of the input space and the corresponding leaf values. 
The first strategy (`local', see Section~\ref{subsec:localnew}) partitions the input space as in standard regression trees, and, after each split, sets the values of only the newly constructed leaves by solving MaxRM.
The second strategy (`global', see Section~\ref{subsec:globalnew}) updates, after each split, all other leaf values, too. 
Figure~\ref{fig:partitionnew} compares both approaches visually.
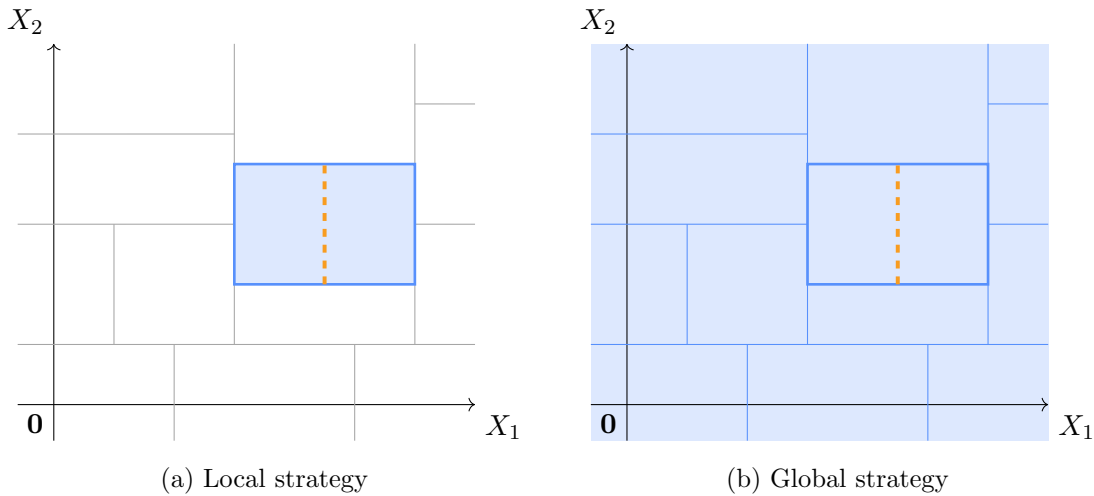
\begin{figure}[t]
  \centering
  \begin{subfigure}{0.5\textwidth}
    \centering
    \begin{tikzpicture}[scale=0.8]
      \def\xmin{-0.6} \def\xmax{7}
      \def\ymin{-0.6} \def\ymax{6}
      \draw[->] (\xmin,0) -- (\xmax,0) node[below right] {$X_1$};
      \draw[->] (0,\ymin) -- (0,\ymax) node[above left] {$X_2$};
      \node[below left] at (0,0) {$\bm{0}$};
      \coordinate (A) at (1,1);
      \coordinate (B) at (2,1);
      \coordinate (C) at (3,1);
      \coordinate (D) at (5,1);
      \coordinate (E) at (6,1);
      \coordinate (F) at (6,2);
      \coordinate (G) at (6,3);
      \coordinate (H) at (6,4);
      \coordinate (I) at (6,5);
      \coordinate (J) at (3,4.5);
      \coordinate (K) at (3,4);
      \coordinate (L) at (3,3);
      \coordinate (M) at (3,2);
      \coordinate (N) at (1,3);
      \draw[gray!70] (\xmin,1) -- (\xmax,1);
      \draw[gray!70] (B) -- (2,\ymin);
      \draw[gray!70] (D) -- (5,\ymin);
      \draw[gray!70] (E) -- (6,\ymax);
      \draw[gray!70] (G) -- (\xmax,3);
      \draw[gray!70] (I) -- (\xmax,5);
      \draw[gray!70] (C) -- (3,\ymax);
      \draw[gray!70] (K) -- (H);
      \draw[gray!70] (M) -- (F);
      \draw[gray!70] (N) -- (A);
      \draw[gray!70] (L) -- (\xmin,3);
      \draw[gray!70] (J) -- (\xmin,4.5);
      \fill[myblue, opacity=0.2] (M) -- (F) -- (H) -- (K) -- cycle;
      \draw[myblue, line width=1pt] (M) -- (F) -- (H) -- (K) -- cycle;
      \draw[dashed, myorange, line width=1.5pt] (4.5,2) -- (4.5,4);
    \end{tikzpicture}
    \caption{Local strategy}\label{fig:partition_local}
  \end{subfigure}\hfill
  \begin{subfigure}{0.5\textwidth}
    \centering
    \begin{tikzpicture}[scale=0.8]
      \def\xmin{-0.6} \def\xmax{7}
      \def\ymin{-0.6} \def\ymax{6}
      \begin{scope}
        \clip (\xmin,\ymin) rectangle (\xmax,\ymax);
        \fill[myblue, opacity=0.2] (\xmin,\ymin) rectangle (\xmax,\ymax);
      \end{scope}
      \draw[->] (\xmin,0) -- (\xmax,0) node[below right] {$X_1$};
      \draw[->] (0,\ymin) -- (0,\ymax) node[above left] {$X_2$};
      \node[below left] at (0,0) {$\bm{0}$};
      \coordinate (A) at (1,1);
      \coordinate (B) at (2,1);
      \coordinate (C) at (3,1);
      \coordinate (D) at (5,1);
      \coordinate (E) at (6,1);
      \coordinate (F) at (6,2);
      \coordinate (G) at (6,3);
      \coordinate (H) at (6,4);
      \coordinate (I) at (6,5);
      \coordinate (J) at (3,4.5);
      \coordinate (K) at (3,4);
      \coordinate (L) at (3,3);
      \coordinate (M) at (3,2);
      \coordinate (N) at (1,3);
      \draw[myblue] (\xmin,1) -- (\xmax,1);
      \draw[myblue] (B) -- (2,\ymin);
      \draw[myblue] (D) -- (5,\ymin);
      \draw[myblue] (E) -- (6,\ymax);
      \draw[myblue] (G) -- (\xmax,3);
      \draw[myblue] (I) -- (\xmax,5);
      \draw[myblue] (C) -- (3,\ymax);
      \draw[myblue] (K) -- (H);
      \draw[myblue] (M) -- (F);
      \draw[myblue] (N) -- (A);
      \draw[myblue] (L) -- (\xmin,3);
      \draw[myblue] (J) -- (\xmin,4.5);
      \draw[myblue, line width=1pt] (M) -- (F) -- (H) -- (K) -- cycle;
      \draw[dashed, myorange, line width=1.5pt] (4.5,2) -- (4.5,4);
    \end{tikzpicture}
    \caption{Global strategy}\label{fig:partition_global}
  \end{subfigure}
  \caption{Comparing, for a two-dimensional covariate space and a fixed partition of the input space, the local (a) and the global strategy (b), see Sections~\ref{subsec:localnew} and~\ref{subsec:globalnew}, respectively. Shaded blue regions are those whose values are updated and the dashed orange line indicates the candidate split. In the local strategy only the child regions of the region being split are updated, whereas in the global strategy the child regions and all other regions are updated jointly.}
  \label{fig:partitionnew}
\end{figure}

\subsubsection{Local leaf value assignments}\label{subsec:localnew}
To determine the next split, we explore candidate regions using a depth-first traversal strategy, as in standard regression trees. Within the region currently being explored, we consider all randomly selected $m_\text{try}\le p$ covariates and their candidate split points. 
For each candidate split of a region into two child regions, we optimize only the two corresponding leaf constants while keeping all other leaf values fixed. If the current partition consists of $M$ regions, when a split creates $M+1$ regions, we construct the parameter vector $\bm{\theta}^\prime\in\mathbb{R}^{M+1}$. We let $\theta^L$ and $\theta^R$ denote the values assigned to the two child regions and determine these by solving, for all $r\in\{\mathrm{MSE},\,\mathrm{NRW},\,\mathrm{Reg}\}$, $\min_{(\theta^L,\,\theta^R)\in\mathbb{R}^2}\,\max_{e\in\mathcal{E}_\text{tr}}\,R_e^r(h_{\bm{\theta}^\prime})$,
analogously to Section~\ref{sec:leaves} (see Appendix~\ref{sec:alternative_alg_local} 
for an alternative method to compute $\theta^L$ and $\theta^R$ that is computationally cheaper when $K\le3$).
We then evaluate the resulting maximum risk across training environments and select the split point that yields the largest reduction. We repeat this process until a stopping criterion is met, for example when a maximum depth is reached or the number of observations in a leaf falls below a predefined threshold. We refer to this procedure as MaxRM-RF-local.

\subsubsection{Global leaf value assignments}\label{subsec:globalnew}
We now traverse the current tree using depth-first search (DFS) as in standard regression trees and Section~\ref{subsec:localnew}. At each internal region visited in DFS order, we consider all randomly selected $m_\text{try}\le p$ covariates and their candidate split points. 
For each candidate split point, we solve a convex optimization problem that jointly assigns values to all leaf regions, not just to the region that is being explored: in contrast to Section~\ref{subsec:localnew}, we now perform leaf value assignment globally rather then locally, and update the constants leaf constants solving, for all $r\in\{\mathrm{MSE},\,\mathrm{NRW},\,\mathrm{Reg}\}$, $\min_{\bm{\theta}^\prime\in\mathbb{R}^{M+1}}\,\max_{e\in\mathcal{E}_\text{tr}}\,R_e^r(h_{\bm{\theta}^\prime})$, which, again, can be solved analogously to Section~\ref{sec:leaves}.
We then evaluate the maximum risk across environments, split the region at the point that yields the largest reduction, and the procedure recurs to the subtrees. 
This procedure comes with a much larger computational cost than MaxRM-RF-posthoc (see Section~\ref{sec:exp_variants_comparison} for an empirical comparison).
We refer to it as MaxRM-RF-global.

One can,
of course, also deviate from DFS and consider all candidate split points across all regions of the current partition. Then, for each candidate, we could solve an optimization problem to update the prediction values of all regions. We would then select both the region and the split point within that region that achieve the largest reduction in maximum risk across environments, split the corresponding region, update the leaf values globally, and repeat this process until no further improvement is possible or a stopping criterion is reached. This strategy, which we call MaxRM-RF-global-NonDFS, prioritizes the splits that provide the greatest improvement in maximum risk and that could be ignored using the DFS traversal order. In our experiments, however, we found that MaxRM-RF-global-NonDFS does not improve over the performance of MaxRM-RF-global but comes with 
an even 
higher computational costs (see Section~\ref{sec:exp_variants_comparison}), so we do not propose to use this in practice. 

\subsection{Choosing tree weights} \label{sec:weights}
Instead of equally weighting the predictions of each tree with $w_b = \tfrac{1}{B}$ for all $b \in \{1, \ldots, B\}$, we can assign the trees different weights according to the MaxRM objective.
Random forest ensembles with non-uniform weights have been proposed before \citep[e.g.,][]{Li10, Beck24}, but, to our knowledge, not in a distribution generalization context. 

Holding the leaf values $\bm{\theta}^{(1)}, \ldots, \bm{\theta}^{(B)}$ fixed, we can choose the weights $\bm{w}$ by optimizing
\begin{equation*}%
    \min_{{\bm{w} \in \Delta_B}}\, \max_{e \in \mathcal{E}_\text{tr}}\, R_e^\text{MSE}(f_{\bm{c},\bm{w}}),
\end{equation*}
where the MSE in environment $e \in \mathcal{E}_\text{tr}$ is given by
$$R_e^\text{MSE}(f_{\bm{c},\bm{w}}) = \mathbb{E}_e\left[ \left(Y^e - \sum_{b=1}^B w_b \, h_{\bm{\theta}^{(b)}}(X^e) \right)^2 \right].$$
Equivalently, we can write the optimization problem in epigraph form and solve it using interior-point methods after plugging in the empirical counterpart of $R_e^\text{MSE}$.
The optimization problem can be formulated analogously with the negative reward or the regret.

In practice, we randomly split the sample into two subsets, ensuring that all environments are represented in both.
We use the first subset to train either a standard random forest or one of the MaxRM-RF variants (post-hoc, local, or global), and the second subset to optimize the weights~$\bm{w}$, while keeping the partition of the input space and the leaf values of each tree fixed. 

If we train a standard RF and then optimize only the tree weights $\bm{w}$, we refer to this as MaxRM-RF-w. If, instead, we first train MaxRM-RF using one of the strategies in Sections~\ref{sec:leaves} or \ref{sec:partition} (post-hoc, local, or global) and then optimize the weights, we denote the resulting methods MaxRM-RF-posthoc-w, MaxRM-RF-local-w, and MaxRM-RF-global-w, respectively.

\subsection{Further details on optimization}\label{sec:algorithms_posthoc}
As we have argued in Section~\ref{sec:leaves},%
~\eqref{eq:obj_posthoc} 
can be solved using interior-point methods. However, the optimization may be challenging in practice when the number of leaves is large (e.g., because the sample size is large or the minimum number of observations per leaf is small) or when the number of environments---and hence the number of constraints---is large. We propose two alternatives to resort to when interior-point methods fail to terminate: an adaptation of the extragradient method \citep{Korpelevich1976} for leaf value optimization %
and a block‑coordinate descent approach.
We provide a detailed description of how these methods can be adapted to our problem in Appendices~\ref{sec:extragradient} and~\ref{sec:bcd}.
We evaluate their performance in %
Appendix~\ref{sec:algorithms_posthoc_evaluation} 
and find that the extragradient method achieves results comparable to interior-point methods, while block-coordinate descent performs %
only slightly less accurately for the chosen block size. 
In practice, in our experiments
we first attempt to solve~\eqref{eq:obj_posthoc} using interior-point methods (e.g., ECOS; \citealp{Domahidi2013}) and only when these solvers fail to terminate we resort to the extragradient method (initialized with the RF solution) or the block-coordinate descent approach. 
In Appendix~\ref{sec:alternative_alg_local}, we discuss faster computation of local split parameters for MaxRM-RF-local when the number of training
environments is small.

\section{Consistency of leaf value assignment
in MaxRM-RF-posthoc
}\label{sec:consistency_maxrm_rf}
We focus on MaxRM random forest with trees constructed using the post-hoc 
adjustment strategy (Section~\ref{sec:leaves}). Specifically, each tree is
trained on a bootstrap sample of the whole training data and its leaf values are re-optimized by solving the empirical version of~\eqref{eq:obj_posthoc}, which extends analogously to the negative reward and the regret. We now establish that, for each tree,
the assigned leaf values are consistent estimators of their population counterparts.

Let therefore $(S_e)_{e\in\mathcal{E}_\text{tr}}$ denote the
samples from each training environment, that is, for all $e\in\mathcal{E}_\text{tr}$, $S_e\coloneqq\{(\bm{x}_i^e,y_i^e)\}_{i=1}^{n_e}$. We draw a bootstrap sample of size $n$ from $\cup_{e\in\mathcal{E}_\text{tr}}S_e$ and denote by $\{S_e^*\}_{e\in\mathcal{E}_\text{tr}}$ the resulting samples in each environment, that is, for all $e\in\mathcal{E}_\text{tr}$, $S_e^*\coloneqq\{(\bm{x}_j^{e,*},y_j^{e,*})\}_{j=1}^{s_e}$,  with $\sum_{e\in\mathcal{E}_\text{tr}}s_e=n$. 
As introduced in Section~\ref{sec:maxrm_rf},
let $h_{\bm{\theta}}$ be a tree with fixed partition $\mathcal{L}$ and vector of leaf values $\bm{\theta}$; we denote the parameter space for $\bm{\theta}$
 by 
$\Theta\subseteq\mathbb{R}^T$. %
For all $\bm{\theta}\in\Theta$ and $e\in\mathcal{E}_\text{tr}$, the population and empirical MSE are
\begin{equation*}
    R_e^\text{MSE}(h_{\bm{\theta}}) = \mathbb{E}_e[(Y^e - h_{\bm{\theta}}(X^e))^2], \qquad \hat{R}_e^{\text{MSE},*}(h_{\bm{\theta}}) = \frac{1}{s_e\vee1} \sum_{j=1}^{s_e} (y_j^{e,*} - h_{\bm{\theta}}(\bm{x}_j^{e,*}))^2.
\end{equation*}
Analogous definitions apply for the negative reward and the regret. We define, for all $r\in\{\text{MSE, NRW, Reg}\}$ and $\bm{\theta}\in\Theta$, the 
maximum population risk and the maximum empirical risk across all training environments, respectively, as 
\begin{equation*}
    R^r(h_{\bm{\theta}}) \coloneqq \max_{e \in \mathcal{E}_\text{tr}}\, R_e^r(h_{\bm{\theta}}), \qquad 
    \hat{R}^{r,*}(h_{\bm{\theta}}) \coloneqq \max_{e \in \mathcal{E}_\text{tr}}\, \hat{R}_e^{r,*}(h_{\bm{\theta}}).
\end{equation*}
For all $r\in\{\text{MSE, NRW, Reg}\}$, the set of \emph{population minimizers} under risk $R^r$ and the set of \emph{post-hoc adjustment estimators} under risk $\hat{R}^{r,*}$ are defined, respectively, as
\begin{equation*}%
    \Theta^r_0\coloneqq\arg\min_{\bm{\theta}\in\Theta}\,R^r(h_{\bm{\theta}}),
    \qquad\hat{\Theta}^{r,*}\coloneqq\arg\min_{\bm{\theta}\in\Theta}\,\hat{R}^{r,*}(h_{\bm{\theta}}).
\end{equation*}
Generally, these sets contain more than one 
vector. 
To state consistency, we thus introduce the notion of deviation between two sets.
    For sets $A,\,B\subseteq\mathbb{R}^T$ and for all $\bm{x}\in\mathbb{R}^T$, the distance from $\bm{x}$ to $B$ is defined as
        $\mathrm{dist}(\bm{x},\,B)\coloneqq\inf_{\bm{x}^\prime\in B}\, \lVert\bm{x}-\bm{x}^\prime\rVert_2^2$.
 We define the \emph{deviation}\footnote{The deviation is also known as the \emph{directed Hausdorff distance}: the Hausdorff distance is defined as $d_H(A,B) \coloneqq \max\{d_{\subset}(A,B),\, d_{\subset}(B,A)\}$.} from $A$ to $B$ as
    \begin{equation*}
        d_{\subset}(A,\,B)\coloneqq\sup_{\bm{x}\in A}\,\mathrm{dist}(\bm{x},\,B).
    \end{equation*}
For all $r\in\{\text{MSE, NRW, Reg}\}$, we show that the post-hoc adjustment estimators $\hat{\Theta}^{r,*}$ are consistent under the following assumptions. 

\begin{assumption}[Compactness of the parameter space]\label{ass:compact_parameter_space} 
    $\Theta$ is compact.
\end{assumption}

\begin{assumption}[Bounded outputs]\label{ass:bdd_output}
    There exists a constant $V < \infty$ such that, for all $e \in \mathcal{E}_{\textnormal{tr}}$, $|Y^e| \leq V$ holds almost surely.
\end{assumption}

\begin{assumption}[Non-vanishing environment proportions]\label{ass:nonvanishing_prop}
    There exist  $(p_e^\star)_{e\in\mathcal{E}_\textnormal{tr}}$ such that, for all $e\in\mathcal{E}_\textnormal{tr}$, $p_e^\star\in(0,1]$,  $\sum_{e\in\mathcal{E}_\textnormal{tr}} p_e^\star=1$, and, for all $e\in\mathcal{E}_\textnormal{tr}$, $n_e/n\to p_e^\star$ as $n\to\infty$.
\end{assumption}
The following Theorem~\ref{thm:consistency} shows that solving the empirical optimization problem with 
the post-hoc adjustment strategy yields a set of values that converges in probability to the set of minimizers of the corresponding population version of the problem.
In addition, the risks converge, too.
A proof can be found in Appendix~\ref{sec:proof_consistency}.
\begin{theorem}[Consistency of post-hoc adjustment estimators]\label{thm:consistency} 
    Let Assumptions~\ref{ass:compact_parameter_space}, \ref{ass:bdd_output} and \ref{ass:nonvanishing_prop} hold true. 
    Then, for all $r\in\{\textnormal{MSE},\, \textnormal{NRW},\, \textnormal{Reg}\}$,
    \begin{equation*}
        d_{\subset}(\hat{\Theta}^{r,*},\,\Theta^r_0) \,\xrightarrow{p} \,0\quad \text{as } n\to\infty,
    \end{equation*}
    In addition, for all $r\in\{\textnormal{MSE},\, \textnormal{NRW},\, \textnormal{Reg}\}$, $\hat{\bm{\theta}}^{r,*}\in\hat{\Theta}^{r,*}$, and $\bm{\theta}^r_0\in\Theta^r_0$, 
    \begin{equation*}
        R^r(h_{\hat{\bm{\theta}}^{r,*}})\,\xrightarrow{p}\,R^r(h_{\bm{\theta}^r_0}) \quad \text{as } n\to\infty.
    \end{equation*}
\end{theorem}

\section{Simulation experiments}\label{sec:sim_experiments}
We now evaluate the performance of our proposed methods on simulated data. All results are fully reproducible using the code available at \url{https://github.com/francescofreni/nldg.git}. Appendix~\ref{sec:additional_experiments} reports further experiments.

We construct training environments by sampling from randomized data generating processes, with and without shifts in the covariate distributions $P^X_e$ across environments. In principle, out-of-distribution performance can be evaluated under many possible shifts across training and test environments. Here, we focus on the convex hull $\mathcal{P}_\text{CVXH}$ of the training environment distributions as the test distributions, as defined in Equation~\eqref{eq:gdro_uncertainty_set}. This choice
is natural in our setting for two reasons. First, Theorem~\ref{thm:gdro_equivalence} shows that minimizing the worst-case risk over $\mathcal{P}_\mathrm{CVXH}$ is equivalent to solving the MaxRM problem in Equation~\eqref{eq:prob_gdro}. Since MaxRM-RF solves this problem, 
no
competing method achieves (in population) a strictly lower worst-case risk on this uncertainty set. Second, the same theorem implies that it suffices to evaluate the worst-case risk on the extreme points of $\mathcal{P}_\mathrm{CVXH}$, namely the training environments themselves. We therefore estimate worst-case performance using new independent samples drawn from these environments.

Section~\ref{sec:exp_variants_comparison} compares the different MaxRM-RF training strategies introduced in Section~\ref{sec:maxrm_rf}. Section~\ref{sec:comparison_gdro} then benchmarks MaxRM-RF against standard random forests \citep{Breiman2001}, group DRO \citep{Sagawa2020}, and magging \citep{Buhlmann2016}.

\subsection{Comparison of different strategies for MaxRM random forest}\label{sec:exp_variants_comparison}
\paragraph{Goal.} We compare the alternative strategies proposed for training MaxRM-RF against each other and against the standard random forest (RF). Here, we evaluate methods in terms of the maximum and pooled mean squared error (MSE) across environments, which allows us to study the trade-off between minimizing the worst-case and the average risk.

\paragraph{Experiment description.} We generate $1{,}000$ training observations, evenly divided across three environments $\mathcal{E}_\text{tr}=\{e_1,\,e_2,\,e_3\}$. For all $e\in\mathcal{E}_\text{tr}$, the covariate $X^e \in \mathbb{R}$ is sampled independently and uniformly from the interval $[-4,\,4]$, and the response $Y^e$ satisfies $$Y^e=\alpha_eX^e\mathbf{1}_{\{X^e\le0\}}+\beta_eX^e\mathbf{1}_{\{X^e>0\}}+\epsilon^e,$$ where $\epsilon^e\sim\mathcal{N}(0,\,\sigma^2)$, with $\sigma=1/2$. The parameters $(\alpha_e,\,\beta_e)$ differ across environments: we use $(\alpha_1,\,\beta_1)=(-1/2,\,4)$, $(\alpha_2,\,\beta_2)=(3,\,1/2)$ and $(\alpha_3,\,\beta_3)=(5/2,\,1)$ for $e_1,\,e_2$ and $e_3$, respectively. An independent test set is generated in the same way. The oracle solutions minimizing the maximum MSE across environments and the pooled MSE correspond to the slope parameters $(\alpha^*,\,\beta^*) = (5/4,\,9/4)$ and $(\alpha_{\mathrm{pool}}^*,\,\beta_{\mathrm{pool}}^*) = (5/3,\,11/6)$, respectively.
Each full experiment---data generation, training, and evaluation---is repeated $20$ times with different random seeds. We report the mean maximum MSE across environments, the pooled MSE, and the runtime\footnote{The runtime is measured on an Apple M4 Pro CPU (14 cores) using 10 parallel workers.}, with $95\%$ confidence intervals.

\paragraph{Methods.}
We compare the following methods:
\textbf{RF} (standard random forest),
    \textbf{MaxRM-RF-posthoc} (RF with post-hoc adjustment, see Section~\ref{sec:leaves}),
\textbf{MaxRM-RF-local} (local tree construction strategy, see Section~\ref{subsec:localnew}),
\textbf{MaxRM-RF-global} and \textbf{MaxRM-RF-global-NonDFS} (global tree construction strategies using the best-reduction heuristic and DFS, resp., see Section~\ref{subsec:globalnew}),
\textbf{MaxRM-RF-w} (RF with non-uniform weights, see Section~\ref{sec:weights}), and 
    \textbf{MaxRM-RF-\{posthoc/local/global\}-w} (MaxRM-RF-posthoc, -local, -global with non-uniform weights, see Section~\ref{sec:weights}).
All forests use $100$ trees and a minimum leaf size of $15$. We use CLARABEL \citep{Goulart2024} as the interior-point solver to solve the convex optimization problems described in Section~\ref{sec:maxrm_rf}. 

\paragraph{Results.}
Table~\ref{tab:sim_diff_methods} shows the maximum MSE across environments and the pooled MSE, with the $95\%$ confidence intervals based on the $20$ repetitions. RF achieves the lowest pooled MSE, as it minimizes the average prediction error. Although MaxRM-RF-local reduces the maximum MSE across training environments compared to RF, jointly updating all the leaf values with the global strategies---MaxRM-RF-global and MaxRM-RF-global-NonDFS---leads to further improvement and achieves the closest performance to the oracle that minimizes the maximum MSE
(the leaf-value assignment is a global procedure with the MaxRM objective). 
The post-hoc adjustment performs similarly to the global variants while being computationally significantly cheaper. In contrast, variants using non-uniform weights for averaging the trees underperform relative to their uniform-weight counterparts:
the potential benefit of non-uniform weights 
comes with the
disadvantage of having access to only 70\% of the training data for fitting the trees.

In the simulation settings we consider in this paper,
MaxRM-RF-posthoc offers the best balance between accuracy and computational cost. It performs comparably to the global variants while requiring only a fraction of the runtime. Therefore, we use MaxRM-RF-posthoc in all subsequent experiments and refer to it simply as MaxRM-RF.

\begin{table}[t]
    \centering
    \caption{Comparison of random forest-based methods in terms of maximum and pooled MSE across training environments, and runtime. Reported values are the sample mean $\pm$ half-length of the $95\%$ confidence interval across 20 repetitions. Both the global and post-hoc variants approach the oracle solution 
    minimizing the maximum MSE across training environments,
    with post-hoc adjustment being fastest.}
    \begin{tabular}{lccc}
        \toprule
        Method & Maximum MSE & Pooled MSE & Runtime (s)\\
        \midrule

       \multicolumn{4}{l}{\textbf{Baseline}}\\
        RF & $24.88 \pm 0.46$ & $12.82 \pm 0.23$ & $0.08 \pm 0.01$ \\
        \addlinespace[0.7ex]
        \midrule
        
        \multicolumn{4}{l}{\textbf{Uniform weights}}\\
        MaxRM\text{-}RF\text{-}posthoc & $16.75 \pm 0.32$ & $13.65 \pm 0.24$ & $1.03 \pm 0.03$ \\
        MaxRM\text{-}RF\text{-}local & $18.06 \pm 0.34$ & $14.71 \pm 0.25$ & $7.27 \pm 0.11$ \\
        MaxRM\text{-}RF\text{-}global & $16.54 \pm 0.28$ & $13.75 \pm 0.24$ & $38.56 \pm 1.00$ \\
        MaxRM\text{-}RF\text{-}global\text{-}NonDFS & $16.55 \pm 0.29$ & $13.74 \pm 0.24$ & $108.88 \pm 3.09$ \\
        \addlinespace[0.7ex]
        \midrule
        
        \multicolumn{4}{l}{\textbf{Non-uniform weights}}\\
        MaxRM\text{-}RF\text{-}w & $20.90 \pm 0.62$ & $13.50 \pm 0.28$ & $0.70 \pm 0.17$ \\
        MaxRM\text{-}RF\text{-}posthoc\text{-}w & $17.12 \pm 0.43$ & $13.89 \pm 0.27$ & $1.73 \pm 0.36$ \\
        MaxRM\text{-}RF\text{-}local\text{-}w & $18.05 \pm 0.28$ & $14.74 \pm 0.24$ & $6.56 \pm 0.33$ \\
        MaxRM\text{-}RF\text{-}global\text{-}w & $17.24 \pm 0.52$ & $14.01 \pm 0.31$ & $25.94 \pm 0.88$ \\
        MaxRM\text{-}RF\text{-}global\text{-}NonDFS\text{-}w & $17.11 \pm 0.46$ & $13.97 \pm 0.30$ & $65.52 \pm 2.09$ \\
        \midrule

        \multicolumn{4}{l}{\textbf{Oracle solutions}}\\
        Oracle (for maximum MSE) & $16.58$ & $13.92$ & $-$\\
        Oracle (for pooled MSE) & $25.29$ & $12.99$ & $-$\\
        \addlinespace[0.7ex]

        \bottomrule
\end{tabular}
    \label{tab:sim_diff_methods}
\end{table}

\subsection{Comparison with group DRO and magging}\label{sec:comparison_gdro}
We simulate data $(X^e, Y^e) \in \mathbb{R}^p \times \mathbb{R}$ from $K\in\mathbb{N}$ training environments $e \in \mathcal{E}_\mathrm{tr} = \{e_1, \ldots, e_K\}$. As before, 
we use the new independent observations from the training environments as the test data. Since for larger numbers of environments $K$ there is more potential for the worst environment to differ from the others, we vary $K$ in the experiments.

We examine three types of shifts across environments. First, Section~\ref{sec:exp_without_shifts} allows shifts only through the 
the conditionals $P^{Y|X}_e$, %
keeping the marginal distributions fixed ($P^X_e = P^X$). 
In this setting, magging is similar to DRoL (distributionally robust learning) proposed by \citet{Wang2025}, which combines environment-specific predictors using weights estimated from (unlabeled) covariates in the test environment.
Section~\ref{sec:exp_with_shifts} allows both the conditionals $P^{Y|X}_e$ and the marginals $P^X_e$ to vary across environments, while keeping a common support $[-1,1]^p$ for $P^X_e$ in all environments. Finally, Section~\ref{sec:exp_identical_env} considers the case in which all environments are statistically identical (for all $e\in\mathcal{E}_\text{tr}$, $P_e = P$). 

For all three simulation settings and all environments $e \in \mathcal{E}_\mathrm{tr}$, the response satisfies 
$$Y^e = f^e(X^e) + \epsilon^e, \quad \epsilon^e \sim \mathcal{N}(0,\sigma^2).$$
In the first two settings, each environment-specific regression function $f^e : [-1,1]^p \to \mathbb{R}$ is drawn independently from a zero-mean Gaussian process prior with squared-exponential kernel, that is, 
$$ f^e \sim \mathcal{GP}\big(0,\,k_\mathrm{sq}\big), \qquad 
k_\mathrm{sq}(x,\,x') = \exp\!\left(-\frac{\lVert x-x'\rVert^2_2}{2\ell^2}\right), $$
with length-scale $\ell \in \mathbb{R}$. Thus, for different environment-specific regression functions $f^e$, the conditional distributions $P^{Y | X}_e$ change across environments.
In the third setting, a single function is drawn from the same distribution but used for all environments.

In this section, we evaluate all methods in terms of the maximum MSE. Section~\ref{sec:app_comparison_grdo_magging_diffrisks} in the appendix shows qualitatively similar results for the negative reward and the regret. To stay in the setting of Theorem~\ref{thm:gdro_equivalence}, we use the same maximum risk (MSE, negative reward, or regret) for evaluation that is used for training the methods based on MaxRM.

\subsubsection{Shifts only in \texorpdfstring{$P^{Y|X}_e$}{PY|Xe}}\label{sec:exp_without_shifts}

\paragraph{Experiment description.}
We consider between $K=2$ and $K=8$ training environments. From each environment, there are $n=2{,}000$ training and $2{,}000$ additional evaluation observations with $p=5$ covariates. The covariates 
are drawn 
as $X^e_{i,j} \stackrel{\text{i.i.d.}}{\sim} \mathrm{Unif}([-1,1])$ ($i\in\{1,\,\dots,\,n\}$, $j\in\{1,\,\dots,\,p\}$),
and responses are generated according to the Gaussian process model described above. The kernel length-scale and the noise level are fixed at $\ell = 1/2$ and $\sigma = 1/4$, respectively.

Group DRO is implemented via the nested gradient-based algorithm of \citet[Algorithm 1]{Sagawa2020} using a neural network predictor. We adapt the implementation and neural network architecture 
used by
\citet{Wang2025}
to optimize both model parameters and group weights using mini-batch stochastic gradient descent instead of full-batch gradient descent.
For magging, we use random forests as the base estimators in each environment. Unlike \citet{Buhlmann2016}, who minimize the maximum negative reward, we minimize the maximum MSE, ensuring a fair comparison when evaluating the maximum MSE. All random-forest-based models (RF, MaxRM-RF, and the environment-specific forests in magging) use $B=100$ trees and a minimum leaf size of 30. Appendix~\ref{sec:exp_hyperparams} investigates the effect of different hyperparameter choices.
We report the average maximum MSE across environments, together with confidence intervals over 100 repetitions. 

\paragraph{Results.}
Figure~\ref{fig:maxmse_noshift} 
shows the maximum MSE across test environments when the covariate distribution is identical across environments. Group DRO performs worst, exhibiting the largest MSE for all numbers of environments. Both magging and MaxRM-RF substantially improve upon RF and group DRO, achieving similarly low maximum MSE values. 
This is to be expected, as in the setting without changes in covariate distribution, both magging and MaxRM-RF minimize the maximum risk (in population). 
We believe that the fact that MaxRM-RF performs slightly better than magging is due to the avoidance of sample splitting (see also Section~\ref{sec:exp_identical_env}).
\begin{table}[t]
\centering
\begin{tabular}{cc}
\begin{minipage}{0.48\linewidth}
    \centering
    \begin{figure}[H]
        \centering
        \includegraphics[width=\linewidth]{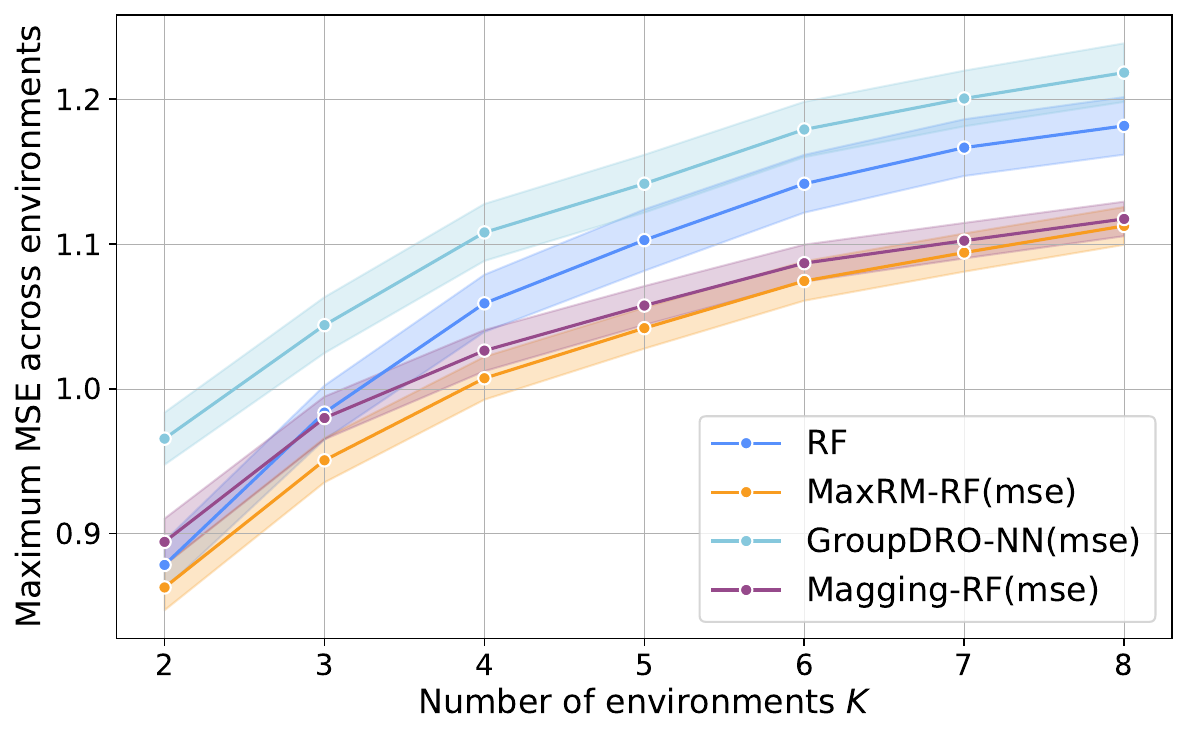}
        \caption{Maximum MSE across environments on test data when $P^X_e = P^X$ for all environments (averaged over 100 repetitions with 95\% confidence intervals). MaxRM-RF and magging perform similarly and outperform group DRO and RF.}
        \label{fig:maxmse_noshift}
    \end{figure}
\end{minipage}
&
\begin{minipage}{0.48\linewidth}
    \centering
    \begin{figure}[H]
        \centering
        \includegraphics[width=\linewidth]{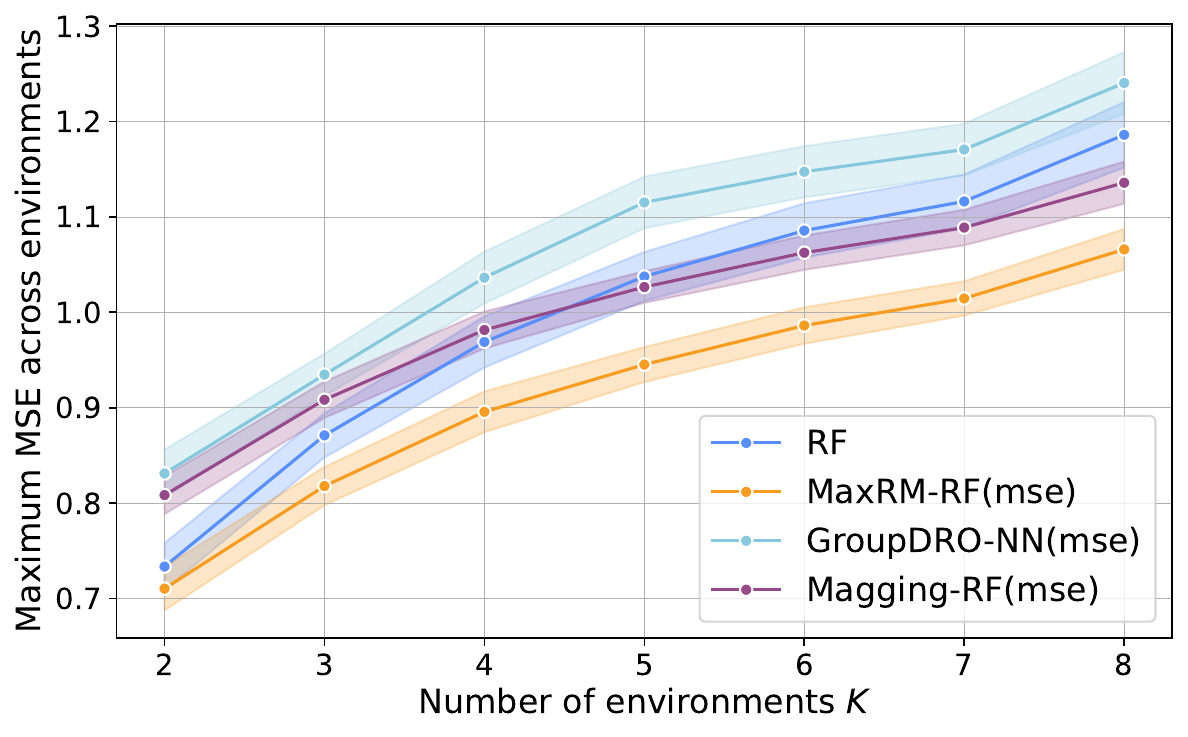}
        \caption{Maximum MSE across environments on test data under shifts in $P^X_e$ (averaged over 100 repetitions with 95\% confidence intervals). MaxRM-RF exhibits the best performance across all numbers of environments.}
        \label{fig:maxmse_shift}
    \end{figure}
\end{minipage}

\end{tabular}
\end{table}

\subsubsection{Shifts in both \texorpdfstring{$P^{Y|X}_e$}{PY|Xe} and \texorpdfstring{$P^X_e$}{PXe}}\label{sec:exp_with_shifts}

\paragraph{Experiment description.} The setup matches the previous experiment, except that 
the environments now have different marginal distributions
$P^X_e$. Here, each marginal distribution $P_e^X$ corresponds to a different Beta distribution. In particular, for each environment $e$, we 
draw parameters $ \alpha_e,\, \beta_e \stackrel{\text{i.i.d.}}{\sim} \mathrm{Unif}([1/2, \, 5/2])$, and sample 
$$ U^{e}_{i,j} \stackrel{\text{i.i.d.}}{\sim} \mathrm{Beta}(\alpha_e,\beta_e), 
        \qquad 
        X^{e}_{i,j} = 2 U^{e}_{i,j} - 1 \in [-1,\,1]. $$
The parameters $(\alpha_e,\beta_e)$ are fixed for that environment across all resamplings. 

\paragraph{Results.}
Figure~\ref{fig:maxmse_shift} 
shows the corresponding results.
In this more challenging setting, magging no longer consistently outperforms RF (in Appendix~\ref{app:comparison_magging}, we prove that in this setting with shifts in $P^X_e$, magging does not necessarily minimize the maximum risk across training environments). Group DRO again yields the weakest performance. MaxRM-RF, however, maintains the lowest maximum MSE across all numbers of environments. Among the examined methods, it demonstrates the strongest robustness against shifts in both the marginal distributions $P^X_e$ and the conditionals $P^{Y|X}_e$ across environments.

\subsubsection{Identical environments}
\label{sec:exp_identical_env}
\paragraph{Goal.}
We now study a setting in which all training environments are statistically identical. In this case, the MaxRM objective \eqref{eq:prob_gdro} coincides with the population version of empirical risk minimization (ERM), and the standard random forest minimizing the pooled MSE is optimal among the considered methods. This experiment checks whether there is a price to pay in terms of maximum MSE when focusing on MaxRM instead of ERM if there are no distribution shifts across environments.

\paragraph{Experiment description.}

We use the same data generating process as in Section~\ref{sec:exp_without_shifts}, with the exception that, for each simulation repetition, a single regression function $f \sim \mathcal{GP}(0,k_\mathrm{sq})$ is drawn for all environments. Thus, for all $e\in\mathcal{E}_\text{tr}$, $f^e = f$ and $P_e = P$, and the environments represent an arbitrary partition of i.i.d.\ observations from $P$. 

We use $K=5$ environments and vary the per-environment sample size $n$. RF, MaxRM-RF, group DRO, and magging are trained as in Section~\ref{sec:exp_without_shifts}. All random-forest-based methods use $B=100$ trees and a minimum leaf size of $15$.
For magging, since each environment-specific random forest is trained on only a $1/K$ fraction of the data, we accordingly set their minimum leaf size to $15/K = 3$. For each $n$, we report the average maximum MSE across $100$ repetitions with $95\%$ confidence intervals.

\paragraph{Results.}
Figure~\ref{fig:maxmse_noshift_equalenv} shows the maximum MSE across environments depending on $n$. MaxRM-RF closely matches RF for all considered $n$. This confirms that, in the absence of shifts across environments, we do not lose performance relative to the random forest. However, if there are shifts across environments, substantial improvements are possible as shown in Section~\ref{sec:exp_without_shifts} and Section~\ref{sec:exp_with_shifts}.

In contrast, magging yields higher maximum MSE values. It trains separate random forests in each environment and then reweights them, so each forest is based on a fraction of the whole dataset. 
Group DRO performs worst overall. 

\begin{figure}[t]
    \centering
        \includegraphics[width=0.65\linewidth]{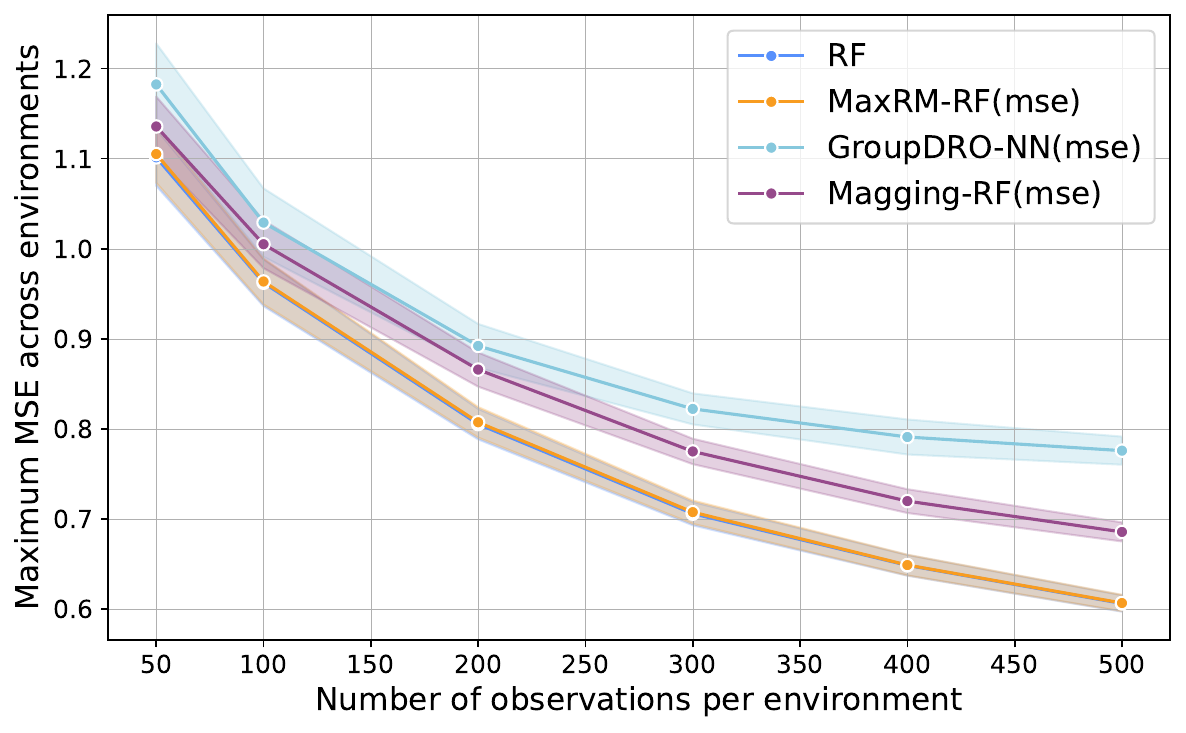}
    \caption{Maximum MSE across test environments when all environments are identical, that is, $P_e = P$ for all $e$. Curves show the average maximum MSE across environments, with shaded areas indicating $95\%$ confidence intervals over $100$ repetitions. 
    Since, in this setting, the MaxRM and ERM objectives coincide (in population), we expect MaxRM-RF and RF to perform similarly. Indeed,
    MaxRM-RF closely matches RF for all sample sizes (the blue line lies mostly underneath the orange line).
    }
    \label{fig:maxmse_noshift_equalenv}
\end{figure}

\section{Application on California housing data}
\label{sec:ca_housing}
We illustrate the behavior of MaxRM-RF on real data by analyzing the California housing dataset \citep{Pace1997}, 
see also 
\citet{Gnecco2024}.
The data contain 20,640 block groups from the 1990 U.S.\ Census, each providing the median house value together with six demographic and housing covariates (median household income, median house age, average number of rooms and bedrooms, population, and household size). 
Local housing markets differ across the state and we treat counties as distinct environments. We divide the 25 largest counties into five spatially coherent folds, as shown in Figure~\ref{fig:ca_map}.
\begin{figure}[t]
    \centering
    \includegraphics[width=0.45\linewidth]{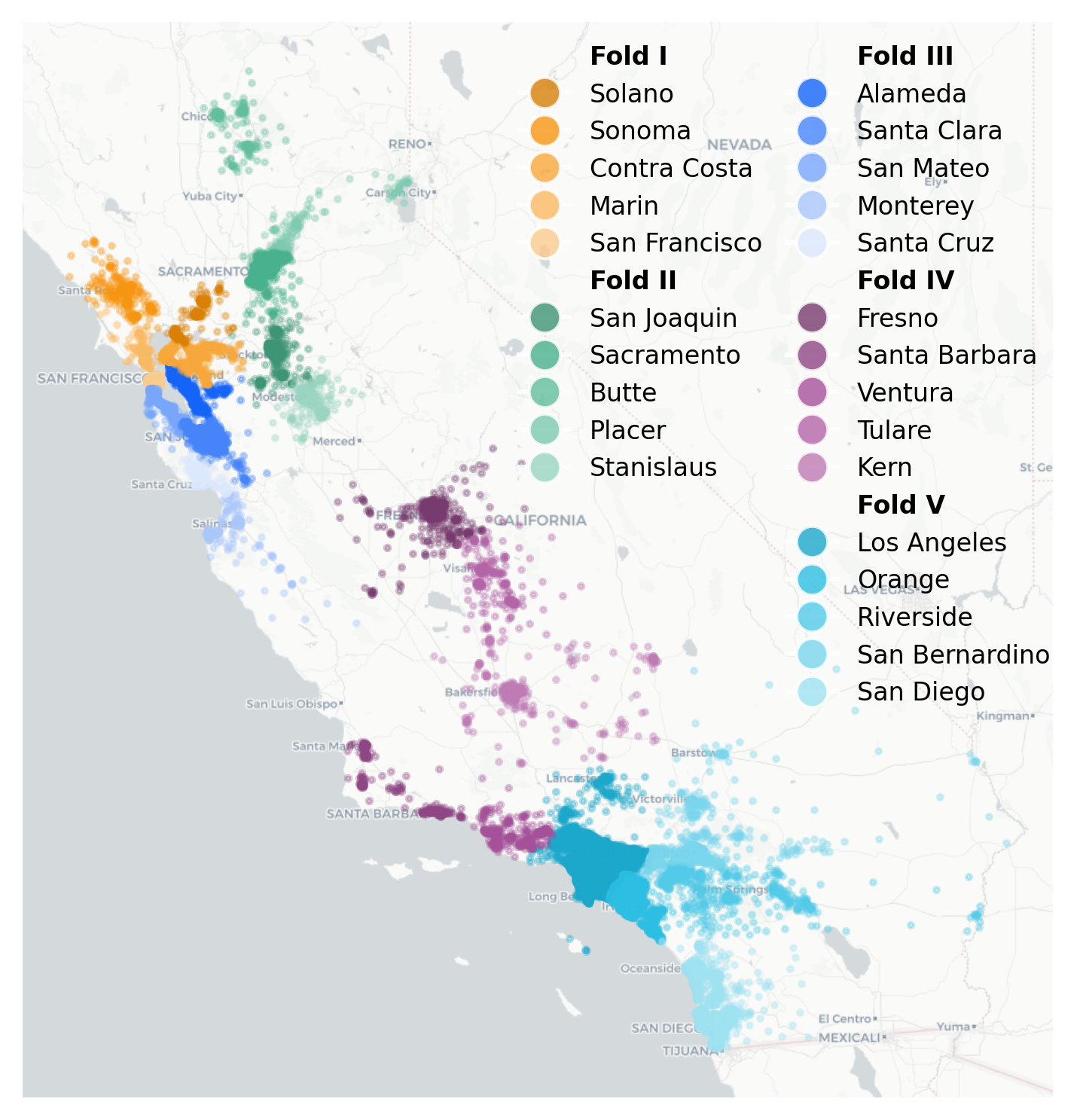}
    \caption{Geographic partition of the California housing dataset into counties. Colors represent the county and cross-validation fold. Map created using OpenStreetMap data and rendered via CARTO. \copyright OpenStreetMap contributors. \copyright CARTO.}
    \label{fig:ca_map}
\end{figure}

We now 
compare MaxRM-RF with post-hoc adjustment (Section \ref{sec:leaves}), 
standard random forests (RF), linear regression (LR), magging, and group DRO on
the following prediction task.
For each fold, all methods are trained on the remaining 20 counties and 
used for prediction
on the five held-out counties. Performance is quantified by the maximum test MSE %
across the held-out counties
and can thus be considered a metric for
worst-case performance under distribution shift.
The results are given in Table~\ref{tab:ca_res} for all methods using the MSE as the risk for training. The results for the other risks are given in Appendix~\ref{app:ca:extended}, Table~\ref{tab:app:ca_res}. For each method, we test whether it significantly improves on RF using a permutation test (gray shading indicates significance), and highlight the best method in bold. 
\begin{table}[t]
    \centering
    \caption{
        Maximum test MSE over five held-out counties. For each held-out environment, the best performing method is highlighted in bold. Cells shaded in gray indicate that the method performs better than RF, testing at a significance level of 0.05 with a permutation test and Bonferroni correction. MaxRM-RF(mse) performs best in four of five folds (significantly for three of five). 
    }
    \label{tab:ca_res}
    \begin{tabular}{lccccc}
    \toprule
    Fold & LR & RF & Magging-RF(mse) & GroupDRO-NN & MaxRM-RF(mse) \\
    \midrule
    I & $11.978$ & $1.269$ & $1.474$ & $1.351$ & \cellcolor{gray!25}\bm{$1.180$} \\ %
    II & $0.605$ & \bm{$0.525$} & $0.950$ & $0.919$ & $0.653$ \\ %
    III & $1.226$ & $0.765$ & $0.754$ & $1.005$ & \cellcolor{gray!25}\bm{$0.586$} \\ %
    IV & $0.962$ & $0.785$ & $1.352$ & $1.366$ & \bm{$0.758$} \\ %
    V & $0.714$ & $0.508$ & $0.944$ & $0.768$ & \cellcolor{gray!25}\bm{$0.490$} \\ %
    \bottomrule
    \end{tabular}
\end{table}

Across the five folds, MaxRM-RF(mse) achieves the lowest worst-case test MSE in four of the five held-out county groups, with three improvements 
being
statistically significant. Unlike RF and linear regression, which minimize the pooled error and do not guard against the worst-case environment, MaxRM-RF(mse) explicitly controls the highest-risk training environments. 
Neither magging nor group DRO improves on RF in this setting (magging is not designed 
to accommodate shifts in the covariate distribution).

Figure~\ref{fig:ca-housing-pairwise} in Appendix~\ref{app:ca:pairwise} provides a pairwise “train-on-A, test-on-B’’ analysis across all counties. Several counties—most notably Marin, San Francisco, and San Mateo—are consistently harder to predict when used as test environments, while diagonal (within-county) errors remain comparatively stable. 
Fold II consists of counties with relatively smaller residuals and prediction errors, suggesting that they do not dominate the maximum-risk objective, which may explain why RF outperforms MaxRM-RF(mse) in that fold.

We also assess robustness under less structured shifts by drawing 200 random test sets of five counties. MaxRM-RF(mse) 
outperforms RF in maximum test MSE in 126 cases (p-value of a binomial test equals 0.0001), indicating that its advantages persist beyond spatially grouped folds. 

In summary, 
both experiments (using spatial folds and randomly chosen folds) indicate that 
the method is able to better
guard against worst-case performance in the presence of heterogeneous county-level distributions than existing methods.

\section{Summary and future work}\label{sec:summary}
We study the regression setting with data collected from multiple and heterogeneous environments. Since ERM can fail in this setting, we consider the maximum risk minimization (MaxRM) framework and aim to maintain robust performance across multiple environments by minimizing the maximum risk over the observed training environments.
We propose MaxRM random forest, a modification of standard random forest that, instead of minimizing the pooled MSE over all observations, aims to 
minimize the maximum risk---MSE, negative reward, or regret---across training environments. Compared to previous group DRO implementations relying on neural networks our method performs better empirically, and,
unlike the magging estimator \citep{Buhlmann2016}, we allow the covariate distribution to vary across environments. We introduce 
different strategies 
to build MaxRM random forests, and, for the post-hoc adjustment strategy, we show that the estimators obtained from the empirical optimization problem converge to the minimizers of the population version of the problem (Theorem~\ref{thm:consistency}). Finally, we adapt the extragradient method and develop a block-coordinate descent approach to solve the post-hoc adjustment problem when interior-point methods fail to terminate. 
We extend existing generalization guarantees and prove consistency results for our method.

We now describe possible directions for future research. First, while Setting~\ref{setting:setting_maxrm} assumes homoscedastic noise within each environment, future work could relax this assumption. Another direction is to increase the flexibility of the predictor: in this work we consider standard regression trees with piecewise-constant leaf predictions; future work could instead consider trees that fit a linear regression model in each leaf \citep{Quinlan1992}; this would yield a more flexible predictor and the resulting optimization problem would remain convex. Moreover, while our current implementation of the extragradient method uses a single step size $\gamma$ for both the minimization and maximization updates (see Algorithm~\ref{alg:extragradient}), future work could explore variants with separate step sizes (see, e.g., \citealp{Li2022}). 
Finally, a natural direction for future work is to extend MaxRM random forest to classification settings, that is, to cases where the response takes values in $\{1,\,\dots,\,C\}$, with $C\in\{2,\,3,\,\dots\}$. Standard classification trees measure the \emph{impurity} of leaf regions through, for instance, the cross-entropy loss. We believe it should be possible, in principle, to minimize the maximum expected cross-entropy loss across environments using our proposed method. If the resulting optimization problem cannot be formulated as an SOCP, then instead of relying on interior-point solvers, it may be possible to  
perform the post-hoc adjustment using the Mirror Prox algorithm \citep[see, e.g.,][]{Guo2025}.

\acks{We thank Elliot Beck for helpful discussions on weighted trees in random forests. During this work, LK was supported by the ETH AI Center through an ETH AI Center doctoral fellowship.}

\vskip 0.2in
\bibliography{references}

\appendix
\clearpage

\addcontentsline{toc}{section}{Appendix}

\listofappendices
\vskip 0.75in

\appsection{Proofs}\label{app:proofs}
\subsection{Additional results}
\begin{lemma}[Empirical Rademacher complexity of piecewise-constants]\label{lemma:rademacher}
    Consider a fixed partition of the input space $\mathbb{R}^p$, denoted by $\mathcal{L}\coloneqq\{\mathcal X_1,\dots,\mathcal X_T\}$. For all $\bm{x}\in\mathbb{R}^p$ and $\bm{\theta}\in\Theta\subseteq\mathbb{R}^T$, define $h_{\bm{\theta}}(\bm{x}) \coloneqq \sum_{t=1}^T \theta_t \mathbf{1}_{\{\bm{x} \in \mathcal{X}_t\}}$, where, for all $t\in\{1,\,\dots,\,T\}$, $|\theta_t|\le W$, $W<\infty$. Define $\mathcal H\coloneqq\{\boldsymbol x\mapsto h_{\boldsymbol\theta}(\boldsymbol x):\boldsymbol\theta\in\Theta\}$. Then, the empirical Rademacher complexity of $\mathcal{H}$ with respect to a sample $S=\{\bm{x}_1,\ldots,\bm{x}_m\}$ of size $m$ satisfies
    \begin{equation*}
        \hat{\mathfrak R}_{m}(\mathcal H; S) \le W\sqrt{\frac{T}{m}}.
    \end{equation*}
\end{lemma}

\begin{proof}
Let $\bm{\sigma}=(\sigma_1,\,\dots,\,\sigma_{m})$ be independent Rademacher random variables, that is, for all $i\in\{1,\,\dots,\,m\}$, $P(\sigma_i=+1)=P(\sigma_i=-1)=1/2$. Then, 
\begin{align*}
    \hat{\mathfrak{R}}_{m}(\mathcal{H};S)
    \le\mathbb{E}_{\bm{\sigma}}\left[\frac{1}{m}\sup_{\bm{\theta}\in\Theta}\,\left|\,\sum_{i=1}^{m}\sigma_i h_{\bm{\theta}}(\bm{x}_i)\,\right| \right]
    =\mathbb{E}_{\bm{\sigma}}\left[\frac{1}{m}\sup_{\bm{\theta}\in\Theta}\,\left|\,\sum_{t=1}^T\theta_t\sum_{i:\bm{x}_i\in\mathcal{X}_t}\sigma_i\,\right|\,\right].
\end{align*}
Then, by the triangle inequality and since, for all $\bm{\theta}\in\Theta$ and $t\in\{1,\,\dots,\,T\}$, $|\theta_t|\le W$,
\begin{equation*}
    \hat{\mathfrak{R}}_{m}(\mathcal{H};S)
    \le\frac{W}{m}\sum_{t=1}^T\mathbb{E}_{\bm{\sigma}}\left[\,\left|\,\sum_{i:\bm{x}_i\in\mathcal{X}_t}\sigma_i\,\right| \,\right].
\end{equation*}
For all $t\in\{1,\,\dots,\,T\}$, by the Cauchy-Schwarz 
inequality and the independence of the Rademacher random variables,
\begin{equation*}
    \mathbb{E}_{\bm{\sigma}}\left[\,\left|\,\sum_{i:\bm{x}_i\in\mathcal{X}_t}\sigma_i\,\right|\,\right]
    \le\sqrt{\mathbb{E}_{\bm{\sigma}}\left[\,\,\sum_{i:\bm{x}_i\in\mathcal{X}_t}\sigma_i^2 +2\sum_{\substack{{i,j:\,i<j,}\\ \bm{x}_i,\bm{x}_j\in\mathcal{X}_t}}\sigma_i\sigma_j \,\right]}
    =\sqrt{|\mathcal{X}_t|},
\end{equation*}
where $|\mathcal{X}_t|\coloneqq\sum_{i:\bm{x}_i\in\mathcal{X}_t}1$. Using again the Cauchy-Schwarz inequality, we get
\begin{equation*}
    \hat{\mathfrak{R}}_{m}(\mathcal{H};S)
    \le\frac{W}{m}\sum_{t=1}^T\sqrt{|\mathcal{X}_t|}
    \le\frac{W}{m}\sqrt{T\sum_{t=1}^T |\mathcal{X}_t|}
    =\frac{W}{m}\sqrt{Tm}=W\sqrt{\frac{T}{m}}.
\end{equation*}
This completes the proof of Lemma~\ref{lemma:rademacher}. %
\end{proof}

\begin{lemma}[Lipschitzness of quadratic losses]\label{lemma:lipschitz}
    For all 
    $c\in\mathbb{R}$ and
    $(y,h)\in\mathbb{R}\times\mathbb{R}$, let $\phi_c(y,h) \coloneqq (y-h)^2 + c$. Fix $V,W<\infty$. For all $c\in\mathbb{R}$ and $y\in\mathbb{R}$ 
    such that $|y|\le V$, the map $h\mapsto\phi_c(y,h)$ is $2(V+W)$-Lipschitz on $[-W,W]$. 
\end{lemma}

\begin{proof}
Let $c \in \mathbb{R}$ and $y\in[-V,V]$. 
Define $q(h)\coloneqq \phi_c(y,h)$. For all $h,h'\in[-W,W]$, by the mean value theorem there exists $\xi$ between $h$ and $h'$ such that 
\begin{align*}
    \left|\,q(h)-q(h')\,\right|=\left|q'(\xi)\right|\cdot\left|h-h'\right|=2\left|y-\xi\right|\cdot\left|h-h'\right|\le2(V+W)\left|h-h'\right|.
\end{align*}
Hence, for all $c \in \mathbb{R}$ and $y\in[-V,V]$, the map $h\mapsto \phi_c(y,h)$ is $2(V+W)$-Lipschitz on $[-W,W]$. This completes the proof of Lemma~\ref{lemma:lipschitz}.
\end{proof}

\subsection{Proof of Theorem~\ref{thm:gdro_equivalence}}\label{sec:proof_gdro_equivalence}
For all $f\in\mathcal{F}$ and $r\in\{\mathrm{MSE},\,\mathrm{NRW}\}$, maximizing $R_P^r(f)$ over $P\in\mathcal{P}_\mathrm{CVXH}$ reduces to maximizing a linear function in $\bm{q}$ over the simplex $\Delta_K$; the result 
then
follows since by Bauer's maximum principle \citep{Bauer1958}, any linear function over the simplex attains its maximum at one of its vertices:
\begin{equation*}
    \max_{P\in\mathcal{P}_\mathrm{CVXH}} R_P^r(f)=\max_{\bm{q}\in\Delta_K}\sum_{k=1}^K q_k R_{P_{e_k}}^r(f)=\max_{e\in\mathcal{E}_\text{tr}}R_e^r(f).
\end{equation*}
With the regret, instead, for all $f\in\mathcal{F}$,
\begin{align*}
    \max_{P\in\mathcal{P}_\mathrm{CVXH}}R_P^\mathrm{Reg}(f) &= \max_{P\in\mathcal{P}_\mathrm{CVXH}}\left\{\mathbb{E}_P\left[(Y-f(X))^2\right]-\inf_{g\in\mathcal{F}}\mathbb{E}_P\left[(Y-g(X))^2\right]\right\} \\[2mm]
    &=\max_{\bm{q}\in\Delta_K}\left\{\sum_{k=1}^K q_k\mathbb{E}_{e_k}\left[(Y-f(X))^2\right]-\inf_{g\in\mathcal{F}}\sum_{k=1}^K q_k\mathbb{E}_{e_k}\left[(Y-g(X))^2\right]\right\} \\[2mm]
    &\le\max_{\bm{q}\in\Delta_K}\sum_{k=1}^K q_k\left\{\mathbb{E}_{e_k}\left[(Y-f(X))^2\right]-\inf_{g\in\mathcal{F}}\mathbb{E}_{e_k}\left[(Y-g(X))^2\right]\right\} \\[2mm]
    &=\max_{\bm{q}\in\Delta_K}\sum_{k=1}^K q_k R_{e_k}^\mathrm{Reg}(f) \\[2mm]
    &=\max_{e\in\mathcal{E}_\text{tr}}R^\mathrm{Reg}_e(f),
\end{align*}
where the last equality follows again from Bauer's maximum principle.
Since $\mathcal{E}_\text{tr}\subseteq\mathcal{P}_\mathrm{CVXH}$, we also have, for all $f\in\mathcal{F}$,
\begin{equation*}
    \max_{P\in\mathcal{P}_\mathrm{CVXH}}R_P^\mathrm{Reg}(f)\ge \max_{e\in\mathcal{E}_\text{tr}} R_e^\mathrm{Reg}(f).
\end{equation*}
Therefore, equality follows. 
This completes the proof of Theorem~\ref{thm:gdro_equivalence}.
\hfill\BlackBox

\subsection{Proof of Proposition~\ref{prop:generalization_test_env}}\label{sec:proof_test_env}
We first show the claim in~\ref{itm:prop_equality} and then prove~\ref{itm:prop_mse_degeneration}. 

\paragraph{Proof of~\ref{itm:prop_equality}.}
Since $P^X$ is invariant across environments by Assumption~\ref{ass:no_changeXdistr}, we omit environment-specific superscripts on $X$. 
Let $\Delta_K^{(1)}$ be the set of
$\bm{q}\in\Delta_K$ such that there is an $e \in \mathcal{E}$ with 
$f^{e}(\cdot) = \sum_{k=1}^K \bm{q}_k f^{e_k}(\cdot)$ (we denote this $e$ by $e^\prime(\bm{q})$). 
For all other $\bm{q}$, i.e., for all $\bm{q} \in \Delta_K^{(2)} := \Delta_K \setminus \Delta_K^{(1)}$, we now construct new environments. To do so, let 
$(e^\prime(\bm{q}))_{\bm{q} \in \Delta_K^{(2)}}$ be environment indices, none of which is in $\mathcal{E}$. For all $\bm{q} \in \Delta_K^{(2)}$, the environment  
$e^\prime(\bm{q})$ has 
 covariate distribution $P^X$,
  an arbitrary but fixed noise distribution with
 noise variance $\sigma^2$ from Assumption~\ref{ass:constant_noisevar}, and conditional mean function 
 $\sum_{k=1}^K \bm{q}_k f^{e_k}(\cdot)$.
Therefore, for all $e\in\Econv$, there exists $\bm{q}\in\Delta_K$ such that $e = e^\prime(\bm{q})$. Depending on $\mathcal{E}$, it is not necessarily true that, for all $\bm{q} \in \Delta_K$, $e^\prime(\bm{q}) \in \Econv$.
We have
$\mathcal{E}_\text{tr}\subseteq\Econv=\mathcal{E}\cap\{e^\prime(\bm{q}):\bm{q}\in\Delta_K\}\subseteq\{e^\prime(\bm{q}):\bm{q}\in\Delta_K\}$.
In particular, for all $f\in\mathcal F$ and $r\in\{\mathrm{NRW},\,\mathrm{Reg},\,\mathrm{MSE}\}$,
\begin{equation} \label{eq:q-Econv}
    \max_{e \in \mathcal{E}_\mathrm{tr}} R_{e}^r(f) \leq \max_{e \in \Econv} R_{e}^r(f) \leq  \max_{\bm{q}\in\Delta_K} R_{e^\prime(\bm{q})}^r(f).
\end{equation}
To prove statement~\ref{itm:prop_equality}, it therefore suffices to show $\max_{e \in \mathcal{E}_\mathrm{tr}} R_{e}^r(f) = \max_{\bm{q}\in\Delta_K} R_{e^\prime(\bm{q})}^r(f)$. We apply Bauer's maximum principle, which states that any continuous convex function defined on a compact convex set attains its maximum at an extreme point of that set. Since the simplex $\Delta_K$ is a compact and convex subset of $\mathbb{R}^K$, we only need to show that, for all $r\in\{\mathrm{NRW},\,\mathrm{Reg},\,\mathrm{MSE}\}$, $\bm{q}\in\Delta_K$, and $f\in\mathcal{F}$, the map $\bm q\mapsto R_{e^\prime(\bm q)}^r(f)$ is continuous and convex.

For all $\bm{q} \in \Delta_K$ and $f\in\mathcal{F}$, the negative reward can be written as \citep[][Appendix~C.1]{Wang2025}
\begin{equation}\label{eq:nrw_qfun}
    R_{e^\prime(\bm{q})}^{\mathrm{NRW}}(f)
    =\mathbb{E}_{P^X}\!\left[f(X)^2\right]
    -2\sum_{k=1}^K q_k\,\mathbb{E}_{P^X}\!\left[f^{e_k}(X)f(X)\right],
\end{equation}
which is continuous and linear---and hence convex---in $\bm{q}$.
For all $\bm{q} \in \Delta_K$ and $f\in\mathcal{F}$, the regret can be written as \citep[][Appendix~C.3]{Wang2025}
\begin{equation}\label{eq:reg_qfun}
    R_{e^\prime(\bm{q})}^{\mathrm{Reg}}(f)
    =\mathbb{E}_{P^X}\!\left[\left(\sum_{k=1}^K q_k f^{e_k}(X)-f(X)\right)^2\right],
\end{equation}
and, analogously, the MSE can be written as
\begin{equation}\label{eq:mse_qfun}
    R_{e^\prime(\bm{q})}^{\mathrm{MSE}}(f)
    =\mathbb{E}_{P^X}\!\left[\left(\sum_{k=1}^K q_k f^{e_k}(X)-f(X)\right)^2\right]+\sigma^2
    =R_{e^\prime(\bm{q})}^{\mathrm{Reg}}(f)+\sigma^2.
\end{equation}
Both \eqref{eq:reg_qfun} and \eqref{eq:mse_qfun} are also continuous and convex in $\bm{q}$ because their Hessians are positive semi-definite \citep[][Appendix~C.3]{Wang2025}.

Therefore, by Bauer's maximum principle, the maximum over $\Delta_K$ is attained at an extreme point of $\Delta_K$ (i.e., a vertex corresponding to some $e\in\mathcal{E}_\text{tr}$). Therefore, we have shown that $\max_{e \in \mathcal{E}_\mathrm{tr}} R_{e}^r(f) = \max_{\bm{q}\in\Delta_K} R_{e^\prime(\bm{q})}^r(f)$. Since the lower and upper bounds coincide, all inequalities in \eqref{eq:q-Econv} must hold with equality.

\paragraph{Proof of~\ref{itm:prop_mse_degeneration}.} If Assumption~\ref{ass:constant_noisevar} does not hold, the equality in \ref{itm:prop_equality} is still satisfied for $r\in\{\mathrm{NRW},\,\mathrm{Reg}\}$, because the objectives in \eqref{eq:nrw_qfun} and \eqref{eq:reg_qfun} do not depend on the noise variance.

We now show that the conclusion can fail for the MSE when the noise variances differ across environments. 
For all $f \in \mathcal{F}$ and $e \in \mathcal{E}_\mathrm{tr}$, the MSE of $f$ in environment $e$ is
\begin{equation*}
    R_e^\mathrm{MSE}(f) 
    = \mathbb{E}_e\bigl[(Y^e - f(X))^2\bigr] 
    = \sigma_e^2 + \mathbb{E}_{P^X}\!\bigl[(f^e(X) - f(X))^2\bigr].
\end{equation*}
Since the second term is nonnegative, for all $f\in\mathcal{F}$ and $e\in\mathcal{E}_\text{tr}$, we have $R_e^\mathrm{MSE}(f) \ge \sigma_e^2$. 
Let the data-generating process be such that 
$e^*\in\mathcal{E}_\text{tr}$ is the unique environment such that $\sigma_{e^*}^2 = \max_{e \in \mathcal{E}_\mathrm{tr}} \sigma_e^2$. Then,
\begin{equation*}
    \min_{f\in\mathcal{F}} \,\max_{e \in \mathcal{E}_\mathrm{tr}}\, R_e^\mathrm{MSE}(f)\ge\sigma_{e^*}^2.
\end{equation*}
For all $e \in \mathcal{E}_\mathrm{tr}$, define the \emph{squared function gap}
between environments $e$ and $e^*$ as
$$\Delta_{e,e^*} \coloneqq \mathbb{E}_{P^X}\!\bigl[(f^e(X) - f^{e^*}(X))^2\bigr].$$
Let the data-generating process be such that
\begin{equation}\label{eq:condition_noise_variance}
    \sigma_{e^*}^2 \ge \max_{e \in \mathcal{E}_\mathrm{tr} \setminus \{e^*\}}\, \left\{ \sigma_e^2 + \Delta_{e,e^*} \right\}
\end{equation}
(this is achievable, e.g., if we let $f\in\mathcal{F}$ and, for all $e\in\mathcal{E}_\text{tr}$, $f^e=f$)
For the noisiest training environment $e^*$, the MSEs of its regression function $f^{e^*}$ are 
\begin{equation*}
    R_{e^*}^\text{MSE}(f^{e^*})=\sigma_{e^*}^2,\qquad 
    R_e^\text{MSE}(f^{e^*})=\sigma_e^2+\Delta_{e,\,e^*},\,\,\forall e\ne e^*.
\end{equation*}
Under condition~\eqref{eq:condition_noise_variance}, we have
\begin{equation*}
    \max_{e\in\mathcal{E}_\text{tr}}\,R_e^\text{MSE}(f^{e^*})=\sigma_{e^*}^2.
\end{equation*}
Hence, in this setting, minimizing the maximum MSE across training environments corresponds to learning $f^{e^*}$, the regression function of the noisiest training environment. (This behavior extends the degeneration behavior described by \citet{Mo2024} in the linear case to the nonlinear case.)

Since Assumption~\ref{ass:constant_noisevar} does not hold, we can construct
an environment $e' \in \Econv$ such that $f^{e'} = f^{e^*}$ 
(choosing $\bm{q}\in\Delta_K$ such that $q_{k^*}=1$ for $k^*$ such that $e^* = e_{k^*}$, and $q_k = 0$ for all $k \neq k^*$)
but $\sigma_{e'}^2 > \sigma_{e^*}^2$. 
For the
predictor $f^{e^*}$,
\begin{equation*}
    \max_{e' \in \Econv} R_{e'}^\mathrm{MSE}(f^{e^*}) \ge \sigma_{e'}^2 > \sigma_{e^*}^2 = \max_{e\in\mathcal{E}_\text{tr}}\,R_e^\text{MSE}(f^{e^*}),
\end{equation*}
which violates the equality \eqref{eq:prop_5} in Proposition~\ref{prop:generalization_test_env}.
\\ \\
This completes the proof of Proposition~\ref{prop:generalization_test_env}. \hfill\BlackBox

\subsection{Proof of Theorem~\ref{thm:consistency}}\label{sec:proof_consistency}
We divide the proof into three parts and show that, for all $r\in\{\text{MSE},\, \text{NRW},\, \text{Reg}\}$, as $n\to\infty$,
\begin{enumerate}[label=(\roman*)]
    \item\label{itm:uniform_convergence} $\sup_{\bm{\theta} \in \Theta}\, \left|\,R^r(h_{\bm{\theta}}) - \hat{R}^{r,*}(h_{\bm{\theta}})\,\right| \,\xrightarrow{p}\, 0$;
    \item\label{itm:excess_risk_convergence} for all $\hat{\bm{\theta}}^{r,*}\in\hat{\Theta}^{r,*}$ and $\bm{\theta}^r_0\in\Theta^r_0$, we have $R^r(h_{\hat{\bm{\theta}}^{r,*}})\,\xrightarrow{p}\,R^r(h_{\bm{\theta}^r_0})$;
    \item\label{itm:set_convergence} $d_{\subset}(\hat{\Theta}^{r,*},\,\Theta^r_0) \,\xrightarrow{p} \,0$.
\end{enumerate}
We first prove~\ref{itm:uniform_convergence}, and then use it to show~\ref{itm:excess_risk_convergence} and~\ref{itm:set_convergence}.\footnote{Unless stated otherwise and in slight abuse of notation, we include the bootstrap randomness into $\mathbb{E}_e$; furthermore, statements including a bootstrap sample are interpreted as holding almost surely.} 

\paragraph{Proof of~\ref{itm:uniform_convergence}.} For all $r\in\{\text{MSE},\, \text{NRW},\, \text{Reg}\}$, we first show uniform convergence of the maximum empirical risk to the maximum population risk. For all $\bm{\theta} \in \Theta$, we have, 
\begin{equation}\label{eq:upperbound}
\begin{split}
    \left|\,R^r(h_{\bm{\theta}})-\hat R^{r,*}(h_{\bm{\theta}})\,\right| &= \left|\,\max_{e \in \mathcal{E}_{\text{tr}}}\, R_e^r(h_{\bm{\theta}}) - \max_{e \in \mathcal{E}_{\text{tr}}}\, \hat{R}_e^{r,*}(h_{\bm{\theta}})\,\right| \\[2mm]
    &\leq \max_{e \in \mathcal{E}_{\text{tr}}}\, \left|\,R_e^r(h_{\bm{\theta}}) - \hat{R}_e^{r,*}(h_{\bm{\theta}})\,\right|.
\end{split}
\end{equation}
We first show that, for all $r\in\{\mathrm{MSE},\,\mathrm{NRW}\}$ and $e\in\mathcal{E}_\text{tr}$, $\mathbb{E}_e\left[\,\sup_{\bm{\theta} \in \Theta} \,\left|\,R_e^r(h_{\bm{\theta}}) - \hat{R}_e^{r,*}(h_{\bm{\theta}})\,\right|\,\right]\to 0$ as $n\to\infty$, and then show it for $r=\mathrm{Reg}$. 

For all $r\in\{\mathrm{MSE},\,\mathrm{NRW}\}$, define the function classes:
\begin{equation*}
    \mathcal{H} \coloneqq \{\bm{x} \mapsto h_{\bm\theta}(\bm{x}) : \theta\in\Theta\},
    \qquad
    \mathcal{G}^r \coloneqq \{(\bm{x},y) \mapsto \phi^r(y,h_{\bm{\theta}}(\bm{x})) : \bm{\theta} \in \Theta\},
\end{equation*}
where, for all $(\bm{x},y)\in\mathbb{R}^p\times\mathbb{R}$ and $\bm{\theta}\in\Theta$,
\begin{equation*}
    \phi^\text{MSE}(y,\,h_{\bm{\theta}}(\bm{x}))\coloneqq(y - h_{\bm\theta}(\bm{x}))^2, 
    \qquad
    \phi^\text{NRW}(y,\,h_{\bm{\theta}}(\bm{x}))\coloneqq(y - h_{\bm\theta}(\bm{x}))^2-y^2.
\end{equation*}
For all $r\in\{\text{MSE},\, \text{NRW},\, \text{Reg}\}$, $\bm{\theta}\in\Theta$ and $e\in\mathcal{E}_\text{tr}$, let $\hat{R}_e^r(h_{\bm{\theta}})$ be defined as $\hat{R}_e^{r,*}(h_{\bm{\theta}})$ but using $S_e$ instead of $S_e^*$ as in~\eqref{eq:empirical-MSE}. By the triangle inequality and \emph{symmetrization} \citep[e.g.,][Lemma~$26.2$]{Shalev2014}, for all $r\in\{\text{MSE},\, \text{NRW}\}$ and $e\in\mathcal{E}_\text{tr}$:
\begin{equation}\label{eq:symmetrization}
\begin{split}
    \mathbb{E}_e\left[\,\sup_{\bm{\theta} \in \Theta} \,\left|\,R_e^r(h_{\bm{\theta}}) - \hat{R}_e^{r,*}(h_{\bm{\theta}})\,\right|\,\right] &\le 
    \mathbb{E}_e\left[\,\sup_{\bm{\theta} \in \Theta} \,\left|\,R_e^r(h_{\bm{\theta}}) - \hat{R}_e^r(h_{\bm{\theta}})\,\right|+\sup_{\bm{\theta} \in \Theta} \,\left|\,\hat{R}_e^r(h_{\bm{\theta}}) - \hat{R}_e^{r,*}(h_{\bm{\theta}})\,\right|\,\right] \\[2mm]
    &\le2\mathbb{E}_e\left[\hat{\mathfrak{R}}_{n_e}(\mathcal{G}^r;S_e)\right]+2\mathbb{E}_{e}\left[\hat{\mathfrak{R}}_{s_e}(\mathcal{G}^r;S_e^*)\right] \\
\end{split}
\end{equation}
where $\hat{\mathfrak{R}}_{n_e}(\mathcal{G}^r;S_e)$ and $\hat{\mathfrak{R}}_{s_e}(\mathcal{G}^r;S_e^*)$ denote the empirical Rademacher complexities of $\mathcal{G}^r$ with respect to $S_e$ and $S_e^*$, respectively, with $\hat{\mathfrak{R}}_0(\cdot;\emptyset)=0$.
In addition, for all $e\in\mathcal{E}_\text{tr}$, given $S_e$, the bootstrap count $s_e$ follows a binomial distribution: $s_e\sim\mathrm{Bin}(n,p_e)$, where $p_e\coloneqq n_e/n$.

We now apply the contraction lemma 
\citep[see, e.g.,][Lemma~$26.9$]{Shalev2014}. 
By Assumption~\ref{ass:bdd_output}, there exists $V<\infty$ such that, for all $e\in\mathcal{E}_\text{tr}$, $|Y^e|\le V$ almost surely. 
Moreover, since $\Theta$ is compact (Assumption~\ref{ass:compact_parameter_space}), it is bounded. Hence, there exists a constant $W < \infty$ such that, for all $\bm{\theta}\in\Theta$ and $t \in \{1, \ldots, T\}$, we have $|\theta_t| \leq W$; consequently, for all $\bm{x}\in\mathbb{R}^p$, $|h_{\bm\theta}(\bm{x})|\leq W$. 

By Lemma~\ref{lemma:lipschitz}, for all $r\in\{\text{MSE},\, \text{NRW}\}$ and $y\in[-V,V]$, the map $h\mapsto\phi^r(y,h)$ is $2(V+W)$-Lipschitz on $[-W,W]$. 
Applying the contraction lemma to~\eqref{eq:symmetrization} yields, for all $r\in\{\text{MSE},\, \text{NRW}\}$ and $e\in\mathcal{E}_\text{tr}$, 
\begin{equation}\label{eq:contraction}
    \mathbb{E}_e\left[\,\sup_{\bm{\theta} \in \Theta} \,\left|\,R_e^r(h_{\bm{\theta}}) - \hat{R}_e^{r,*}(h_{\bm{\theta}})\,\right|\,\right] \le
    4(V+W)\left(\mathbb{E}_e\left[\hat{\mathfrak{R}}_{n_e}(\mathcal{H};S_e)\right] +\mathbb{E}_{e}\left[\hat{\mathfrak{R}}_{s_e}(\mathcal{H};S_e^*)\right]\right).
\end{equation}
By \eqref{eq:contraction} and Lemma~\ref{lemma:rademacher}, for all $r\in\{\text{MSE},\, \text{NRW}\}$ and $e\in\mathcal{E}_\text{tr}$,
\begin{equation*}
\mathbb{E}_e\left[\,\sup_{\bm{\theta} \in \Theta} \,\left|\,R_e^r(h_{\bm{\theta}}) - \hat{R}_e^{r,*}(h_{\bm{\theta}})\,\right|\,\right] \le 4(V+W)W\sqrt{T} \left(\frac{1}{\sqrt{n_e}}+\mathbb{E}_e\left[\frac{1}{\sqrt{s_e\vee 1}}\right]\right),
\end{equation*}
For all $e\in\mathcal{E}_\text{tr}$ and $\delta\in(0,1)$,
\begin{align*}
    \mathbb{E}_e\left[\frac{1}{\sqrt{s_e\vee 1}}\right] &=
    \mathbb{E}_e\left[\frac{1}{\sqrt{s_e\vee 1}}\mathbf{1}_{\{s_e\ge(1-\delta)np_e\}}\right]+\mathbb{E}_e\left[\frac{1}{\sqrt{s_e\vee 1}}\mathbf{1}_{\{s_e<(1-\delta)np_e\}}\right] \\[2mm]
    &\le\frac{1}{\sqrt{(1-\delta)np_e}}P_e\left(s_e\ge(1-\delta)np_e\right)+P_e\left(s_e<(1-\delta)np_e\right) \\[2mm]
    &\le\frac{1}{\sqrt{(1-\delta)np_e}}+\exp\left\{-\frac{\delta^2}{2}np_e\right\},
\end{align*}
where, in the last inequality, we use the Chernoff bound for the sum of Bernoulli trials \citep[][Theorem~$4.5$]{Mitzenmacher2005}. Therefore, for all $r\in\{\text{MSE},\, \text{NRW}\}$ and $e\in\mathcal{E}_\text{tr}$,
\begin{equation*}
    \mathbb{E}_e\left[\,\sup_{\bm{\theta} \in \Theta} \,\left|\,R_e^r(h_{\bm{\theta}}) - \hat{R}_e^{r,*}(h_{\bm{\theta}})\,\right|\,\right] \le 4(V+W)W\sqrt{T}\left(\frac{1}{\sqrt{n_e}}+\frac{1}{\sqrt{(1-\delta)np_e}}+\exp\left\{-\frac{\delta^2}{2}np_e\right\}\right),
\end{equation*}
which goes to zero as $n\to\infty$. Indeed, by Assumption~\ref{ass:nonvanishing_prop}, for all $e\in\mathcal{E}_\text{tr}$, there exists $p_e^\star\in(0,1]$ such that $n_e/n\to p_e^\star$ as $n\to\infty$, so $n_e\to\infty$ as $n\to\infty$.

When $r=\mathrm{Reg}$, instead, for all $e\in\mathcal{E}_\text{tr}$ and $\bm{\theta}\in\Theta$, 
\begin{align*}
    \left|\,R_e^\mathrm{Reg}(h_{\bm{\theta}}) - \hat{R}_e^\mathrm{Reg}(h_{\bm{\theta}})\,\right|
    &\le\left|\,R_e^\mathrm{MSE}(h_{\bm{\theta}})-\hat{R}_e^\mathrm{MSE}(h_{\bm{\theta}})\,\right| + \left|\,\inf_{g\in\mathcal{H}}\,R_e^\mathrm{MSE}(g)-\inf_{g\in\mathcal{H}}\,\hat{R}_e^\mathrm{MSE}(g)\,\right| \\[2mm]
    &\le\left|\,R_e^\mathrm{MSE}(h_{\bm{\theta}})-\hat{R}_e^\mathrm{MSE}(h_{\bm{\theta}})\,\right|+\sup_{g\in\mathcal{H}}\,\left|\,R_e^\mathrm{MSE}(g)-\hat{R}_e^\mathrm{MSE}(g)\,\right|,
\end{align*}
and analogously for $\left|\,\hat{R}_e^\mathrm{Reg}(h_{\bm{\theta}}) - \hat{R}_e^{\mathrm{Reg},*}(h_{\bm{\theta}})\,\right|$. Hence, for all $e\in\mathcal{E}_\text{tr}$ and $\bm{\theta}\in\Theta$, 
\begin{align*}
    \mathbb{E}_e\left[\,\sup_{\bm{\theta} \in \Theta} \,\left|\,R_e^\mathrm{Reg}(h_{\bm{\theta}}) - \hat{R}_e^{\mathrm{Reg},*}(h_{\bm{\theta}})\,\right|\,\right]
    &\le4\mathbb{E}_e\left[\hat{\mathfrak{R}}_{n_e}(\mathcal{G}^\mathrm{MSE};S_e)\right]+4\mathbb{E}_{e}\left[\hat{\mathfrak{R}}_{s_e}(\mathcal{G}^\mathrm{MSE};S_e^*)\right],
\end{align*}
and the remaining steps proceed analogously to the cases $r\in\{\mathrm{MSE},\,\mathrm{NRW}\}$, with the only difference that the constant in the analogue of~\eqref{eq:contraction} is $8(V+W)$ instead of $4(V+W)$.

For all $r\in\{\mathrm{MSE},\,\mathrm{NRW},\,\mathrm{Reg}\}$ and $e\in\mathcal{E}_\text{tr}$, we have therefore established that
\begin{equation}\label{eqn:sup_e_to_zero}
    \mathbb{E}_e\left[\,\sup_{\bm{\theta} \in \Theta} \,\left|\,R_e^r(h_{\bm{\theta}}) - \hat{R}_e^{r,*}(h_{\bm{\theta}})\,\right|\,\right]\to 0 
    \quad \text{as } n \to \infty.
\end{equation}
By Markov's inequality, we obtain that $\sup_{\bm{\theta} \in \Theta} \,\left|\,R_e^r(h_{\bm{\theta}}) - \hat{R}_e^{r,*}(h_{\bm{\theta}})\,\right| \xrightarrow{p} 0$ as $n\to \infty$. 
Using the union bound, we have that for all $\epsilon\in(0,\infty)$ and $r\in\{\text{MSE},\, \text{NRW},\, \text{Reg}\}$,
\begin{equation*}
    P\left(\max_{e\in\mathcal{E}_\text{tr}}\,\sup_{\bm{\theta}\in\Theta}\,\left|\,R_e^r(h_{\bm{\theta}}) - \hat{R}_e^{r,*}(h_{\bm{\theta}})\,\right|\geq \epsilon\right)\le \sum_{k=1}^K P_{e_k}\left(\sup_{\bm{\theta}\in\Theta}\,\left|\,R_{e_k}^r(h_{\bm{\theta}}) - \hat{R}_{e_k}^{r,*}(h_{\bm{\theta}})\,\right|\geq \epsilon\right).
\end{equation*}
Since $K$ is a fixed, finite constant, and, by~\eqref{eqn:sup_e_to_zero}, the right-hand side is a finite sum of terms, each of which converges to zero as $n\to\infty$, the left-hand side also vanishes as $n\to\infty$, yielding the desired convergence in probability:
$\max_{e\in\mathcal{E}_\text{tr}}\,\sup_{\bm{\theta}\in\Theta}\,\left|\,R_e^r(h_{\bm{\theta}}) - \hat{R}_e^{r,*}(h_{\bm{\theta}})\,\right|\xrightarrow{p}0$ as $n\to\infty$. 
By \eqref{eq:upperbound}, we conclude that, for all $r\in\{\text{MSE},\, \text{NRW},\, \text{Reg}\}$, 
\begin{equation}\label{eq:uniform_convergence}
    \sup_{\bm{\theta} \in \Theta}\, \left|\,R^r(h_{\bm{\theta}}) - \hat{R}^{r,*}(h_{\bm{\theta}})\,\right| 
    \leq  
\sup_{\bm{\theta}\in\Theta} \max_{e\in\mathcal{E}_\text{tr}} \,\left|\,R_e^r(h_{\bm{\theta}}) - \hat{R}_e^{r,*}(h_{\bm{\theta}})\,\right|
    \,\xrightarrow{p} 0\, \quad \text{as } n \to \infty,
\end{equation}
which proves uniform convergence of the maximum empirical risk to the maximum population risk.

\paragraph{Proof of~\ref{itm:excess_risk_convergence}.} For all $r\in\{\text{MSE},\, \text{NRW},\, \text{Reg}\}$, $\hat{\bm{\theta}}^{r,*}\in\hat{\Theta}^{r,*}$, and 
$\bm{\theta}^r_0\in\Theta^r_0$, we now turn to the excess risk, 
defined as $R^r(h_{\hat{\bm{\theta}}^{r,*}})-R^r(h_{\bm{\theta}^r_0})$, which is non-negative since ${\bm{\theta}}^{r}_0$ minimizes ${R}^{r}$. 
Compactness of $\Theta$ (Assumption~\ref{ass:compact_parameter_space}) and continuity of $R^r(h_{\bm{\theta}})$ and $\hat R^{r,*}(h_{\bm{\theta}})$ on $\Theta$ ensure the existence of $\hat{\bm{\theta}}^{r,*}$ and $\bm{\theta}^r_0$. 
Consider the excess risk decomposition:
\begin{equation*}
\begin{split}
    R^r(h_{\hat{\bm{\theta}}^{r,*}})
    - R^r(h_{\bm{\theta}^r_0})
    = &\bigl[ R^r(h_{\hat{\bm{\theta}}^{r,*}}) 
    - \hat{R}^{r,*}(h_{\hat{\bm{\theta}}^{r,*}}) \bigr] \\
    &+ \bigl[ \hat{R}^{r,*}(h_{\hat{\bm{\theta}}^{r,*}})
    - \hat{R}^{r,*}(h_{\bm{\theta}^r_0}) \bigr] \\
    &+ \bigl[ \hat{R}^{r,*}(h_{\bm{\theta}^r_0})
    - R^r(h_{\bm{\theta}^r_0}) \bigr].
\end{split}
\end{equation*}
The middle term is non-positive since $\hat{\bm{\theta}}^{r,*}$ minimizes $\hat{R}^{r,*}$, while the first and third terms are both bounded by the uniform deviation:
\begin{equation*}
    R^r(h_{\hat{\bm{\theta}}^{r,*}}) 
    - \hat{R}^{r,*}(h_{\hat{\bm{\theta}}^{r,*}})\le\sup_{\bm \theta\in\Theta}\,\left|\,R^r(h_{\bm \theta})-\hat R^{r,*}(h_{\bm \theta})\,\right|,
\end{equation*}
\begin{equation*}
    \hat{R}^{r,*}(h_{\bm{\theta}^r_0})
    - R^r(h_{\bm{\theta}^r_0})\le\sup_{\bm \theta\in\Theta}\,\left|\,R^r(h_{\bm \theta})-\hat R^{r,*}(h_{\bm \theta})\,\right|.
\end{equation*}
Hence,
\begin{equation*}
    R^r(h_{\hat{\bm{\theta}}^{r,*}})
    - R^r(h_{\bm{\theta}^r_0}) 
    \le 2\sup_{\bm \theta\in\Theta}\,\left|\,R^r(h_{\bm \theta})-\hat R^{r,*}(h_{\bm \theta})\,\right|.
\end{equation*} %
Therefore, by~\eqref{eq:uniform_convergence}, we conclude that
\begin{equation*}
    R^r(h_{\hat{\bm{\theta}}^{r,*}}) \xrightarrow{p} R^r(h_{\bm{\theta}^r_0}), \quad \text{as } n \to \infty.
\end{equation*}

\paragraph{Proof of~\ref{itm:set_convergence}.} For all $r\in\{\text{MSE},\, \text{NRW},\, \text{Reg}\}$, we show consistency of the post-hoc estimators $\hat{\Theta}^{r,*}$. Let $R_{\min}^r\coloneqq\min_{\bm{\theta}\in\Theta}\,R^r(h_{\bm{\theta}})$ and, as in \citet[Section~$5.1$]{DuchiGlynnHongseok2021}, for all $\epsilon\in(0,\,\infty)$, denote by $\Theta_{0,\epsilon}^r\coloneqq\{\bm{\theta}\in\Theta\;:\;\mathrm{dist}(\bm{\theta},\,\Theta^r_0)\le\epsilon\}$ the $\epsilon$-enlargement of $\Theta^r_0$. Define $\delta_\epsilon~\coloneqq~\inf_{\bm{\theta}\in\Theta\setminus\Theta^r_{0,\epsilon}}R^r(h_{\bm{\theta}})-R^r_{\min}>0$. 
Fix an $\epsilon>0$ and an $\eta\in(0,\,\delta_\epsilon/3)$. 
By~\eqref{eq:uniform_convergence}, %
\begin{equation*}
    P\left(\sup_{\bm{\theta}\in\Theta}\,\left|\,\hat{R}^{r,*}(h_{\bm{\theta}})-R^r(h_{\bm{\theta}})\,\right|\le\eta\right)\;\to\;1 \quad \text{as } n\to \infty.
\end{equation*}
If, 
for a given realization,
$\sup_{\bm{\theta}\in\Theta}\,\left|\,\hat{R}^{r,*}(h_{\bm{\theta}})-R^r(h_{\bm{\theta}})\,\right|\le\eta$, 
the following two statements hold
\begin{enumerate}[label=(\roman*)]
    \item 
    $\inf_{\bm{\theta}\in\Theta} \hat{R}^{r,*}(h_{\bm{\theta}})\le R_{\min}^r+\eta$
    (this holds because,
    for all $\bm{\theta}\in\Theta^r_0$, $\hat{R}^{r,*}(h_{\bm{\theta}})\le R_{\min}^r+\eta$),
    \item for all $\bm{\theta}\in\Theta\setminus\Theta^r_{0,\epsilon}$, $\hat{R}^{r,*}(h_{\bm{\theta}})\ge R^r(h_{\bm{\theta}})-\eta\ge R_{\min}^r + \delta_\epsilon-\eta>R_{\min}^r+\tfrac{2}{3}\delta_\epsilon  > R_{\min}^r+2\eta$.
\end{enumerate}
Inequalities 
(i) and (ii)
imply that
\begin{equation*}
    \inf_{\bm{\theta}\in\Theta\setminus\Theta_{0,\epsilon}^r}\,\hat{R}^{r,*}(h_{\bm{\theta}})>\inf_{\bm{\theta}\in\Theta}\,\hat{R}^{r,*}(h_{\bm{\theta}}),
\end{equation*}
that is, the post-hoc adjustment estimators cannot lie outside $\Theta_{0,\epsilon}^r$ (i.e., $d_{\subset}(\hat{\Theta}^{r,*},\,\Theta^r_0)\le\epsilon$). Equivalently, 
\begin{equation*}
    \left\{\sup_{\bm{\theta}\in\Theta}\,\left|\,\hat{R}^{r,*}(h_{\bm{\theta}})-R^r(h_{\bm{\theta}})\,\right|\le\eta\right\}\subseteq 
    \left\{d_{\subset}(\hat{\Theta}^{r,*},\,\Theta^r_0)\le\epsilon\right\}.
\end{equation*}
Therefore,
\begin{equation*}
    P\left(d_{\subset}(\hat{\Theta}^{r,*},\,\Theta^r_0)\le\epsilon\right)\ge
    P\left(\sup_{\bm{\theta}\in\Theta}\,\left|\,\hat{R}^{r,*}(h_{\bm{\theta}})-R^r(h_{\bm{\theta}})\,\right|\le\eta\right)\to 1 \quad \text{as } n\to \infty,
\end{equation*}
which proves that, for all $r\in\{\text{MSE},\, \text{NRW},\, \text{Reg}\}$, $d_{\subset}(\hat{\Theta}^{r,*},\,\Theta^r_0) \,\xrightarrow{p} \,0$ as $n\to \infty$.
This completes the proof of Theorem~\ref{thm:consistency}. \hfill\BlackBox

\appsection{Magging under changing covariate distribution}\label{app:comparison_magging}
We show that, when the covariate distribution changes across environments, the magging estimator \citep{Buhlmann2016} need not minimize the maximum risk across training environments. Namely, there may be a different estimator that yields lower worst-case risk.

Consider the following data-generating setting. Let $\mathcal{E}_\text{tr}:=\{1,\,2,\,3\}$ denote the set of training environments and define
\begin{equation*}
    X^{(1)},X^{(3)}\sim0.9\cdot\mathcal{U}(0,4)+0.1\cdot\mathcal{U}(-4,0), \qquad 
    X^{(2)}\sim0.1\cdot\mathcal{U}(0,4)+0.9\cdot\mathcal{U}(-4,0).
\end{equation*}
For all $e\in\mathcal{E}_\text{tr}$, $Y^{(e)}=c_eX^{(e)}+\epsilon^{(e)}$, with $\epsilon^{(e)}\sim\mathcal{N}(0,1)$ independent of $X^{(e)}$, and slopes $c_1:=3$, $c_2:=-3$, and $c_3:=2$. 

We now show that, for all $x\in[-4,4]$, the optimal predictor of the form 
\begin{equation}\label{eq:predictors_piecewiselinear}
    f(x)=c_+x\mathbf{1}_{\{x\ge0\}}+c_-x\mathbf{1}_{\{x<0\}},
\end{equation}
with 
suitably chosen
$c_+, c_-\in\mathbb{R}$, achieves a lower maximum MSE compared to the optimal predictor of the form
\begin{equation}\label{eq:predictors_magging}
    g(x)=\beta x, \qquad \beta=q_1c_1+q_2c_2+q_3c_3, \qquad q_1,q_2,q_3\in[0,1],\,\sum_{e=1}^3 q_e=1
\end{equation}
(the latter corresponds to the magging estimator).
Indeed, 
with predictors of the form~\eqref{eq:predictors_piecewiselinear}, the mean squared errors in each environment are
\begin{align*}
    \mathbb{E}_1\left[\left(Y^{(1)}-f\left(X^{(1)}\right)\right)^2\right] &= \frac{24}{5}(3-c_+)^2+\frac{8}{15}(3-c_-)^2+1 \\
    \mathbb{E}_2\left[\left(Y^{(2)}-f\left(X^{(2)}\right)\right)^2\right] &= \frac{8}{15}(-3-c_+)^2+\frac{24}{5}(-3-c_-)^2+1 \\
    \mathbb{E}_3\left[\left(Y^{(3)}-f\left(X^{(3)}\right)\right)^2\right] &= \frac{24}{5}(2-c_+)^2+\frac{8}{15}(2-c_-)^2+1.
\end{align*}
Minimizing the maximum MSE across the three environments yields $c_+^*=2.4$, $c_-^*=-2.4$, and the resulting worst-case MSE is $18.28$. With predictors of the form~\eqref{eq:predictors_magging}, instead, minimizing the maximum MSE across environments yields $\beta^*=0$, which achieves a worst-case risk of $49$. Therefore, in this setting, the predictor minimizing the maximum MSE lies outside the class of magging estimators. 

This can be verified empirically, too:
Figure~\ref{fig:comparison_magging} shows the difference between MaxRM-RF and the magging estimator, both
aiming to minimize 
the maximum MSE. The former
is closer to 
the oracle function, while the latter 
is closer to 
the zero function, yielding a larger maximum MSE.

\begin{figure}[t]
    \centering
    \includegraphics[width=0.65\linewidth]{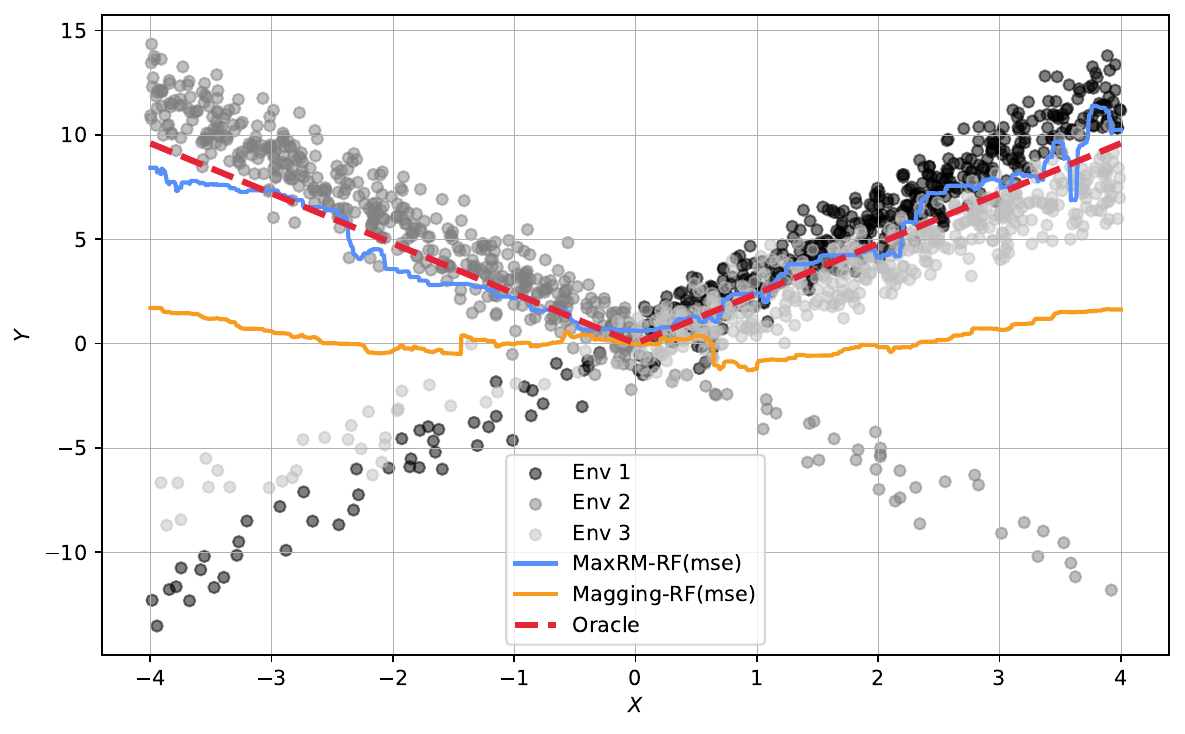}
    \caption{Data from three environments with different covariate distributions. The oracle function (dashed red) corresponds to $f^*(x)=2.4\,x\mathbf{1}_{\{x\ge0\}}-2.4\,x\mathbf{1}_{\{x<0\}}$. MaxRM random forest (solid blue) 
    more closely approximates the oracle solution and yields
    a lower maximum MSE than the magging estimator (solid orange).
    }
    \label{fig:comparison_magging}
\end{figure}

\appsection{Further details on optimization} \label{app:optim}
\subsection{Extragradient method for MaxRM-RF-posthoc}\label{sec:extragradient}
The extragradient method \citep{Korpelevich1976} is an algorithm designed to solve variational inequalities and saddle-point problems, that is, problems where the objective is simultaneously minimized with respect to some variables and maximized with respect to other variables \citep[see, e.g., also][]{Tseng1995, Gorbunov2022, Bach2019, Guo2025, Wang2025_pca}. 
A standard approach for such problems is gradient descent-ascent \citep[GDA;][]{Lin2020}, which, at each iteration, performs a descent step in the minimization variables and an ascent step in the maximization variables. However, GDA can fail to converge even in convex–concave problems \citep[see, e.g.,][and references therein]{Li2022}. To stabilize and improve the accuracy of the parameter updates, the extragradient method introduces an additional (`look-ahead') (half-)step and uses the gradient evaluated at this intermediate step for the full step (we provide details below). %

In our setting, we apply the extragradient method to the optimization problem~\eqref{eq:obj_posthoc}, after replacing $R_e^r$ with its empirical counterpart $\hat{R}_e^r$. 
By Theorem~\ref{thm:gdro_equivalence},~\eqref{eq:obj_posthoc} 
admits the equivalent saddle-point formulation:
\begin{equation*}
    \min_{\bm{\theta}\in\mathbb{R}^T}\,\max_{\bm{q}\in\Delta_K}\,\sum_{k=1}^K q_k \hat{R}_{e_k}^r (h_{\bm{\theta}}).
\end{equation*}
The resulting objective satisfies the standard convex-concave assumptions required for the extragradient method: the objective is convex in $\bm{\theta}$, since each $\hat{R}_e^r$ is quadratic in $\bm{\theta}$, and linear---and hence also concave---in $\bm{q}$. 

Algorithm~\ref{alg:extragradient} outlines the procedure. The projection operator $\Pi_{\Delta_K}$ in the ascent update denotes the projection onto the probability simplex $\Delta_K$, performed using %
Algorithm~1 of \citet{Wang2013}. %
In the experiments (see Appendix \ref{sec:algorithms_posthoc_evaluation}), we use the default hyperparameter values of our current implementation: step size $\gamma=0.1$, maximum number of iterations $T_{\max}=100$, tolerance $\delta=10^{-3}$, and patience $P=5$.

For simplicity, Algorithm~\ref{alg:extragradient}
is written for the empirical MSE. To minimize the maximum negative reward or regret across environments, it suffices to modify the function \textsc{RhatEMSE}
in Algorithm~\ref{alg:extragradient}:
\begin{itemize}
    \item For the negative reward, for all $e\in\mathcal{E}_\text{tr}$, subtract $\frac{1}{n_e} \sum_{i : e_i = e} y_i^2$;
    \item For the regret, one must first fit, for all $e \in \mathcal{E}_\text{tr}$, a separate regression tree by minimizing the MSE. For all $i\in\{1,\,\dots,\,n_e\}$, let $\hat{y}_i^{(e)}$ denote the fitted value for observation $i$ produced by the model trained on environment $e$. Then, in 
    \textsc{RhatEMSE}, 
    subtract $\frac{1}{n_e} \sum_{i : e_i = e} (y_i - \hat{y}_i^{(e)})^2$. 
\end{itemize}
No other changes to the algorithm are needed. 

\begin{algorithm}[t]
\caption{Extragradient method for post-hoc adjustment}\label{alg:extragradient}
\begin{algorithmic}[1]
\Statex\hspace{-\algorithmicindent} \textbf{Input:} Response values and environment labels $\{(y_i, e_i)\}_{i=1}^n$; training environments $\mathcal{E}_\text{tr}$; 
index partition $\{I_t\}_{t=1}^T$ with $I_t := \{i \in \{1, \ldots, n \} \mid  \bm{x}_i\in\mathcal{X}_t \}$; initial values $\bm{\theta}^{(0)} \in \mathbb{R}^{T}$; step size $\gamma$; 
max iterations $T_{\max}$; tolerance $\delta$; patience $P$ 
\Statex\hspace{-\algorithmicindent} \textbf{Output:} Optimized leaf values $\bm{\theta}^*$
\Statex\hspace{-\algorithmicindent} \textbf{Function} $\textsc{RhatEMSE}(\bm{\theta}, e)\gets\frac{1}{n_e} \sum_{t=1}^{T} \sum_{
i\in I_t
} (y_i - \theta_t)^2 \mathbf{1}_{\{e_i = e\}}$
\Statex\hspace{-\algorithmicindent} \textbf{Function} $\textsc{NablaRhatEMSE}(\bm{\theta}, e)
    \gets \left( -\frac{2}{n_e} \sum_{
    i\in I_t
    } (y_i - \theta_t)\mathbf{1}_{\{e_i = e\}} \right)_{t=1}^T$
\State $K \gets |\mathcal{E}_\text{tr}|$ 
\State $\text{no\_improve} \gets 0$, $\quad l^{(0)} \gets +\infty$, $\quad q_k^{(0)} \gets 1/K,\; \forall k\in\{1,\,\dots,\,K\}$
\State $\bm{\theta}^* \gets \bm{\theta}^{(0)}$, $l^* \gets l^{(0)}$
\For{$j = 0,\dots,T_{\max}-1$}
    \State %
    $\hat{R}_e^\text{MSE}(h_{\bm{\theta}^{(j)}})\gets \textsc{RhatEMSE}(\bm{\theta}^{(j)},e), \quad \forall e \in \mathcal{E}_\text{tr}$
    \State $\nabla \hat{R}_e^\text{MSE}(h_{\bm{\theta}^{(j)}}) \gets \textsc{NablaRhatEMSE}(\bm{\theta}^{(j)},e), \quad \forall e \in \mathcal{E}_\text{tr}$
    \State \textcolor{mygray}{$\triangleright$ Extragradient step 1: half step}
    \State $\bm{\theta}^{(j+1/2)} \gets \bm{\theta}^{(j)} - \gamma \sum_{k=1}^K q_k^{(j)} \nabla \hat{R}_{e_k}^\text{MSE}(h_{\bm{\theta}^{(j)}})$
    \State $q_k^{(j+1/2)} \gets \Pi_{\Delta_K} \left(q_k^{(j)} + \gamma \, \hat{R}_{e_k}^\text{MSE}(h_{\bm{\theta}^{(j)}})\right), \quad \forall k\in\{1,\,\dots,\,K\}$
    \State \textcolor{mygray}{$\triangleright$ Evaluate at half step}
    \State %
    $\hat{R}_e^\text{MSE}(h_{\bm{\theta}^{(j+1/2)}})\gets \textsc{RhatEMSE}(\bm{\theta}^{(j+1/2)},e), \quad \forall e \in \mathcal{E}_\text{tr}$
    \State $\nabla \hat{R}_e^\text{MSE}(h_{\bm{\theta}^{(j+1/2)}}) \gets \textsc{NablaRhatEMSE}(\bm{\theta}^{(j+1/2)},e), \quad \forall e \in \mathcal{E}_\text{tr}$
    \State \textcolor{mygray}{$\triangleright$ Extragradient step 2: full step}
    \State $\bm{\theta}^{(j+1)} \gets \bm{\theta}^{(j)} - \gamma \sum_{k=1}^K q_k^{(j+1/2)} \nabla \hat{R}_{e_k}^\text{MSE}(h_{\bm{\theta}^{(j+1/2)}})$
    \State $q_k^{(j+1)} \gets \Pi_{\Delta_K} \left(q_k^{(j)} + \gamma \, \hat{R}_{e_k}^\text{MSE}(h_{\bm{\theta}^{(j+1/2)}})\right), \quad \forall k\in\{1,\,\dots,\,K\}$
    \State \textcolor{mygray}{$\triangleright$ Evaluate at full step}
    \State %
    $\hat{R}_e^\text{MSE}(h_{\bm{\theta}^{(j+1)}})\gets \textsc{RhatEMSE}(\bm{\theta}^{(j+1)},e), \quad \forall e \in \mathcal{E}_\text{tr}$
    \State $l^{(j+1)} \gets \max_{e \in \mathcal{E}_\text{tr}} \hat{R}_e^\text{MSE}(h_{\bm{\theta}^{(j+1)}})$
    \If{$l^{(j+1)} < l^* - \delta$}
        \State $l^* \gets l^{(j+1)}, \quad \bm{\theta}^* \gets \bm{\theta}^{(j+1)}$
        \State $\text{no\_improve} \gets 0$
    \Else
        \State $\text{no\_improve} \gets \text{no\_improve} + 1$
    \EndIf
    \If{$\text{no\_improve} \geq P$}
        \State \textbf{break}
    \EndIf
\EndFor
\State \Return $\bm{\theta}^*$
\end{algorithmic}
\end{algorithm}

\subsection{Block-coordinate descent for MaxRM-RF-posthoc}\label{sec:bcd}
Section~\ref{sec:leaves} describes how solving~\eqref{eq:obj_posthoc} can be tackled using an SOCP.
To reduce the dimensionality of the problem when using interior point solvers, 
we also consider a block-coordinate descent (BCD) strategy for~\eqref{eq:obj_posthoc}. We partition the parameter vector $\bm{\theta}$ into disjoint blocks of fixed size $b$ and, at each iteration, solve a sub-problem by optimizing over a single block while keeping the remaining blocks fixed. 
Each subproblem can be cast in epigraph form and solved as an SOCP with interior-point solvers. We use fixed-sized blocks and update them in a cyclic manner across iterations. Alternatively, one could randomly select a block to update at each iteration or use variable-sized blocks. 
Algorithm~\ref{alg:bcd} summarizes the procedure. 
Although Algorithm~\ref{alg:bcd} is described for the MSE objective, it can be used for the negative reward or regret by adjusting the constraints accordingly, as described in 
Section~\ref{sec:extragradient}. 
As for the extragradient method, in our experiments we select the default hyperparameter values of our implementation: block size $b=15$, maximum number of iterations $T_{\max}=100$, tolerance $\delta=10^{-3}$, and patience $P=1$.
\begin{algorithm}[t]
\caption{Block-coordinate descent for posthoc adjustment}\label{alg:bcd}
\begin{algorithmic}[1]
\Statex\hspace{-\algorithmicindent} \textbf{Input:} Response values and environment labels $\{(y_i, e_i)\}_{i=1}^n$; training environments $\mathcal{E}_\text{tr}$; 
index partition $\{I_t\}_{t=1}^T$ with $I_t := \{i \in \{1, \ldots, n \} \mid  \bm{x}_i\in\mathcal{X}_t \}$; initial values $\bm{\theta}^{(0)}\in\mathbb{R}^{T}$; block size $b$; max iterations $T_{\max}$; tolerance $\delta$; patience $P$ 
\Statex\hspace{-\algorithmicindent} \textbf{Output:} Optimized leaf values $\bm{\theta}^*$
\State Partition $\{1,\dots,T\}$ into $B=\lceil T/b\rceil$ blocks $\{\mathcal{B}_1,\dots,\mathcal{B}_B\}$
\State $\text{no\_improve}\gets0$, $z^{(0)}\gets+\infty$
\State $\bm{\theta}^*\gets\bm{\theta}^{(0)}$, $z^*\gets z^{(0)}$
\For{$j=0,\dots,T_{\max}-1$}
    \State $b_j\gets(j\bmod B)+1$
    \State Fix $\theta_h^{(j+1)}=\theta_h^{(j)}\quad\forall\,h\notin\mathcal{B}_{b_j}$
    \State Solve
        \[
        \begin{aligned}
        \min_{\{\theta_l^{(j+1)}\}_{l\in\mathcal{B}_{b_j}}\in\mathbb{R}^{|\mathcal{B}_{b_j}|},\,z^{(j+1)}\in[0,\,\infty)}\,&z^{(j+1)},\\
        \text{s.t.}\,&
            \frac{1}{n_e}\sum_{t=1}^{T}\sum_{
            i\in I_t}
            (y_i-\theta_t^{(j+1)})^2\,\mathbf{1}_{\{e_i=e\}}
            \le z^{(j+1)},\quad\forall e\in\mathcal{E}_\text{tr}.
        \end{aligned}
        \]
        \indent and update block values $\{\theta_l^{(j+1)}\}_{l\in\mathcal{B}_{b_j}}$ and $z^{(j+1)}$
    \If{$|z^{(j+1)}-z^*|<\delta$}
        \State $\text{no\_improve}\gets\text{no\_improve}+1$
    \Else
        \State $\text{no\_improve}\gets0$
    \EndIf
    \If{$z^{(j+1)}<z^*$}
        \State $z^*\gets z^{(j+1)}$, \quad $\bm{\theta}^*\gets\bm{\theta}^{(j+1)}$
    \EndIf
    \If{$\text{no\_improve}\ge P$}
        \State \textbf{break}
    \EndIf
\EndFor
\State \Return $\bm{\theta}^*$
\end{algorithmic}
\end{algorithm}

\subsection{Alternative computation of local split parameters for MaxRM-RF-local}\label{sec:alternative_alg_local}
To avoid relying on interior-point solvers for the local update step (i.e., for computing $\theta^L$ and $\theta^R$), we exploit the structure of the optimization problem, which, empirically (not shown), is faster than interior point solvers if there are $\leq 3$ environments. 

Suppose the current partition consists of $M$ leaf regions and that the parameter vector resulting after splitting the $j$-th region is $\bm{\theta}^\prime\coloneqq(\theta_1,\,\dots,\,\theta_{j-1},\,\theta^{L},\,\theta^{R},\,\theta_{j+1},\,\dots,\,\theta_M)\in\mathbb{R}^{M+1}$, with $j\in\{1,\dots,M\}$. We wish to determine the values $\theta^L$ and $\theta^R$ assigned to the two resulting regions by solving, for an $r\in\{\mathrm{MSE},\,\mathrm{NRW},\,\mathrm{Reg}\}$, $\min_{(\theta^L,\,\theta^R)\in\mathbb{R}^2}\,\max_{e\in\mathcal{E}_\text{tr}}\,\hat{R}_e^r(h_{\bm{\theta}^\prime})$.

Let therefore $r\in\{\mathrm{MSE},\,\mathrm{NRW},\,\mathrm{Reg}\}$ be fixed. We can characterize the solutions of the problem using the Karush-Kuhn-Tucker (KKT) conditions, which in this case are both necessary and sufficient due to the convexity and differentiability of the objective and constraints, together with the existence of a strictly feasible point (Slater's condition)---see \citet[][Chapter~5]{Boyd2004}. For all $k\in\{1,\,\dots,\,K\}$, we introduce the Lagrange multiplier $\lambda_k$. The Lagrangian reads $$L(\theta^L,\, \theta^R,\, z,\, \lambda_1,\, \dots,\,\lambda_K)=z+\sum_{k=1}^K\lambda_k(\hat{R}_{e_k}^r(h_{\bm{\theta}^\prime})-z).$$
Let $(\theta^{*,\,L},\,\theta^{*,\,R},\,z^*, \{\lambda^*_k\}_{k=1}^K)$ be an optimal primal-dual solution and let the optimal parameter be $\bm{\theta}^*\coloneqq(\theta_1,\,\dots,\,\theta_{j-1},\,\theta^{*,\,L},\,\theta^{*,\,R},\,\theta_{j+1},\,\dots,\,\theta_M)$. The KKT conditions are as follows: 
\begin{enumerate}[label=(\roman*)]
    \item\label{itm:kkt1} Primal feasibility: for all $k\in\{1,\,\dots,\,K\}$, $\hat{R}_{e_k}^r(h_{\bm{\theta}^*})\le z^*$;
    \item\label{itm:kkt2} Dual feasibility: for all $k\in\{1,\,\dots,\,K\}$, $\lambda_k^*\ge0$; 
    \item\label{itm:kkt3} Complementary slackness: for all $k\in\{1,\,\dots,\,K\}$, $\lambda_k^*(\hat{R}_{e_k}^r(h_{\bm{\theta}^*})-z^*)=0$, which implies that if $\lambda_k^*>0$, then $\hat{R}_{e_k}^r(h_{\bm{\theta}^*})=z^*$;
    \item\label{itm:kkt4} Stationarity at the optimum:
    \begin{align*}
        \nabla_{(\theta^{L},\,\theta^R)}L\Big|_{\left(\bm{\theta}^*,z^*,\{\lambda_k^*\}_{k=1}^K\right)}&=\sum_{k=1}^K\lambda_k\nabla_{(\theta^L,\,\theta^R)}\hat{R}_{e_k}^r(h_{\bm{\theta}^\prime})\Big|_{\left(\bm{\theta}^*,\{\lambda_k^*\}_{k=1}^K\right)}=0, \\
        \frac{\partial L}{\partial z}\Big|_{\left(\bm{\theta}^*,z^*,\{\lambda_k^*\}_{k=1}^K\right)}&=1-\sum_{k=1}^K\lambda_k^*=0 \,\Leftrightarrow\,\sum_{k=1}^K\lambda_k^*=1.
    \end{align*}
\end{enumerate}
Define the \emph{active} set at the optimum as $\mathcal{A}\coloneqq\{k\in \{1,\,\dots,\,K \} \mid \hat{R}_{e_k}^r(h_{\bm{\theta}^*})=z^*\}$.
Complementary slackness implies that $\lambda_k^*>0$ only if $k\in\mathcal{A}$.
Hence, the KKT conditions~\ref{itm:kkt2}, \ref{itm:kkt3} and \ref{itm:kkt4} 
are equivalent to
the following statements: $\forall k\in\mathcal A,\;\lambda_k^*\ge 0$, $\forall k\in\{1,\,\dots,\,K\}\setminus\mathcal A,\;\lambda_k^*=0$,
\begin{equation}\label{eq:KKT_grad}
    \sum_{k\in\mathcal{A}}\lambda_k^*=1,\quad \sum_{k\in\mathcal{A}}\lambda_k^* \nabla_{(\theta^{L},\,\theta^R)}\hat{R}_{e_k}^r(h_{\bm{\theta}^\prime})\Big|_{\bm{\theta}^\prime=\bm{\theta}^*}=0.
\end{equation}
For all $k\in\{1,\ldots,K\}$, let $n_k^L$ and $n_k^R$ denote the number of observations from environment $e_k$ that fall in $\mathcal{X}^L$ and $\mathcal{X}^R$, respectively, and let $\mu_k^L$ and $\mu_k^R$ denote the corresponding sample means. Then, $\nabla_{(\theta^L,\,\theta^R)}\hat{R}_{e_k}^r(h_{\bm{\theta}^\prime})=(2n_k^L(\theta^L-\mu_k^L),\,2n_k^R(\theta^R-\mu_k^R))^\top/n_k$ where $n_k$ is the total number of observations from environment $e_k$.

\paragraph{Case $K=1$.} When there is only one environment, as in standard regression trees, the solution is given by $\theta^{*,\,L}=\mu_1^L$,  $\theta^{*,\,R}=\mu_1^R$.

\paragraph{Case $K=2$.} When two environments are present, consider multipliers $(\lambda,1-\lambda)$ with $\lambda\in[0,1]$ and a candidate solution $\bm{\theta}(\lambda)$. The stationarity condition reads $$\lambda\nabla_{(\theta^L,\,\theta^R)}\hat{R}_{e_1}^r(h_{\bm{\theta}(\lambda)})+(1-\lambda)\nabla_{(\theta^L,\,\theta^R)}\hat{R}_{e_2}^r(h_{\bm{\theta}(\lambda)})=0,\quad\lambda\in[0,1].$$
Solving this system yields the parameterized solution  
$$
\theta^{L}(\lambda)=\frac{\lambda \frac{n_1^L}{n_1}\mu_1^L+(1-\lambda)\frac{n_2^L}{n_2}\mu_2^L}{\lambda \frac{n_1^L}{n_1}+(1-\lambda)\frac{n_2^L}{n_2}},\quad 
\theta^{R}(\lambda)=\frac{\lambda \frac{n_1^R}{n_1}\mu_1^R+(1-\lambda)\frac{n_2^R}{n_2}\mu_2^R}{\lambda \frac{n_1^R}{n_1}+(1-\lambda)\frac{n_2^R}{n_2}}.
$$ 
When $\lambda\in\{0,\,1\}$, the solution corresponds to a boundary case where only one environment is active, and the problem reduces to minimizing the risk in that environment. When $\lambda\in(0,\,1)$, both environments are active, and complementary slackness requires that their risks coincide: $\hat{R}_{e_1}^r(h_{\bm{\theta}(\lambda)})=\hat{R}_{e_2}^r(h_{\bm{\theta}(\lambda)})$.
If such a $\lambda$ exists, it can be found by solving this equality with a univariate root finder. Finally, we compare this candidate with the boundary cases $\lambda\in\{0,\,1\}$ and select the one with the smallest maximum risk. We denote the resulting choice by $\lambda^*$ and the corresponding parameters by $(\theta^{*,L},\theta^{*,R})$ (and maximum risk $z^*$).

\paragraph{Case $K\ge 3$.} When three or more environments are present, we use the following result.

\begin{theorem}[Carathéodory]\label{thm:caratheodory}
Let $X$ be a subset of $\mathbb{R}^d$, and let $\bm{q} \in \mathrm{conv}(X)$, the convex hull of $X$. 
Then $\bm{q}$ can be written as a convex combination of at most $d+1$ points in $X$.
\end{theorem}
For all $k \in \mathcal{A}^*$, let 
$g_k := \nabla_{(\theta^L,\,\theta^R)}\hat{R}_{e_k}^r(h_{\bm{\theta}^\prime})\Big|_{\bm{\theta}^\prime=\bm{\theta}^*}$ 
be the environment-specific gradient at the optimum.
Condition~\eqref{eq:KKT_grad} implies that the zero vector lies in the convex hull of these environment-specific gradients, that is, $\bm{0} \in \mathrm{conv}(\{g_k\}_{k\in\mathcal{A}^*})\subseteq \mathbb R^2$. By Theorem~\ref{thm:caratheodory}, there exists a subset $S\subseteq\mathcal{A}^*$ with $|S|\le 3$ such that $\bm{0}\in\mathrm{conv}(\{g_k\}_{k\in S})$. Therefore, it suffices to consider candidate active sets with at most three environments.

We have already considered the cases $|S|=1$ and $|S|=2$. Now suppose that three environments $i,\,j,\,k\in\{1,\ldots,K\}$ are active.
Consider multipliers
$(\lambda_i,\lambda_j,\lambda_k)$ with $\lambda_k:=1-\lambda_i-\lambda_j$ and
$\lambda_i,\lambda_j,\lambda_k\ge 0$ and candidate parameters $\bm{\theta}(\lambda_i,\lambda_j)$.
The stationarity condition becomes
\begin{equation*}
    \lambda_i\nabla_{(\theta^L,\,\theta^R)}\hat{R}_{e_i}^r(h_{\bm{\theta}(\lambda_i,\lambda_j)})+\lambda_j\nabla_{(\theta^L,\,\theta^R)}\hat{R}_{e_j}^r(h_{\bm{\theta}(\lambda_i,\lambda_j)})+(1-\lambda_i-\lambda_j)\nabla_{(\theta^L,\,\theta^R)}\hat{R}_{e_k}^r(h_{\bm{\theta}(\lambda_i,\lambda_j)})=0.
\end{equation*}
Solving this system yields the parameterized solution
\begin{equation*}
    \theta^L(\lambda_i, \lambda_j)=\frac{\lambda_i \frac{n_i^L}{n_i}\mu_i^L+\lambda_j\frac{n_j^L}{n_j}\mu_j^L+(1-\lambda_i-\lambda_j)\frac{n_k^L}{n_k}\mu_k^L}{\lambda_i \frac{n_i^L}{n_i}+\lambda_j\frac{n_j^L}{n_j}+(1-\lambda_i-\lambda_j)\frac{n_k^L}{n_k}},
\end{equation*}
and analogously for $\theta^R(\lambda_i, \lambda_j)$. 
When all three multipliers are strictly positive, complementary slackness requires equality of the risks: $\hat{R}_{e_i}^r(h_{\bm{\theta}(\lambda_i,\lambda_j)})=\hat{R}_{e_j}^r(h_{\bm{\theta}(\lambda_i,\lambda_j)})=\hat{R}_{e_k}^r(h_{\bm{\theta}(\lambda_i,\lambda_j)})$, 
from which we obtain candidate values of $\lambda_i,\,\lambda_j,\,\lambda_k$ and thus of $\theta^L$ and $\theta^R$.

In practice, since the number of active environments is unknown, we proceed as follows:
\begin{enumerate}[label=(\roman*)]
    \item Enumerate all subsets $S\subseteq\{1,\,\dots,\,K\}$ with $|S|\le 3$;
    \item For every $S$, solve for a candidate $(\theta^L,\theta^R,z,\{\lambda_k\}_{k\in S})$
    such that: 
    $$\forall k\in S,\;\;\hat{R}_{e_k}^\text{MSE}(h_{\bm{\theta}^\prime})=z,\;\;\lambda_k>0; \quad \sum_{k\in S}\lambda_k\nabla_{(\theta^L,\,\theta^R)}\hat{R}_{e_k}^\text{MSE}(h_{\bm{\theta}^\prime})=0;\quad \sum_{k\in S}\lambda_k=1;$$
    \item Check that $\forall k^\prime\notin S$, $\hat{R}_{e_{k^\prime}}^\text{MSE}(h_{\bm{\theta}^\prime})\le z$; 
    \item Choose the feasible solution with smallest maximum risk $z$ and denote the resulting parameters by $(\theta^{*,L},\theta^{*,R})$ and the maximum risk by $z^*$.
\end{enumerate}
Empirically, when $K\le 3$, we find that, on the simulated datasets used in our experiments, this method to compute $\theta^{*,\,L}$ and $\theta^{*,\,R}$ in the local update strategy is faster than relying on interior-point solvers.

\appsection{Additional numerical experiments}\label{sec:additional_experiments}

\subsection{Comparison between MaxRM trees and MaxRM random forest}\label{sec:comparison_tree_rf}
\paragraph{Goal.} Theorem~\ref{thm:consistency} shows that, for a single regression tree, the set of post-hoc adjustment estimators converges to the set of population minimizers of the maximum risk across environments. We now aim to illustrate that 
MaxRM-RF reduces variance compared to
MaxRM regression trees (MaxRM-RT).

\paragraph{Experiment description.} We consider the data-generating setting introduced in Section~\ref{sec:exp_variants_comparison}, which admits a known oracle solution minimizing the maximum MSE across environments. We train both a single MaxRM-RT and MaxRM-RF with $100$ trees on $900$ observations, compute predictions on a test set, and evaluate the squared bias relative to the oracle solution and the variance of the predictions.

\paragraph{Results.} Table~\ref{tab:bias_variance} reports the squared bias and variance averaged over $50$ simulation runs for both MaxRM-RT and MaxRM-RF. A single MaxRM regression tree exhibits high variance ($0.60$) compared to the MaxRM-RF ($0.018$). Moreover, the forest also achieves a lower squared bias ($0.007$ compared to $0.035$), since its richer function class can better approximate the oracle solution.

\begin{table}[t]
    \centering
    \caption{Bias--variance decomposition for a single MaxRM regression tree (RT) and for the MaxRM random forest (RF). Reported values are averaged over 50 runs.}
    \begin{tabular}{lcc}
        \toprule
        Method & Bias$^2$ & Variance \\
        \midrule
        MaxRM-RT & $0.035$ & $0.604$ \\
        MaxRM-RF   & $0.007$ & $0.018$ \\
        \bottomrule
    \end{tabular}
    \label{tab:bias_variance}
\end{table}

\subsection{Evaluation of alternative algorithms for post-hoc adjustment}\label{sec:algorithms_posthoc_evaluation}
\paragraph{Goal.} In Appendix~\ref{app:optim}, we introduce two alternative algorithms to solve the post-hoc adjustment problem in~\eqref{eq:obj_posthoc}: the extragradient method (EG; Appendix~\ref{sec:extragradient}) and block-coordinate descent (BCD; Appendix~\ref{sec:bcd}). Here, we aim to evaluate the empirical performance of these two methods against the interior-point method used as the default for the post-hoc adjustment in MaxRM random forest.

\paragraph{Experiment description.} We consider the same data-generating process and same hyperparameters for all forests as in Section~\ref{sec:exp_variants_comparison}. All MaxRM-RF variants aim to minimize the maximum MSE across environments. In addition to RF and MaxRM-RF(mse) using the interior-point solver CLARABEL \citep{Goulart2024}, denoted simply by MaxRM-RF(mse), we evaluate MaxRM-RF(mse) with EG (with 
the default hyperparameter values specified in Appendix~\ref{sec:extragradient}%
) and with BCD (with 
the default hyperparameter values specified in Appendix~\ref{sec:bcd}).
Each full experiment is repeated $20$ times with different random seeds, and we
report the mean maximum MSE across environments, the pooled MSE, and the total runtime\footnote{The runtime is measured on an Apple M4 Pro CPU (14 cores) using 10 parallel workers.}, including $95\%$ confidence intervals. %

\paragraph{Results.} Figure~\ref{fig:xtrgrd} compares the fitted functions for a single simulated dataset.
Both MaxRM-RF(mse) and MaxRM-RF(mse) with EG closely align with the oracle solution that minimizes the maximum MSE across environments. While MaxRM-RF(mse) with BCD improves over RF in terms of maximum MSE, it yields a less accurate approximation than the other two MaxRM methods, possibly due to the small block size. Table~\ref{tab:xtrgrd} reports the average maximum and pooled MSE with $95\%$ confidence intervals across the $20$ repetitions. MaxRM-RF(mse) and MaxRM-RF(mse) with EG are not significantly different in terms of maximum and pooled MSE. Consistent with Figure~\ref{fig:xtrgrd}, MaxRM-RF(mse) with BCD also improves the maximum MSE compared to RF, but its approximation is less accurate, 
possibly
due to the small block size. 
\begin{figure}[t]
    \centering
    \includegraphics[width=0.7\linewidth]{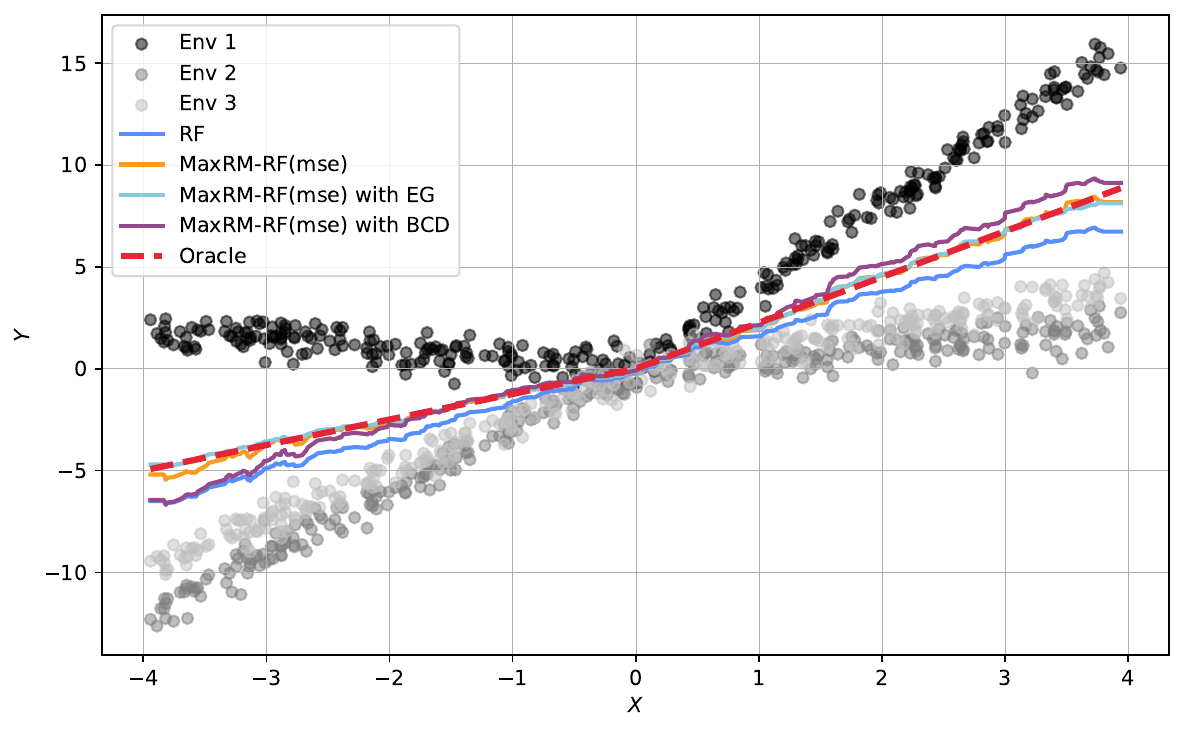}
    \caption{Comparison of fitted functions across methods: random forest (solid blue), MaxRM-RF with post-hoc adjustment using the interior-point method (solid orange), extragradient method (solid light blue), and block-coordinate descent (solid purple). All MaxRM-RF methods better approximate the oracle solution compared to RF, with BCD being slightly worse than than the other MaxRM-RF methods. Table~\ref{tab:xtrgrd} provides a quantitative comparison. 
    }\label{fig:xtrgrd}
\end{figure}
\begin{table}[t]
    \centering
    \caption{Comparison of forest-based methods in terms of maximum and pooled MSE across training environments, and runtime. Reported values are the mean $\pm$ half-length of the $95\%$ confidence interval. There are no significant differences between MaxRM-RF(mse) and MaxRM-RF(mse) with EG in terms of maximum MSE and pooled MSE. MaxRM-RF(mse) with BCD also improves over RF in terms of maximum MSE, but to a lesser extent.}
    \begin{tabular}{lccc}
        \toprule
        Method & Maximum MSE & Pooled MSE & Runtime (s)\\
        \midrule
        RF                       & $24.88 \pm 0.46$  & $12.82 \pm 0.23$ & $0.08 \pm 0.00$\\
        MaxRM-RF(mse)            & $16.75 \pm 0.32$  & $13.65 \pm 0.24$ & $1.05 \pm 0.03$\\
        MaxRM-RF(mse) with EG       & $16.76 \pm 0.30$  & $13.79 \pm 0.25$ & $1.92 \pm 0.12$\\
        MaxRM-RF(mse) with BCD       & $17.36 \pm 0.32$  & $14.04 \pm 0.25$ & $1.54 \pm 0.04$\\
        Oracle (for maximum MSE) & $16.58$           & $13.92$           & $-$\\
        Oracle (for pooled MSE)  & $25.29$ & $12.99$ & $-$\\
        \bottomrule
    \end{tabular}
    \label{tab:xtrgrd}
\end{table}

\subsection{Resolving leaf indeterminacy} \label{app:indet} 
We repeat the experiment from Section~\ref{sec:exp_with_shifts}, using the setting with shifts in $P^X_e$ across environments for the maximum MSE (results for other risks and for the no-shift setting are qualitatively similar and omitted). This time, after post-hoc adjustment, we revert the values of the leaves in each tree that do not contain any observations from the worst-case environments for that tree to the empirical mean of all observations of the response variable in the corresponding leaf (ignoring possible changes in the overall objective), as done by RF.
The results of Figure~\ref{fig:revert_to_RF}
\begin{figure}[t]
  \centering
  \subcaptionbox{Maximum MSE.}%
    {\includegraphics[width=0.49\linewidth]{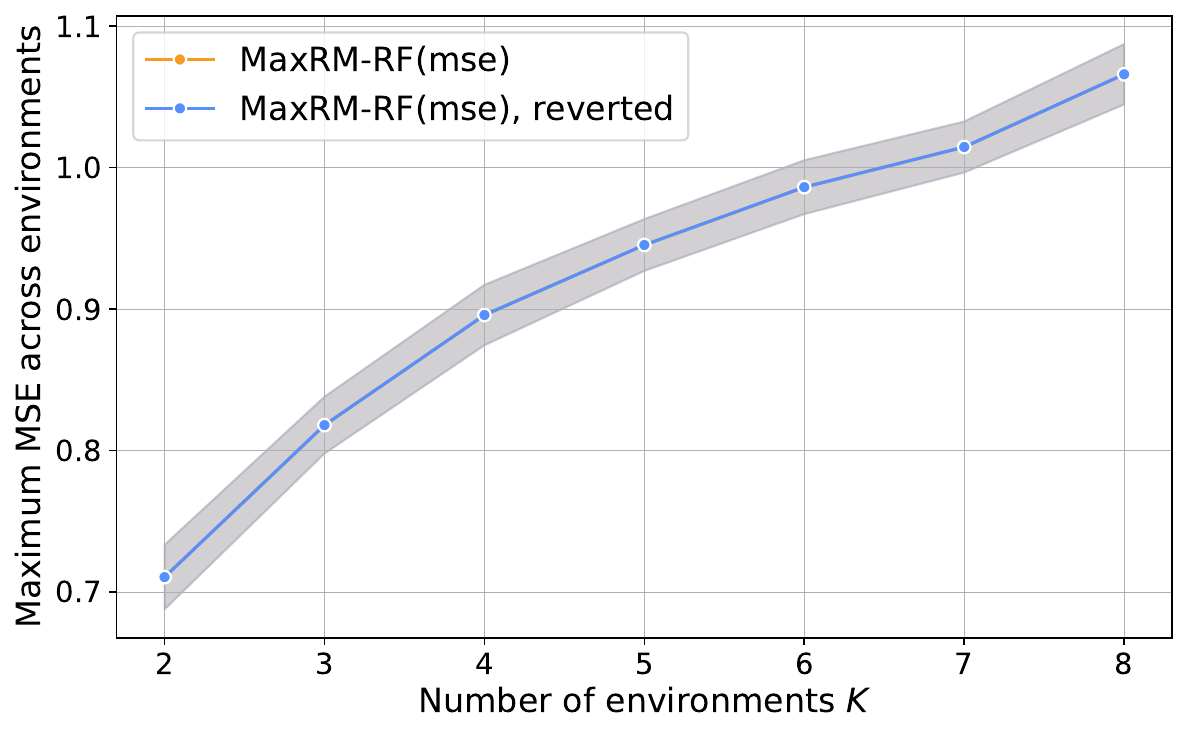}}
  \hfill
  \subcaptionbox{Proportion of indeterminate leaves.}%
    {\includegraphics[width=0.49\linewidth]{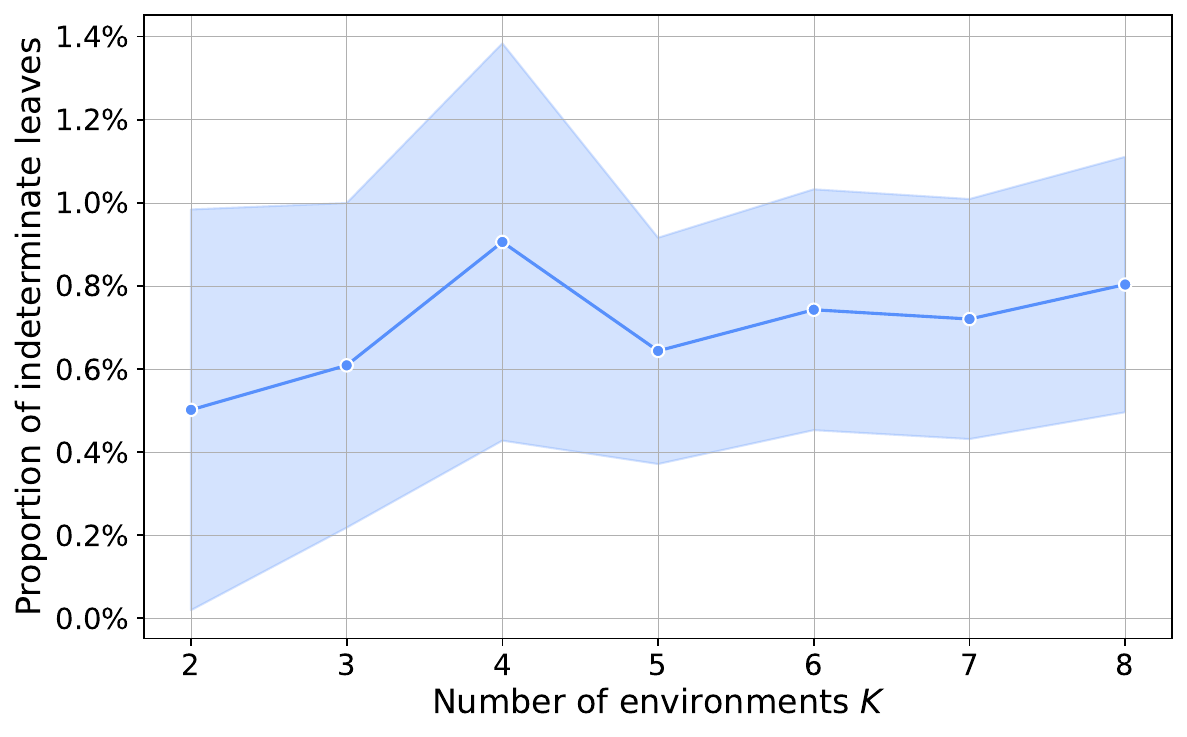}}
  \caption{Experiment from Section~\ref{sec:exp_with_shifts} (shifts in $P^X_e$), repeated for MaxRM-RF(mse) with implicit regularization (Remark~\ref{rem:indeterm}) and for a modified version in which leaf values not containing observations from the worst-case environments are reverted to their standard RF values (MaxRM-RF(mse), reverted). Curves show in (a) the average maximum risk on the test data of both variants over 100 repetitions and (b) the average proportion of indeterminate leaves among all leaves of the random forests; shaded regions show $95\%$ confidence intervals. The two variants exhibit indistinguishable performance (the yellow curve lies below the blue curve), and fewer than \(1\%\) of leaves are indeterminate.}
  \label{fig:revert_to_RF}
\end{figure}
show that this yields similar performance, and only few leaves are indeterminate. We therefore propose to use the implicit regularization as described in Remark~\ref{rem:indeterm}.

\subsection{Comparing MaxRM-RF and RF across hyperparameters}\label{sec:exp_hyperparams}

\paragraph{Goal.} 
In Sections~\ref{sec:exp_without_shifts}--\ref{sec:exp_identical_env}, we use the following default random forest hyperparameters: $m_\text{try}=5$ (all covariates considered at each split), an unrestricted maximum tree depth, and a minimum leaf size of $30$. We now aim to check how different choices for these hyperparameters affect the difference in performance between MaxRM-RF and RF in terms of maximum risk minimization.

\paragraph{Experiment description.} We use the simulation setting from Section~\ref{sec:exp_with_shifts} with shifts in $P^X_e$ across environments. We use the maximum MSE as the risk for both training and evaluation. The results for the regret and negative reward, as well as for the setting without shifts, are qualitatively similar and not shown. We perform three experiments, varying one hyperparameter at a time (keeping the others fixed at their default values):
\begin{enumerate}[label=(\alph*)]
    \item $m_\text{try} \in \{1,2,3,4,5\}$,
    \item maximum tree depth $\in \{3,4,5,6,7,8\}$,
    \item minimum leaf size $\in \{1,3,5,10,15,20,25,30,35,40\}$.
\end{enumerate}
All forests use $B=100$ trees. We report the average maximum test MSE over $100$ repetitions.

\paragraph{Results.} Figure~\ref{fig:hyperparams}
\begin{figure}[t]
  \centering
  \subcaptionbox{$m_\text{try}$\label{fig:mtry}}%
    {\includegraphics[width=0.32\linewidth]{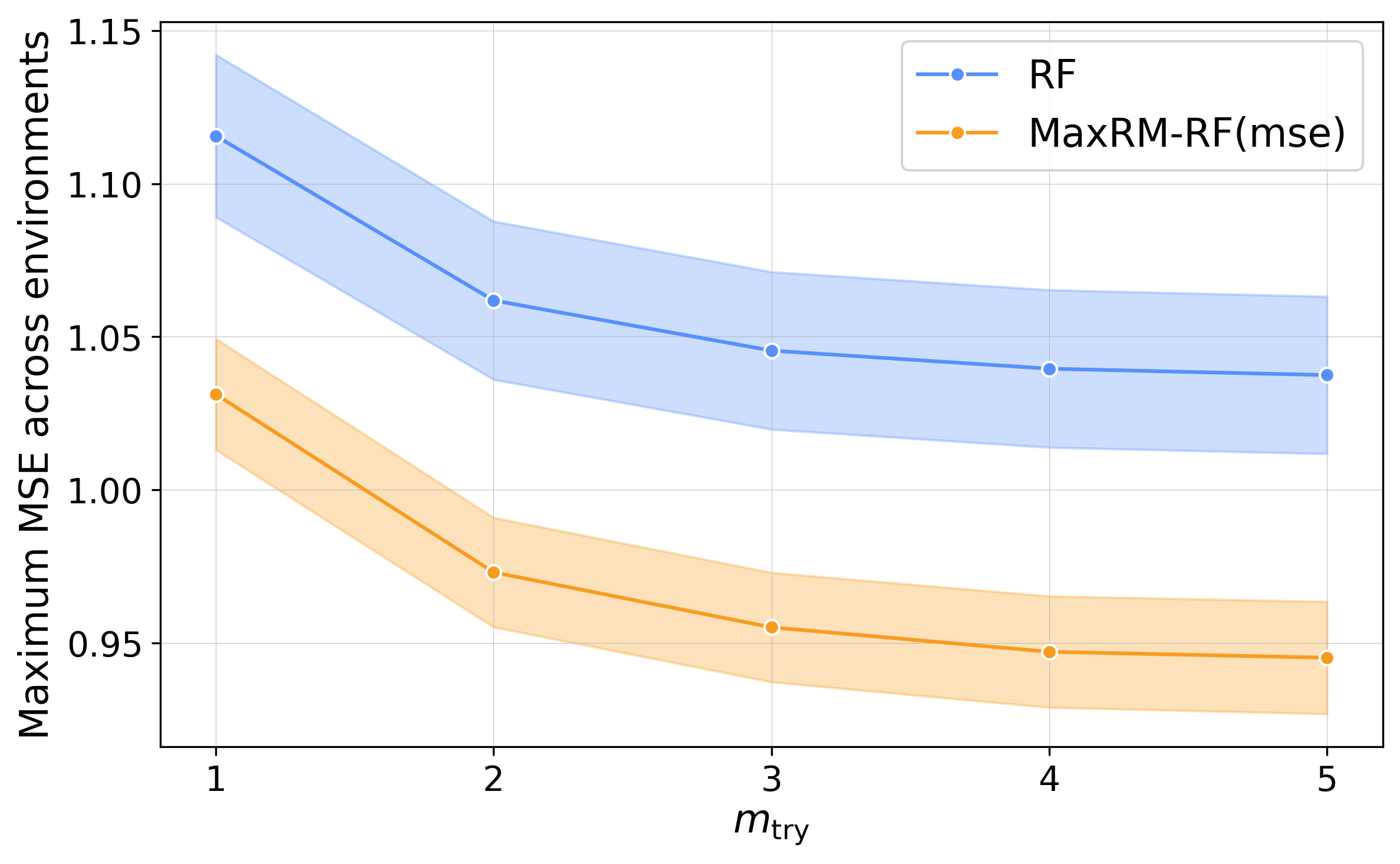}}
  \hfill
  \subcaptionbox{Maximum tree depth\label{fig:max_tree_depth}}%
    {\includegraphics[width=0.32\linewidth]{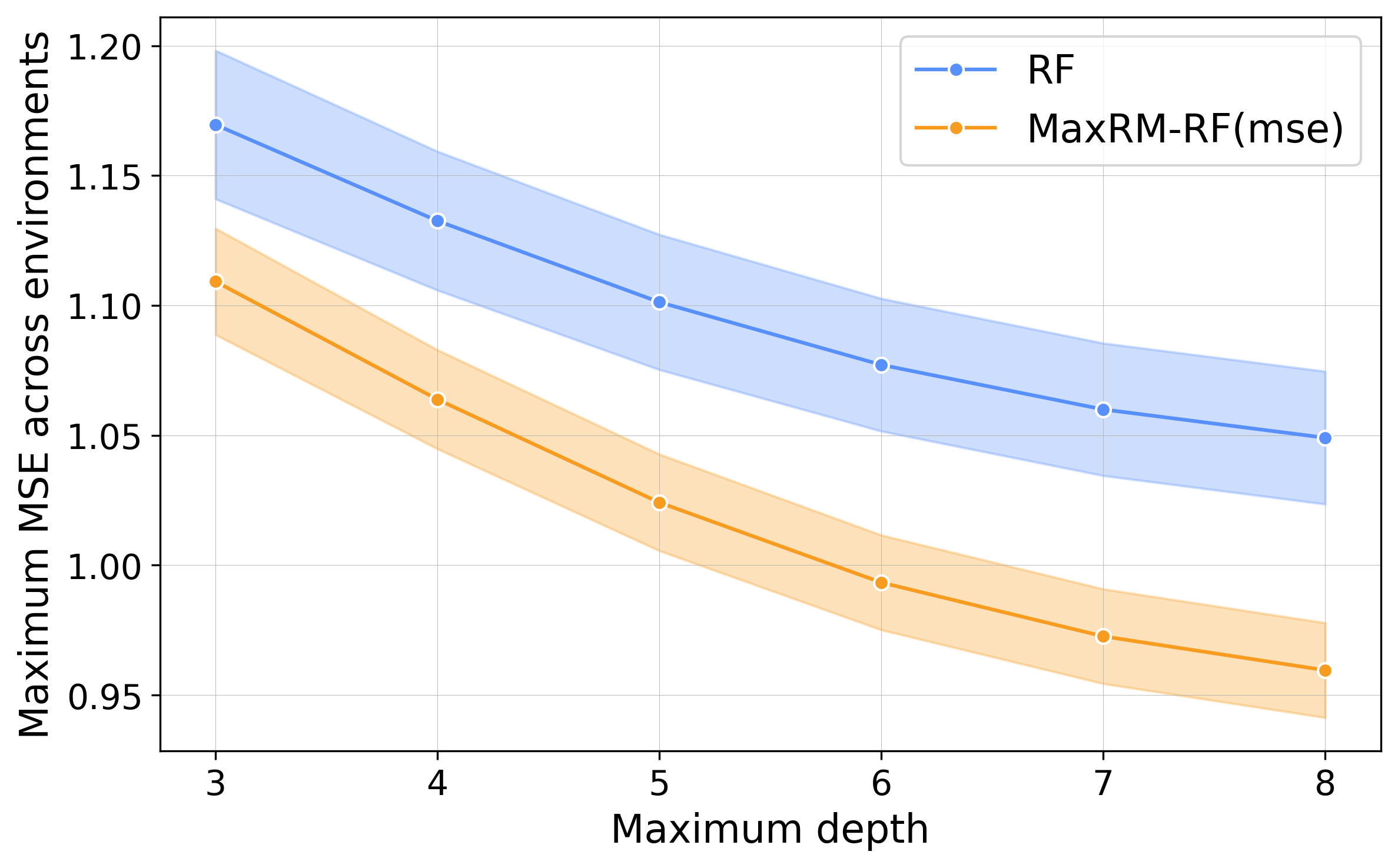}}
    \hfill
    \subcaptionbox{Minimum leaf size\label{fig:min_no_obs_leaf}}%
    {\includegraphics[width=0.32\linewidth]{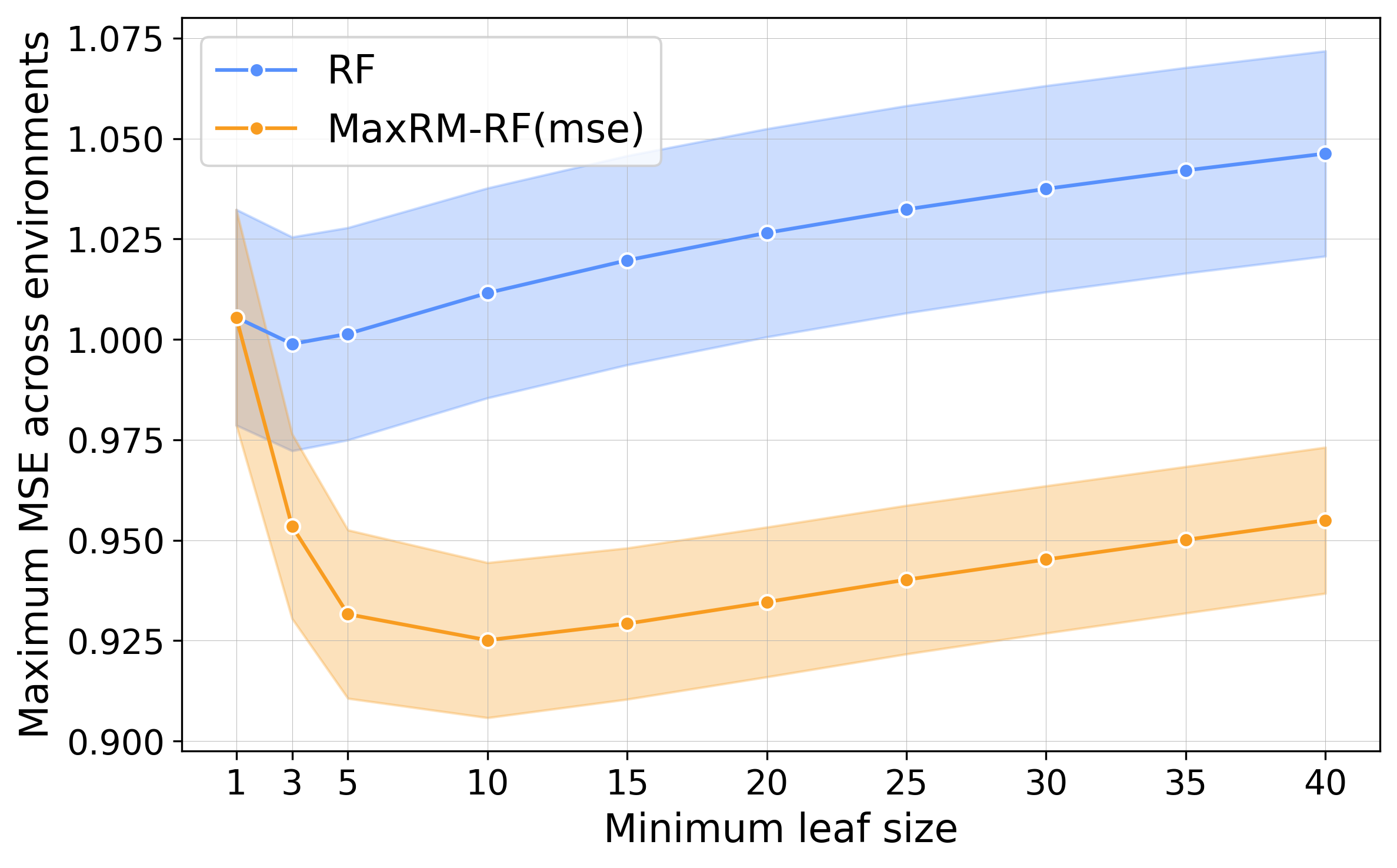}}
  \hfill
  \caption{Effect of hyperparameter choices on worst-case performance for RF and MaxRM-RF in terms of maximum MSE. Curves show the average maximum test MSE over $100$ repetitions with $95\%$ confidence intervals. MaxRM-RF outperforms RF across all settings, except for when the minimum leaf size is too small (this is to be expected, see Section~\ref{sec:exp_hyperparams}).}
  \label{fig:hyperparams}
\end{figure}
shows that MaxRM-RF reduces the maximum test MSE compared to RF for all tested hyperparameter values. For increasing $m_\text{try}$ (Figure~\ref{fig:hyperparams} (a)), the maximum MSEs decrease for both methods. Increasing the maximum tree depth (Figure~\ref{fig:hyperparams} (b)) improves worst-case performance for both methods. Increasing the minimum leaf size (Figure~\ref{fig:hyperparams} (c)) leads to higher maximum MSEs for both methods, as fewer splits reduce the flexibility of random forests. For a small minimum leaf size, the performance of MaxRM-RF approaches that of RF. This is expected, because in the extreme case of a single observation per leaf, each leaf corresponds to a single response value and environment. The RF leaf value already minimizes both the pooled and per-environment risks in that leaf, and therefore also the maximum risk, so post-hoc adjustment does not alter the RF leaf value.

\subsection{Comparison with group DRO and magging under alternative risks} \label{sec:app_comparison_grdo_magging_diffrisks}

\paragraph{Goal.} Section~\ref{sec:comparison_gdro} evaluates the performance of MaxRM-RF, the standard random forest, group DRO, and magging using the MSE. We now study their performance using the negative reward and regret instead.

\paragraph{Experiment description.} We repeat the experiments from Section~\ref{sec:comparison_gdro} with the same setup, but we replace the MSE with the negative reward and the regret. MaxRM-RF, group DRO, and magging are each configured to minimize the same maximum risk (negative reward or regret) that is used for evaluation.

\paragraph{Results.} Overall, the results for both alternative risks remain qualitatively the same as
the findings for the MSE from Section~\ref{sec:comparison_gdro}.
Figure~\ref{fig:comparison_nrw_reg_noshift} 
\begin{figure}[t]
  \centering
  \subcaptionbox{Maximum negative reward.}%
    {\includegraphics[width=0.49\linewidth]{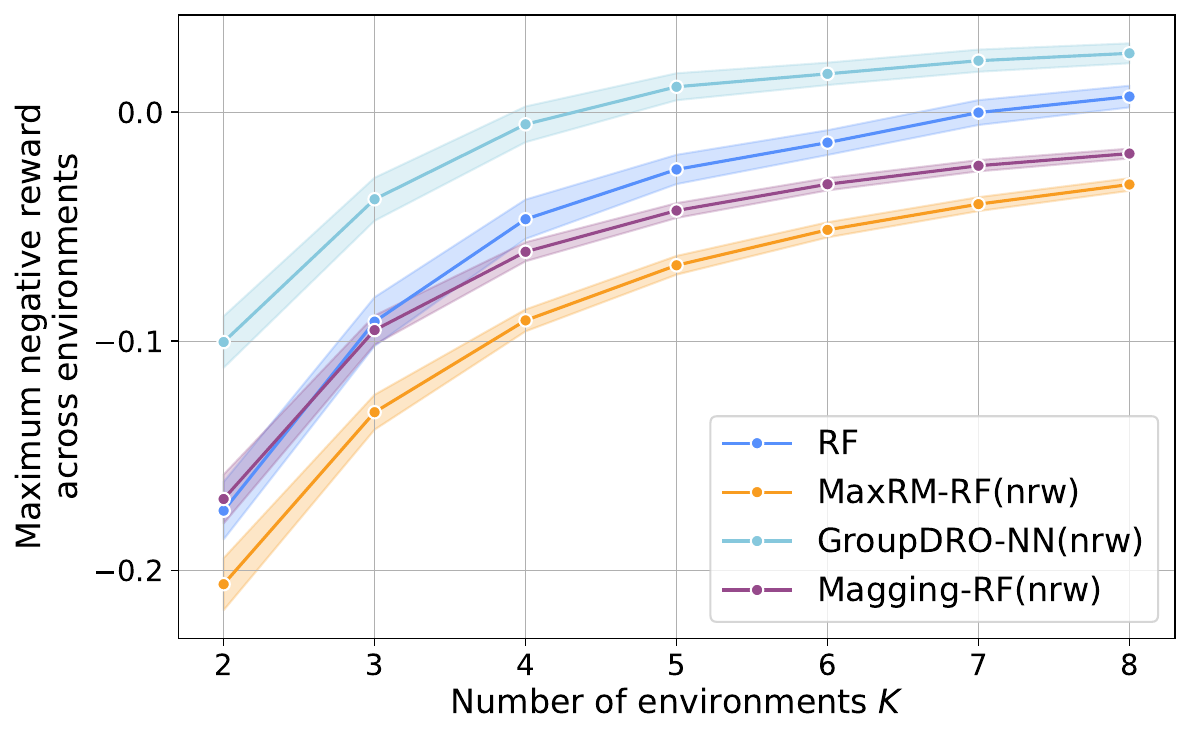}}
  \hfill
  \subcaptionbox{Maximum regret.}%
    {\includegraphics[width=0.49\linewidth]{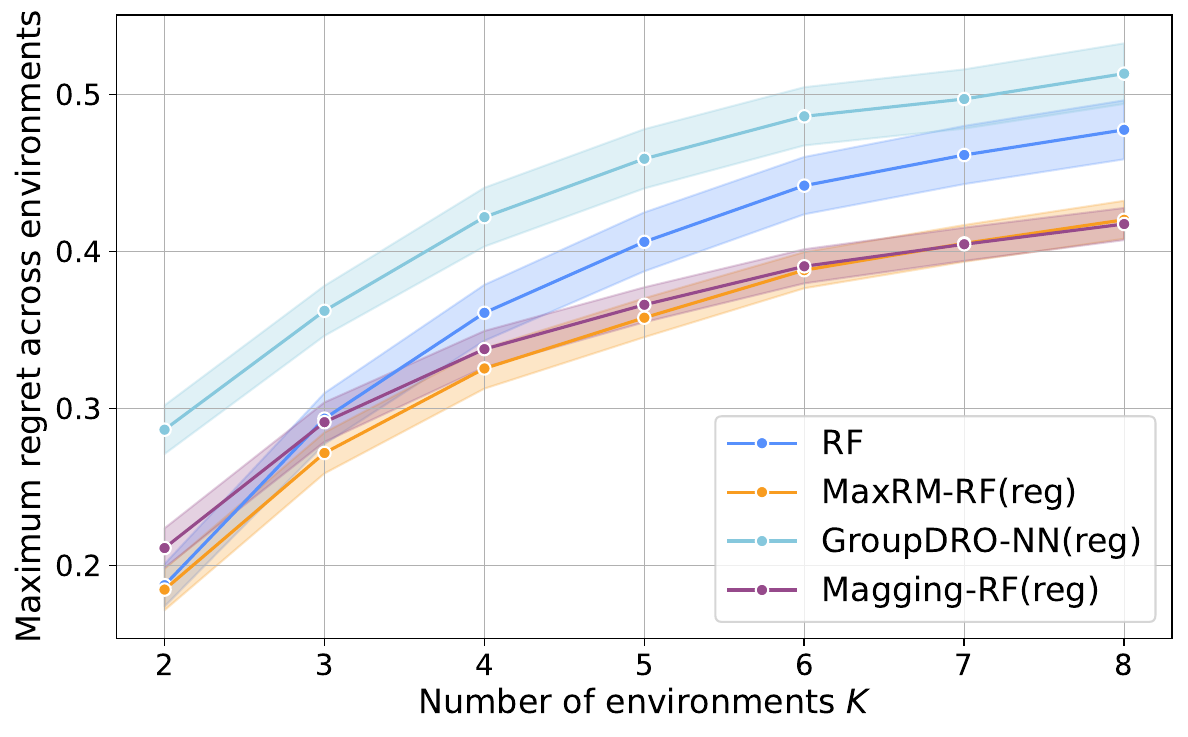}}
  \caption{Experiment from Section~\ref{sec:exp_without_shifts} without shifts in $P^X_e$ across environments, repeated for (a) the negative reward  and (b) the regret instead of the MSE. Curves report the average maximum risk over 100 repetitions; shaded regions show $95\%$ confidence intervals. MaxRM-RF and magging outperform group DRO and RF for both risks. For the negative reward in (a), MaxRM-RF obtains lower maximum negative rewards than magging, while the maximum regrets are similar for both methods in (b).}
  \label{fig:comparison_nrw_reg_noshift}
\end{figure}
\begin{figure}[t]
  \centering
  \subcaptionbox{Maximum negative reward.}%
    {\includegraphics[width=0.49\linewidth]{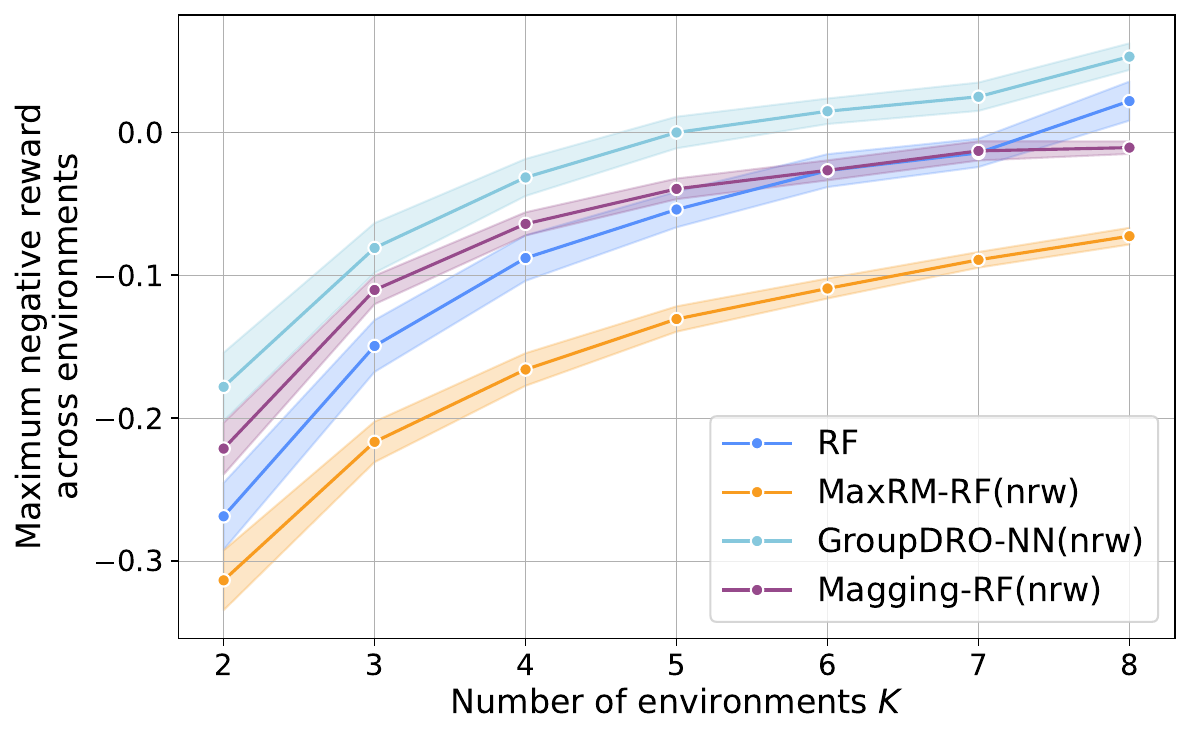}}
  \hfill
  \subcaptionbox{Maximum regret.}%
    {\includegraphics[width=0.49\linewidth]{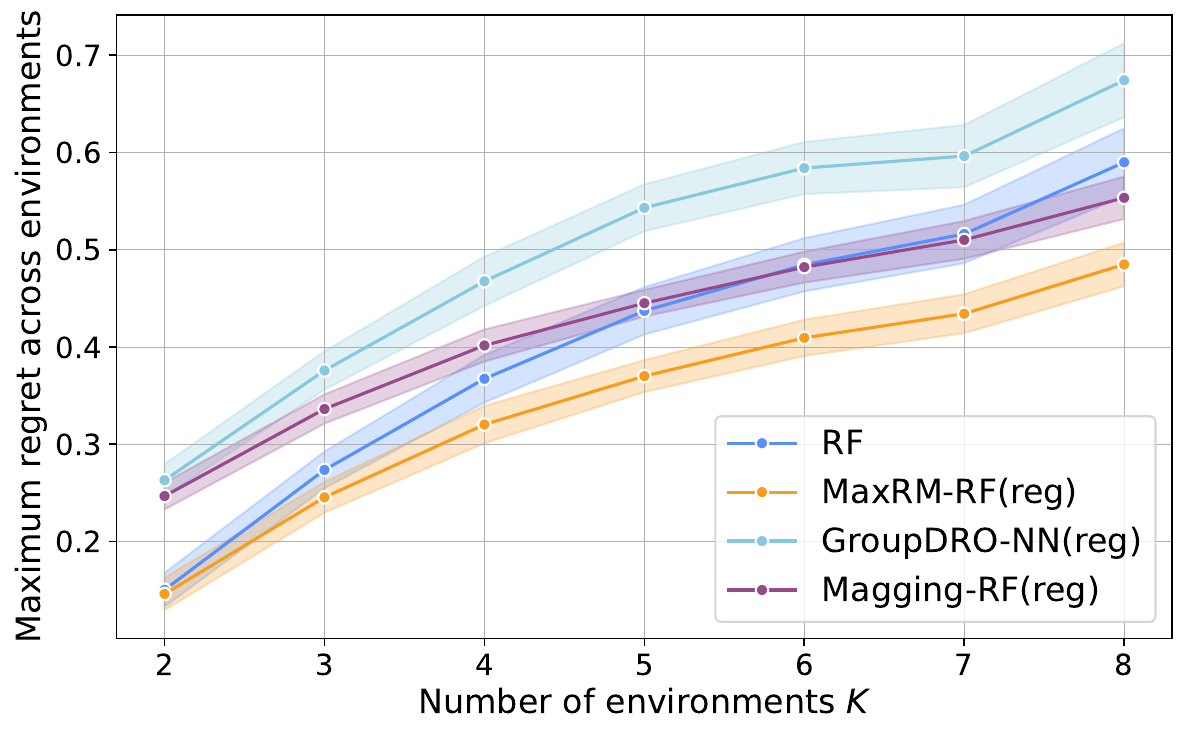}}
  \caption{Experiment from Section~\ref{sec:exp_with_shifts} with shifts in $P^X_e$ across environments, repeated for (a) the negative reward and (b) the regret instead of the MSE. Curves show the average maximum risk across environments, with shaded areas indicating $95\%$ confidence intervals over 100 repetitions. In the presence of shifts in $P^X_e$, MaxRM-RF consistently attains the lowest maximum risks, outperforming all baselines for both risks. Magging no longer improves upon RF under these shifts.}
  \label{fig:comparison_nrw_reg_shift}
\end{figure}
shows the maximum negative reward and regret across test environments when covariate distributions are identical across environments. As for the MSE, group DRO performs worst, exhibiting the largest maximum risks for all numbers of environments. Magging and MaxRM-RF improve upon RF for $K \geq 5$ environments. For the negative reward, MaxRM-RF 
outperforms all other methods. For the regret, MaxRM-RF and magging 
perform similarly.
Figure~\ref{fig:comparison_nrw_reg_shift} shows the corresponding results in the presence of covariate shifts. In this setting, MaxRM-RF again improves upon magging for both risks, as previously observed for the MSE in Section~\ref{sec:exp_with_shifts}. Moreover, magging no longer improves upon the standard random forest when covariate shifts are present.

\appsection{Additional analysis for the California housing application}\label{app:ca-housing}
\subsection{Full results for all MaxRM-RF variants and baselines}\label{app:ca:extended}
Table \ref{tab:app:ca_res} reports the full set of results for all methods and all risk objectives. The main text focuses on models trained with the maximum MSE risk; here, we additionally include the MaxRM-RF(nrw), MaxRM-RF(reg), magging(nrw), and magging(reg) variants.
MaxRM-RF(nrw) and MaxRM-RF(reg) variants improve upon RF in Fold II. The regret-based estimator behaves similarly to RF across folds, while the negative-reward variant performs worse and systematically shrinks predictions toward zero, consistent with observations in earlier studies \citep{Lund2022,Mo2024,Zhang2024}. In contrast, MaxRM-RF(mse)—reported in Section \ref{sec:ca_housing}—yields the greatest improvements in worst-case test MSE overall.
\begin{table}[t]
    \centering
    \caption{Maximum test MSE 
    in the California hosing application
    over five held-out counties. For each held-out fold, the best performing method is highlighted in bold. Cells shaded in gray indicate that the method performs better than RF, testing at a significance level of 0.05 with a permutation test and Bonferroni correction. MaxRM-RM(mse) performs best in four of five folds (significantly better than RF in three of the five folds). Abbreviations: Mag: Magging-RF, GDRO: GroupDRO-NN, MRF: MaxRM-RF.
    }
    \resizebox{\linewidth}{!}{
    \begin{tabular}{lccccccccc}
    \toprule
    Fold & LR & RF & GDRO & Mag(mse) & Mag(nrw) & Mag(reg) & MRF(mse) & MRF(nrw) & MRF(reg) \\
    \midrule
    I & $11.978$ & $1.269$ & $1.351$ & $1.474$ & $3.075$ & $1.653$ & \cellcolor{gray!25}\bm{$1.180$} & $2.295$ & $1.452$ \\ %
    II & $0.605$ & $0.525$ & $0.919$ & $0.950$ & $0.926$ & $0.712$ & $0.653$ & $0.485$ & \cellcolor{gray!25}\bm{$0.403$} \\ %
    III & $1.226$ & $0.765$ & $1.005$ & $0.754$ & $2.633$ & $0.959$ & \cellcolor{gray!25}\bm{$0.586$} & $1.683$ & $0.946$ \\ %
    IV & $0.962$ & $0.785$ & $1.366$ & $1.352$ & $1.129$ & $1.071$ & \bm{$0.758$} & $1.004$ & $0.839$ \\ %
    V & $0.714$ & $0.508$ & $0.768$ & $0.944$ & $1.028$ & $0.779$ & \cellcolor{gray!25}\bm{$0.490$} & $1.063$ & $0.546$ \\ %
    \bottomrule
    \end{tabular}
    }
    \label{tab:app:ca_res}
\end{table}

\subsection{Train–test pairwise performance across counties}\label{app:ca:pairwise}
To examine the heterogeneity between environments, we compute the maximum MSE of RFs for all ordered train–test county pairs. Figure~\ref{fig:ca-housing-pairwise} reports the resulting $25 \times 25$ matrix. Rows correspond to training counties and columns to test counties. 
\begin{figure}[t]
\centering
\includegraphics[width=0.7\linewidth]{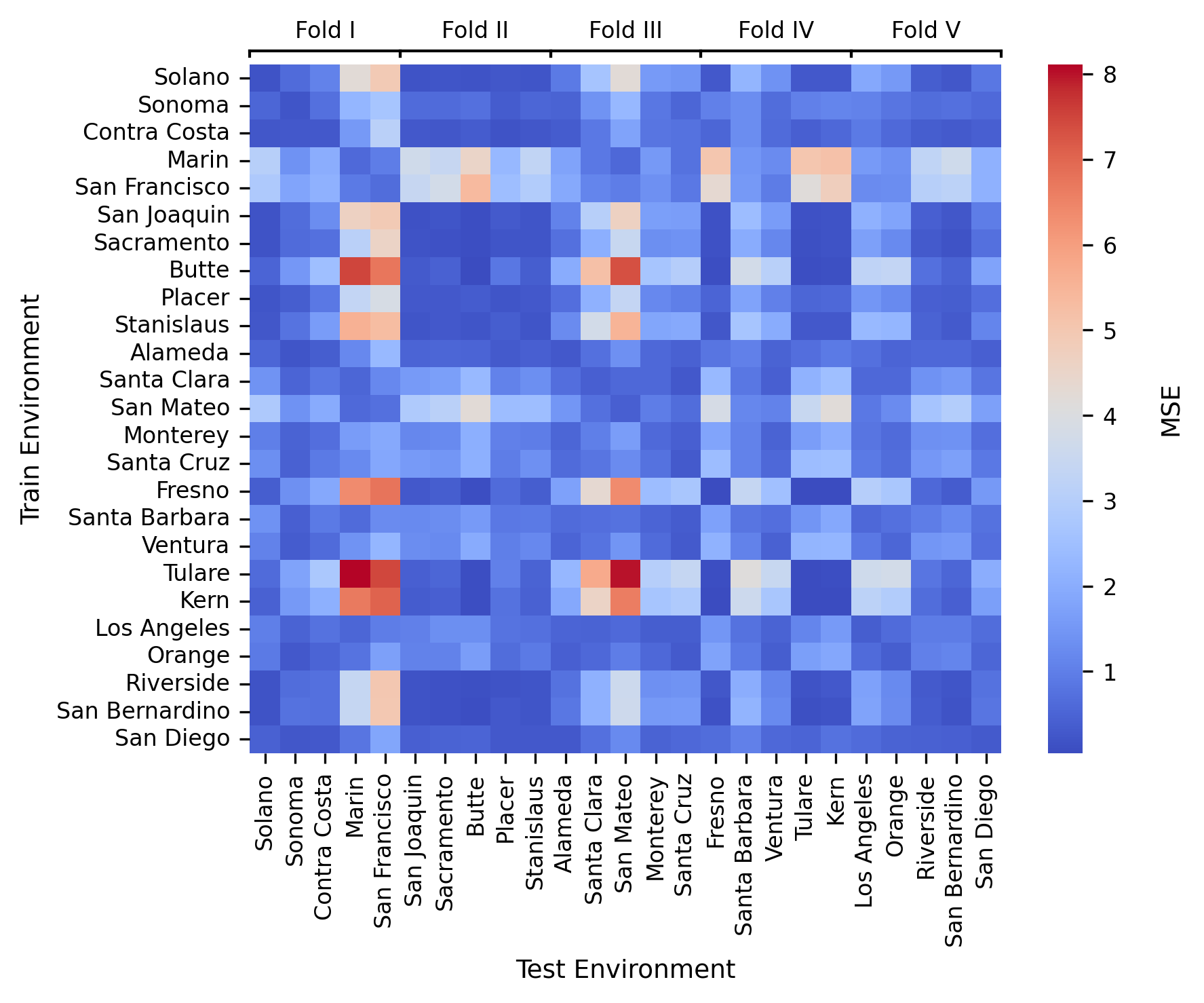}
\caption{Train–test MSEs for RFs for all ordered county pairs. Rows denote training counties and columns denote test counties. Diagonal values are the out-of-bag MSE. Redder cells indicate higher MSE when a model trained on one county is applied to another. Some
counties (e.g., Marin, San Francisco, San Mateo),
are  difficult to predict, 
independently of the choice of training environment. Here, the counties are grouped into $5$ folds as visualized in Figure~\ref{fig:ca_map}.
}
\label{fig:ca-housing-pairwise}
\end{figure}

Several counties—most notably Marin, San Francisco, and San Mateo—appear consistently more difficult to predict when used as test environments, whereas diagonal values (within-county errors) are comparatively small and similar. This structure confirms the presence of substantial and uneven distribution shifts across counties.

\end{document}